\documentclass[pdflatex, sn-apa]{sn-jnl}

\usepackage{graphicx}%
\usepackage{subcaption}%
\usepackage{multirow}%
\usepackage{amsmath,amssymb,amsfonts}%
\usepackage{amsthm}%
\usepackage{mathrsfs}%
\usepackage[title]{appendix}%
\usepackage[svgnames, dvipsnames]{xcolor}%
\usepackage{textcomp}%
\usepackage{manyfoot}%
\usepackage{booktabs}%
\usepackage{algorithm}%
\usepackage{algorithmicx}%
\usepackage{algpseudocode}%
\usepackage{listings}%
\usepackage{anyfontsize}%
\usepackage{float}
\usepackage{soul}
\usepackage{makecell}
\usepackage{changepage}
\usepackage{pdflscape}
\usepackage{tikz}%
\usetikzlibrary{ext.paths.ortho}

\usepackage[frozencache, cachedir=.]{minted}
\colorlet{custombg}{black!10}
\setminted{autogobble=true, frame=single, fontfamily=tt, escapeinside=||}
\setminted[py]{bgcolor=custombg, fontsize=\small}

\theoremstyle{thmstyleone}%
\newtheorem{theorem}{Theorem}

\newtheorem{lemma}{Lemma}

\theoremstyle{thmstyletwo}%
\newtheorem{remark}{Remark}%

\theoremstyle{thmstylethree}%

\begin{document}

\title[MomentQuant]{MomentQuant: an even more minimalist interval method with linear time complexity for time series classification}

\author[]{\fnm{Johann} \sur{Faouzi}}\email{johann.faouzi@ensai.fr}

\affil[]{Univ Rennes, Ensai, CNRS, CREST - UMR 9194, F-35000 Rennes, France}

\abstract{
    Time series data is very common in many real-world applications and in numerous domains, with increasing interest for automated information extraction using machine learning.
    One of these subfields is time series classification, which consists in assigning a label to each new, unseen time series.
    Many algorithms have been developed over the past decades, with the trade-off between predictive performance and computational cost being consistently discussed.
    Quant, an interval-based algorithm extracting quantiles from recursive, fixed, dyadic intervals, was shown to achieve high accuracy, while being very fast.
    We propose two changes to make this algorithm even faster.
    The first one is a better optimized implementation of the exact same algorithm.
    The second one is to derive approximate quantiles, using the Cornish-Fisher expansion, instead of exact quantiles.
    This change removes the necessity to sort the time series, leading to a smaller computational complexity.
    We call this novel algorithm MomentQuant.
    We provide evidence that our implementation of Quant is faster than the original one, and that MomentQuant is even faster than our implementation of Quant, at the cost of a tiny decrease in predictive performance.
    These improvements are especially relevant for real-life applications, where inference is performed much more often than training.
}


\maketitle

\section{Introduction}\label{sec1}

Time series classification is the supervised learning task of assigning a discrete label to a time-ordered sequence of measurements, and it arises across a wide range of application domains, including health monitoring, industrial process control, astronomy, and human activity recognition.
Because the discriminative structure of a time series can lie in its overall shape, in local patterns, or in its frequency content, depending on the domain, methods that perform consistently well across many kinds of data sets typically combine features extracted from several complementary representations of the input series rather than committing to a single one.
The introduction and steady growth of the University of California, Riverside (UCR) time series archive \citep{dauUCRTimeSeries2019e}, a public collection of benchmark data sets spanning dozens of application domains, made it possible to compare such methods systematically, and has driven two decades of active method development in the field.

This growth has, however, exposed a persistent tension between predictive accuracy and computational cost.
The most accurate methods on the UCR archive are typically large, heterogeneous ensembles combining several complementary representations and classifiers, such as the Hierarchical Vote Collective of Transformation Ensembles (HIVE-COTE) 2.0 \citep{middlehurstHIVECOTE20New2021a}.
Their accuracy gains come at a steep computational price, often requiring hours or days to train and test even on the comparatively small (by modern machine learning standards) UCR data sets, which makes them impractical to scale to larger data sets or to apply repeatedly, e.g., for hyperparameter tuning.
This tension motivated a distinct line of work seeking to reproduce most of that accuracy at a small fraction of the cost, most prominently the Random Convolutional Kernel Transform (ROCKET) algorithm \citep{dempsterROCKETExceptionallyFast2020} and its successor MiniRocket \citep{dempsterMiniRocketVeryFast2021a}, which summarize each series through a large bank of (random or fixed) convolutional kernels and train a linear classifier on the resulting features, reaching near-state-of-the-art accuracy while being orders of magnitude cheaper to train than the heterogeneous ensembles above.

Quant \citep{dempsterQuantMinimalistInterval2024} is a recent, markedly simpler entry in this second line of work.
Rather than convolutional kernels, it computes a single type of feature, empirical quantiles, from a fixed, hierarchical set of intervals of each series, over four complementary representations of the series (the raw series, a smoothed first-order difference, the second-order difference, and the magnitude of its discrete Fourier transform), and feeds the concatenated quantiles, unchanged, to an off-the-shelf tree ensemble classifier.
Despite this minimalism, and despite using a single feature type where most interval-based methods before it combined several, Quant matches the accuracy of the most accurate interval-based methods on the UCR archive, and stays close in aggregate accuracy to the very best, much more expensive heterogeneous ensembles.
Quant only needs, by the original authors' own account, under fifteen minutes of combined training and inference time on a single central processing unit (CPU) core to process the $142$ univariate data sets of the UCR archive \citep{dempsterQuantMinimalistInterval2024}.

Although Quant's overall speed is undeniable, its computational complexity has not been established so far.
The original publication has a short section on this topic, but does not provide an in-depth analysis \citep{dempsterQuantMinimalistInterval2024}.
The authors mention that they ``\emph{treat the computational cost of sorting the values as an upper bound on the cost of computing the quantiles: $\mathcal{O}(l \cdot \log(l))$, where $l$ is the time series length}''.
However, the subseries from each interval are also sorted, not just the whole time series, in order to compute the quantiles.
Moreover, the authors do not prove that the cost of computing the quantiles is upper bounded by the cost of sorting the values.
The authors also provide the computational complexity of the training of the classification step of Quant, and their results demonstrate that the training the classification step is the longest part of the whole pipeline.
Overall, the original publication overlooks the computational complexity of the transformation step, as if it was irrelevant compared to the computational complexity of the classification step.
We will demonstrate that this analysis does not hold when considering only the inference phase of the algorithm, which is very relevant for real-life applications of any machine learning model, and is too often overlooked in the time series classification literature.
We will also provide an in-depth analysis of the computational complexity of Quant, both theoretical and practical.

Quant's overall speed does not mean that every part of its reference implementation\footnote{\url{https://github.com/angus924/quant}} is equally well suited to every setting that it is run in.
The implementation commits to a single, fixed computational structure: one call to PyTorch's batched \mintinline{py}{torch.quantile} function per interval, amortized over the whole batch of series at once.
This is a natural design for a tensor library built around batched, hardware-accelerated execution.
However, it is not well-matched to either end of the series-length spectrum on a single CPU core, the setting that Quant's own headline runtime figures above were obtained in, as we detail in \autoref{sec3}.
For short series, the fixed per-call dispatch overhead of this design, paid once per interval regardless of how little work that interval actually requires, comes to dominate the total cost.
For long series, the same design instead pays a growing sorting cost, since exactly computing a set of quantiles from $l$ raw values requires first sorting them, which is an $\mathcal{O}(l \cdot \log(l))$ operation.

This paper addresses both regimes, but its main contribution targets the long-series one.
We reimplement Quant from scratch in NumPy, both to reproduce the original algorithm faithfully as an \emph{exact} mode, which we further speed up for short series (\autoref{sec5}) by choosing between two different loop orderings for the same computation, guided by a full theoretical cost analysis (\autoref{sec4}). Then, we introduce a genuinely new \emph{approximate} mode aimed at long series.
Instead of sorting each interval to obtain its exact quantiles, this approximate mode estimates them directly from that interval's own sample mean, variance, skewness, and excess kurtosis.
These four quantities are computable in a single $\mathcal{O}(l)$ pass, via the Cornish-Fisher expansion \citep{cornishMomentsCumulantsSpecification1938, fisherPercentilePointsDistributions1960}, a classical asymptotic correction of the standard normal quantile for a distribution's departure from normality.
Doing so trades the exact mode's $\mathcal{O}(l \cdot \log(l))$ sorting cost for an $\mathcal{O}(l)$ moment-based one, at the cost of a small, and, as we show empirically, largely controllable approximation error.
Concretely, this paper makes the following contributions:
\begin{itemize}
    \item A precise theoretical cost model, with proofs, for the exact mode and its two natural loop orderings (\autoref{sec4} and \autoref{sec5}), and for the new approximate mode (\autoref{sec6}), including a closed-form asymptotic characterization of how each mode's total cost scales with series length.
    \item A moment-based approximate mode for Quant built on the Cornish-Fisher expansion, together with an empirical study of its quantile-approximation fidelity and of the resulting accuracy and runtime trade-off, conditioned on series length, across the $142$ univariate data sets of the UCR archive (\autoref{sec6} and \autoref{sec8}).
    \item An automatic dispatch heuristic for each mode (exact and approximate), calibrated per machine, that selects between the loop orderings, based on the series length at hand (\autoref{sec6}, \autoref{sec8}).
    \item A comprehensive, single-threaded and multithreaded analysis of an optimization that was mentioned in \citep{dempsterQuantMinimalistInterval2024}, but neither implemented nor evaluated (\autoref{sec9}).
\end{itemize}

The remainder of this paper is organized as follows.
\autoref{sec2} provides some background, with an overview of the time series classification literature, the interval-based methods, and Quant.
In \autoref{sec3}, we explain the specific limitations of Quant's reference implementation that motivate this paper.
\autoref{sec4} develops a theoretical cost model for processing a single series under Quant's exact quantile computation.
\autoref{sec5} uses this model to choose, for the exact mode, between two loop orderings depending on series length.
In \autoref{sec6}, we introduce the moment-based approximate mode built on the Cornish-Fisher expansion, together with its own cost model.
\autoref{sec7} presents our experimental setup, while \autoref{sec8} presents all our results.
We investigate the optimization mentioned in Quant's original publication in \autoref{sec9}, before concluding in \autoref{sec10}.

\section{Background}\label{sec2}

We divide this section dedicated to the relevant background into three parts.
First, we provide an overview of the time series classification literature.
Then, we focus on the specific family of algorithms that Quant belongs to.
Finally, we provide an in-depth presentation of the Quant algorithm.
For a recent review of the time series classification literature, we refer the readers to \citep{middlehurstBakeReduxReview2024}.
We will not cover deep learning in this section and refer the readers to \citep{ismailfawazDeepLearningTime2019a} and \citep{mohammadifoumaniDeepLearningTime2024}.

\subsection{Time series classification}

Numerous algorithms for time series classification have been developed over the past decades.
These methods can be grouped in different families based on their structures.

Distance-based methods rely on computing (dis)similarity scores between samples.
For classification, such classic methods include nearest-neighbor methods \citep{fixDiscriminatoryAnalysisNonparametric1989, coverNearestNeighborPattern1967} and support vector machines \citep{cortesSupportVectorNetworks1995a}.
For time series classification, specific distances and kernels have been developed.
Dynamic Time Warping (DTW) \citep{sakoeDynamicProgrammingAlgorithm1978, berndtUsingDynamicTime1994} is an elastic distance that uses dynamic programming to find the optimal alignment between two time series by computing the minimum path through a cost matrix consisting of the pairwise point-wise squared differences.
Several variants of DTW have been developed, some of them adding a region constraint on the possible set of paths \citep{sakoeDynamicProgrammingAlgorithm1978, itakuraMinimumPredictionResidual1975} and some others adding weights penalizing alignments with high phase differences \citep{jeongWeightedDynamicTime2011, herrmannAmercingIntuitiveEffective2023}.
Global alignment kernels \citep{cuturiKernelTimeSeries2007, cuturiFastGlobalAlignment2011a} are kernels specific to time series that can be used with support vector machines for classification.

Feature-based approaches consist in computing statistics from the whole time series.
For time series classification, a standard classification algorithm is then built on top of these derived statistics.
The canonical time series characteristics (Catch22) \citep{lubbaCatch22CAnonicalTimeseries2019} are 22 features that have been determined to be discriminative on the UCR data sets \citep{dauUCRTimeSeries2019e}, and a decision tree was used to perform classification.
The Time Series Feature Extraction based on Scalable Hypothesis Tests (TSFresh) algorithm \citep{christTimeSeriesFeatuRe2018a} is a set of nearly 800 features extracted from time series.
This set of features can be pruned using statistical tests.
The Random forest \citep{breimanRandomForests2001a} and AdaBoost \citep{freundExperimentsNewBoosting1996} algorithms were investigated to perform classification using these extracted features.

Shapelets are subseries extracted from the training time series and used to differentiate time series.
To do so, the most commonly used metric is the minimum of the Euclidean distances between the shapelet and all the subseries, of the same length as the shapelet, extracted from the time series.
For time series classification, shapelets were first investigated in \citep{yeTimeSeriesShapelets2011} and the transformation was followed by a decision tree to perform classification.
The Shapelet Transform Classifier (STC) \citep{hillsClassificationTimeSeries2014a} extracts all the possible shapelets from the training set before selecting the most discriminative ones, followed by an ensemble of classifiers.
Several refinements have been made since the release of its original version to improve its performance and scalability, notably performing a random search for the shapelets and using a single classification algorithm on top of the transformation \citep{bostromBinaryShapeletTransform2017, bostromEvaluatingImprovementsShapelet2016a}.
The Random Dilated Shapelet Transform (RDST) algorithm \citep{guillaumeRandomDilatedShapelet2022} adds two novel elements to existing shapelet-based approaches: dilation and two new features (the position of the minimum distance and the number of occurrences of the shapelet).

Dictionary-based methods, similarly to shapelet-based approaches, also extract subseries from time series.
However, no distance between the subseries and the time series is computed.
Instead, each subseries is turned into a short sequence of discrete symbols, which is usually called a word.
The frequencies of all the words extracted from a time series are then computed to obtain the new representation of this time series.
For time series classification, a standard machine learning classification algorithm is applied on top of this transformation.
There exist two main symbolic representations of time series.
The first one, in the time domain, is the Symbolic Aggregate approXimation (SAX) \citep{linExperiencingSAXNovel2007a}, which performs dimensionality reduction first in the time domain (using the mean from non-overlapping windows) and then in the value domain (using discretization based on quantiles).
The second one, in the frequency domain, is the Symbolic Fourier Approximation (SFA) \citep{schaferSFASymbolicFourier2012}, which computes the discrete Fourier transform of the time series, selects a subset of the Fourier coefficients, and discretizes them (using quantiles).
Several algorithms in this family have been developed, mostly using the SFA representation.
The Bag-of-SFA-Symbols (BOSS) model \citep{schaferBOSSConcernedTime2015} extracts subseries from a time series using overlapping windows, then transforms each subseries into a word using SFA, and finally the word frequencies are computed.
An ensemble of BOSS models is used in practice, with different values for several hyperparameters.
The Word Extraction for Time Series Classification (WEASEL) algorithm \citep{schaferFastAccurateTime2017a} involves several new changes compared to BOSS.
Notably, the selection of the Fourier coefficients is supervised (using ANOVA tests), the quantization of SFA is also supervised (using information gain), bigrams and multiple window lengths are considered, a subset of words is selected in a supervised fashion (using chi-squared tests), and the classification step is replaced with logistic regression.
WEASEL was later refined in a new version called WEASEL 2.0 \citep{schaferWEASEL20Random2023}, adding notably a novel dilation mapping, using less supervised selection methods in SFA, and replacing logistic regression with a Ridge classification algorithm.

Convolution-based methods rely on the convolution operator to extract features from time series using kernels.
However, it has been shown that using random kernels instead of learned kernels is much faster but still very effective for time series classification.
This strategy was introduced with the ROCKET algorithm \citep{dempsterROCKETExceptionallyFast2020}.
ROCKET generates numerous random kernels (with random length, weights, bias, padding and dilation), applies the convolution for each of them, and extracts two aggregate features for each of them: the maximum value and the proportion of positive values.
Finally, a Ridge classifier is trained on these extracted features.
ROCKET has been extended in two versions.
The first one is MiniRocket \citep{dempsterMiniRocketVeryFast2021a}, which involves much less randomness in the generation of the kernels than ROCKET and only the proportion of positive values is derived.
The second extension is MultiRocket \citep{tanMultiRocketMultiplePooling2022}, which adopts the improvements of MiniRocket but also includes two notable new changes: more features are derived, and half of the convolutions are applied to the first-order difference of the time series.
A model combining both dictionary- and convolution-based approaches, called HYbrid Dictionary-ROCKET Architecture (Hydra) \citep{dempsterHydraCompetingConvolutional2023}, was later proposed.
The kernels are aggregated into groups, the best matching kernel among each group is recorded, and their frequencies are computed.

Hybrid algorithms combine algorithms from different families in order to cover as many data sets and problems as possible.
HIVE-COTE \citep{linesTimeSeriesClassification2018a} is a heterogeneous ensemble consisting of five algorithms, each from a different representation.
This algorithm was updated shortly after to improve its scalability: HIVE-COTE 1.0 \citep{bagnallUsagePerformanceHierarchical2020a} uses four algorithms instead of five, and faster algorithms from each family.
HIVE-COTE 2.0 \citep{middlehurstHIVECOTE20New2021a} addresses further scalability issues and updates its components to use better, faster algorithms.

\subsection{Interval-based methods}

The pipelines of most interval-based methods are very similar.
Intervals are defined to extract subseries from the whole series.
From each subseries, descriptive statistics (such as the mean, the variance, higher-order moments, quantiles, etc.) are extracted.
Additionally, in several methods, multiple representations (such as the first-order difference and the discrete Fourier transform) of the whole series are used in order to extract more diverse features.
All the extracted features are then concatenated to obtain the design matrix, which is finally fed to a standard machine learning classification.
Tree-based algorithms are the most commonly used classification algorithms with interval-based methods.

The Time Series Forest (TSF) algorithm \citep{dengTimeSeriesForest2013} selects several random intervals, computes three statistics for each interval (the mean, the standard deviation and the slope) and all the features are concatenated to train a decision tree.
This process is repeated several times (with different random intervals and different trees) to build an ensemble model, with the final prediction being obtained using majority voting.

The TSF algorithm has received two extensions.
Supervised Time Series Forest (STSF) \citep{cabelloFastAccurateTime2020} includes several representations of the time series (raw, first-order difference, and discrete Fourier transform) and a supervised method for selecting the most discriminative interval features using a decision tree.
Several decision trees are built using this process, and the final prediction is obtained using majority voting.
Randomized STSF (RSTSF) \citep{cabelloFastAccurateExplainable2024} extends STSF with more randomness, and builds a single design matrix used to train an extremely randomized trees model \citep{geurtsExtremelyRandomizedTrees2006a}.

The Random Interval Spectral Ensemble (RISE) \citep{flynnContractRandomInterval2019} is similar to TSF, but was designed for tackling audio problems, and thus computes spectral features (periodogram and auto-regression), which are known to be discriminatory for such problems, instead of time-domain features.
Contrary to TSF, a single interval is randomly selected for each tree instead of several ones.

The Canonical Interval Forest (CIF) algorithm \citep{middlehurstCanonicalIntervalForest2020} is another extension of TSF with more features being extracted.
In addition to the three features of TSF, the 22 Catch22 features \citep{lubbaCatch22CAnonicalTimeseries2019} are also included.
An ensemble of trees is built on top of the extracted features.
Additional diversity is obtained by considering only a subset of the 25 features for each tree, similarly to what is done in a random forest \citep{breimanRandomForests2001a}.
The Diverse Representation Canonical Interval Forest (DrCIF) algorithm \citep{middlehurstHIVECOTE20New2021a} extends CIF with two additional time series representations: periodograms and first-order differences.

\subsection{Quant}

Quant \citep{dempsterQuantMinimalistInterval2024} is an interval-based method with a transformation step followed by a classification step.
The classification algorithm used is extremely randomized trees \citep{geurtsExtremelyRandomizedTrees2006a}.
Two other algorithms were also investigated in the original publication: random forests \citep{breimanRandomForests2001a} and Ridge \citep{hoerlRidgeRegressionApplications1970}.
For the rest of this section, we focus on the transformation step of Quant.
More generally, we use the term \emph{Quant} for both the transformation step and the pipeline of the transformation and classification steps, with the context implicitly indicating which one.

Quant is an interval-based method with the following characteristics:
\begin{itemize}
    \item It uses four distinct representations of the time series: the raw representation, the (smoothed) first-order difference, the second-order difference, and the discrete Fourier transform (in practice, the magnitudes of the complex Fourier coefficients).
    \item It derives fixed dyadic intervals at several levels.
    Starting from any representation of the whole series, the whole series is considered, then the first half and the second half of the whole series are considered, then the four quarters of the whole series are considered, etc.
    We call these intervals the base intervals.
    Moreover, Quant also includes shifted intervals for any level excluding the first one, with the shift being equal to half the interval length.
    The \emph{depth}, that is the number of levels that are considered, and denoted by $d$, is a hyperparameter of Quant, and its default value is $d=6$.
    \autoref{fig:quant_intervals} illustrates the set of intervals considered by Quant.
    \item The descriptive statistics computed are quantiles.
    Instead of using a fixed number of quantiles per interval, which would be independent of the interval length, Quant defines a \emph{quantile divisor} such that the number of quantiles is proportional to the interval length.
    Denoting by $m$ the length of any interval and $\nu$ the quantile divisor, Quant computes $k = m / \nu$ evenly-spaced quantiles from the minimum to the maximum, that is the $(0, 1 / (k-1), \ldots, (k - 2)/(k - 1), 1)$ quantiles, per interval.
    \item The interval mean is subtracted to every other quantile.
    The intuition is that the extracted features (that is, the quantiles) represent both the distribution of the values in the interval and the distribution of values in the interval relative to the mean.
    \autoref{fig:quant_quantiles} illustrates the quantiles extracted by Quant.
    This idea has already been used in other families of time series classification algorithms, notably in dictionary-based methods where the subseries are often normalized in order to consider the relative (to the interval) values of the subseries instead of the absolute values.
\end{itemize}

\begin{figure}
    \includegraphics[width=\textwidth]{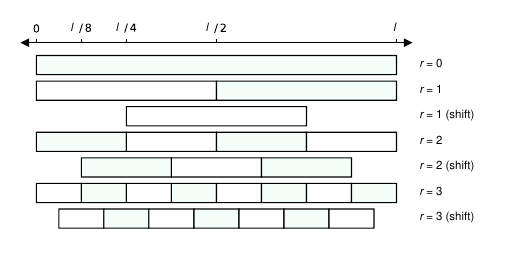}
    \caption{
        Illustration of the set of intervals that are considered by Quant.
        In this example, the value of the depth is $d=4$.
        For every level $r \in \{0, \ldots, d - 1\}$, base intervals of length $l \cdot 2^{-r}$ are considered.
        For every level $r \in \{1, \ldots, d - 1\}$, shifted intervals of length $l \cdot 2^{-r}$ are considered, with the shift being equal to half the interval length (that is, $l \cdot 2^{-r - 1}$).
        The figure is adapted from \citep{dempsterQuantMinimalistInterval2024} and the corresponding author has allowed its reuse.}
    \label{fig:quant_intervals}
\end{figure}

\begin{figure}
    \includegraphics[width=\textwidth]{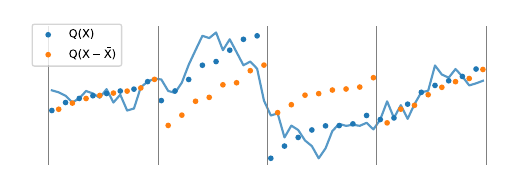}
    \caption{
        Illustration of the quantiles derived by Quant.
        This example has 4 intervals, corresponding to the base intervals at the dyadic level $r = 2$.
        Quantiles are (possibly interpolated) values of the intervals.
        Every other quantile is centered, which is illustrated by the color of the dots: blue dots represent raw quantiles, while orange dots represent centered quantiles.
        Raw quantiles quantify the distribution of the absolute values in the interval, while centered quantiles quantify the distribution of the relative values in the interval.
        The figure is from \citep{dempsterQuantMinimalistInterval2024} and the corresponding author has allowed its reuse.}
    \label{fig:quant_quantiles}
\end{figure}

In practice, there are a few more subtleties, and we believe that it is important to mention them:
\begin{itemize}
    \item The four representations of a series do not have the same length.
    If $l$ is the length of the raw representation, then the lengths of the (smoothed) first-order difference, the second-order difference and the discrete Fourier transform are $l - 1$, $l - 2$, and $\lfloor l / 2 \rfloor + 1$ respectively, where $\lfloor \cdot \rfloor$ is the floor function.
    \item The first-order difference is smoothed with a length-$5$ centered moving average, applied after padding both ends of the differenced series by $2$ using edge-value (replicate) padding, which restores its original (post-differencing) length.
    \item The depth of Quant is capped by the series length.
    Indeed, it would be impossible to split a subseries of length 1.
    In practice, if $\lfloor \log_2(l) \rfloor < d - 1$, then the value of the depth that is used is $\lfloor \log_2(l) \rfloor + 1$, and otherwise the provided $d$ value is used.
    In the general case, the value of the depth used is thus:
    $$
        e = \min \left( d, \lfloor \log_2(l) \rfloor + 1 \right)
    $$
    For $d=6$, it means that the series of length $l < 32$ have fewer than $6$ dyadic levels.
    \item All the intervals of any level $r \in \{1, \ldots, e - 1\}$ are not exactly equal-sized, unless $l$ is exactly divisible by $2^r$.
    Given the implementation chosen in Quant, the lengths of any two intervals of the same level differ by at most $\pm 1$.
    Let $q_r$ and $s_r$ be the quotient and the remainder of the Euclidean division of $l$ by $2^r$ respectively:
    $$
        q_r = \lfloor l \cdot 2^{-r} \rfloor \qquad \text{and} \qquad s_r = l - q_r \cdot 2^r
    $$
    There are exactly $2^r - s_r$ intervals of length $q_r$ and $s_r$ intervals of length $q_r + 1$.
    \item The exact value for the shift of the shifted intervals of any level $r \in \{1, \ldots, e - 1\}$ is $\lceil l \cdot 2^{-r - 1} \rceil$.
    \item At the last level $r = e - 1$, shifted intervals are included if and only if the median of the base interval lengths of this level is greater than $1$, which is equivalent to $l \geq 1.5 \cdot 2^{e - 1}$.
    \item There are three different cases (two edge cases and the normal case) to compute the quantiles:
    \begin{enumerate}
        \item If a subseries is of length $m = 1$, the single value of the subseries is returned.
        \item If a subseries is of length $m$ larger than $1$ and lower than or equal to the quantile divisor $\nu$, that is $1 < m \leq \nu$, then only the median is computed (and no centering is performed).
        \item Otherwise, the actual number of quantiles is equal to $1 + \lfloor (m - 1) / \nu \rfloor$.
        However, in the general case, the quantiles are not exact values from the subseries.
        For instance, consider a subseries of length $11$.
        If the number of quantiles is equal to $5$, the three quartiles would correspond to the following indices (using zero-based numbering) of the sorted values: $2.5$, $5$, and $7.5$.
        However, two indices are not integers ($2.5$ and $7.5$).
        When an index is not an integer, linear interpolation is performed between the lower and upper bounds.
    \end{enumerate}
\end{itemize}

\section{Limitations of Quant's reference implementation}\label{sec3}

Quant's reference implementation\footnote{\url{https://github.com/angus924/quant}} is written in Python and commits to a single, fixed computational structure for extracting interval quantiles.
First, given the length of each time series $l$ and the maximum depth $d$, the fixed endpoints of all the intervals are computed (not shown).
Each interval issues one call to the \mintinline{py}{f_quantile} function, which invokes the \mintinline{py}{torch.quantile} function on the entire batch of time series at once, as illustrated in \autoref{listing:1}.
This process is repeated for all the intervals using a for loop (over all the intervals), as illustrated in \autoref{listing:2}.

\begin{listing}
    \begin{minted}{py}
    def f_quantiles(X, div=4):
        m = X.shape[-1]  # m is the interval width
        num_quantiles = 1 + (m - 1) // div  # Number of quantiles computed
        if num_quantiles == 1:  # If a single quantile is computed
            # Compute the median
            quantiles = X.quantile(torch.tensor([0.5]), dim=-1)
        else:  # If several quantiles are computed
            # Compute the quantiles for the base intervals
            quantiles = X.quantile(torch.linspace(0, 1, num_quantiles), dim=-1)
            # Center every other quantile
            quantiles[..., 1::2] = quantiles[..., 1::2] - X.mean(-1)
        return quantiles
    \end{minted}
    \caption{
        Simplified version of the \mintinline{py}{f_quantiles} Python function used to compute the quantiles for all the samples for a given interval in Quant's reference implementation.
    }
    \label{listing:1}
\end{listing}

\begin{listing}
    \begin{minted}{py}
    features = []  # Instantiate an empty list
    for a, b in intervals:  # For each interval
         # Compute the quantiles for this interval
        features.append(f_quantile(X[..., a:b], div = self.div))
    # Concatenate all the quantiles into a single tensor
    features = torch.cat(features, -1)
    \end{minted}
    \caption{
        Simplified version of the for loop, over the intervals, used to compute the quantiles for all the samples and all the intervals in Quant's reference implementation.
    }
    \label{listing:2}
\end{listing}

This implementation has an obvious utility for researchers: it is extremely simple (very few lines of code) and easily readable.
Nonetheless, it might not be optimal in practice.
Indeed, this design amortizes the cost of every quantile computation across the full batch of samples, and is a natural choice for a tensor library built around batched, hardware-accelerated execution such as graphical processing units (GPU).
On CPU, however, this single fixed structure exposes two distinct inefficiencies at opposite ends of the series-length spectrum, both of which stem from the same root cause: the number of \mintinline{py}{torch.quantile} invocations is fixed by the total number of intervals (which is bounded by the value of $d$ in practice), independent of how much actual numerical work each invocation performs.

For short series, interval widths are small, so each \mintinline{py}{torch.quantile} call does very little genuine sorting and interpolation work.
What dominates instead is the fixed, per-call cost of going through PyTorch's dispatcher.
This overhead is paid once per interval, regardless of interval width, and is largely unavoidable within PyTorch's general-purpose tensor-operation model.

For long series, the opposite failure mode appears.
For the first depth level, the quantiles for the whole time series are computed.
The \mintinline{py}{torch.quantile} function sorts the input to compute the quantiles.
This sorting operation has an $\mathcal{O}(l \cdot \log l)$ cost just to sort the whole time series.
This cost does not even include sorting all the subseries extracted.

Quant's fixed structure is not well-matched to either extreme of time series length by design: it always pays per-call dispatch overhead proportional to the total number of intervals (which is the predominant term for short series), and always pays a total sort cost (which is the predominant term for long series), with no mechanism to trade one off against the other.

We propose two distinct solutions to improve the runtimes of Quant in both settings (short and long series).
But first, we provide an analysis of the theoretical complexity of processing a single time series.

\section{Complexity of processing a single series}\label{sec4}

In this section, we investigate the theoretical complexity of processing a single time series.
First, we focus on the processing of only the raw series.
Then, we explain how to generalize the results for the whole processing, which includes the four different representations of the series.
For any given representation, there are two dominant terms in the processing: sorting the values in each interval and extracting the interpolated (and possibly centered) quantiles from the sorted values in each interval.
To keep the main text focused on the results themselves, the proofs of every theorem and lemma in this section, as well as in \autoref{sec5} and \autoref{sec6}, are deferred to Appendix~\ref{secA1}.

\subsection{Processing the raw series}

\subsubsection{Preliminary results}

Before diving into the total cost of both dominant terms, we provide preliminary results about the number of intervals derived by Quant (\autoref{theorem:1}), the total width of all these intervals (\autoref{theorem:2}), the total number of quantiles extracted by Quant (\autoref{theorem:3} and \autoref{theorem:4}), and the number of length-one intervals (\autoref{theorem:5}).

\begin{theorem}[Number of intervals]
    \label{theorem:1}
    For any series length $l$ and any depth $d$, define $k = \min(d - 1, \lfloor \log_2(l) \rfloor)$.
    The total number of intervals that Quant builds, denoted by $N_i(l, d)$, is equal to:
    $$
        N_i(l, d) := \begin{cases}
            2^{k+2} - k - 3 & \text{if } l \geq 1.5 \cdot 2^k \\
            3 \cdot 2^k - k - 2 & \text{if } l < 1.5 \cdot 2^k
        \end{cases} = \Theta \left( 2^k \right)
    $$
\end{theorem}

\begin{proof}
    See the proof of \autoref{theorem:1} on page~\pageref{proof:theorem:1} of Appendix~\ref{secA1}.
\end{proof}

\begin{theorem}[Total width]
    \label{theorem:2}
    For any series length $l$ and any depth $d$, define $k = \min(d - 1, \lfloor \log_2(l) \rfloor)$.
    The total width, that is the sum of the widths of all the base and shifted intervals, and denoted by $W(l, d)$, is equal to:
    $$
        W(l, d) := \begin{cases}
            \displaystyle \left( 2k + 2^{-k} \right) \cdot l - \sum_{r=1}^k \left( \left\lceil l \cdot 2^{-r} \right\rceil - l \cdot 2^{-r} \right) & \text{if } l \geq 1.5 \cdot 2^k \\
            \displaystyle \left( 2k - 1 + 2^{-(k-1)} \right) \cdot l - \sum_{r=1}^{k-1} \left( \left\lceil l \cdot 2^{-r} \right\rceil - l \cdot 2^{-r} \right) & \text{if } l < 1.5 \cdot 2^k
        \end{cases} = \Theta \left( k \cdot l \right)
    $$
\end{theorem}

\begin{proof}
    See the proof of \autoref{theorem:2} on page~\pageref{proof:theorem:2} of Appendix~\ref{secA1}.
\end{proof}

\begin{theorem}[Total number of quantiles]
    \label{theorem:3}
    Let $\{x\}$ be the fractional part of any positive real number $x$, that is $\{x\} = x - \lfloor x \rfloor$ with $0 \leq \{x\} < 1$.
    For any series length $l$, any depth $d$, and any quantile divisor $\nu$, the total number of quantiles, denoted by $N_q(l, d, \nu)$, is equal to:
    $$
        N_q(l, d, \nu) := N_i(l, d) + \frac{W(l, d) - N_i(l, d)}{\nu} - \sum_j \left\{ \frac{m_j - 1}{\nu} \right\}
    $$
    with the sum being over all the base and shifted intervals and $m_j$ being the width of the $j$-th interval.
\end{theorem}

\begin{proof}
    See the proof of \autoref{theorem:3} on page~\pageref{proof:theorem:3} of Appendix~\ref{secA1}.
\end{proof}

We denote by $E(l, d, \nu)$ the residual term in the formula of $N_q(l, d, \nu)$:
$$
    E(l, d, \nu) := \sum_j \left\{ \frac{m_j - 1}{\nu} \right\} \geq 0
$$
The residual term $E(l, d, \nu)$ is always non-zero (except in trivial settings), but a simple bound can be derived.
We state these results in \autoref{theorem:4}.

\begin{theorem}[Value of the residual term]
    \label{theorem:4}
    For $\nu = 1$, and for any $l$ and $d$, the residual term $E(l, d, \nu)$ is always zero.
    For $d = 1$, the residual term is zero if and only if $l - 1$ is exactly divisible by $\nu$.
    For any $d \geq 2$ and $\nu \geq 2$, and for any $l$, the residual term $E(l, d, \nu)$ is always non-zero, is bounded by $\frac{\nu - 1}{\nu} \cdot N_i(l, d)$, with the bound being attained for any $l$ being exactly divisible by $2^{d-1} \cdot \nu$.
\end{theorem}

\begin{proof}
    See the proof of \autoref{theorem:4} on page~\pageref{proof:theorem:4} of Appendix~\ref{secA1}.
\end{proof}

Among the intervals counted by \autoref{theorem:1}, those of width $1$ play no role in the sort and extraction costs studied later in this section: sorting a single value is a no-op, and no quantile is interpolated from it (the value is simply copied). \autoref{theorem:5} counts these width-$1$ intervals exactly, which lets the extraction-cost results in this section (\autoref{theorem:9}, \autoref{theorem:16} and their approximate-algorithm counterparts) be stated exactly, rather than as small-series approximations.

\begin{theorem}[Number of length-one intervals]
    \label{theorem:5}
    For any series length $l$ and any depth $d$, define $k = \min(d - 1, \lfloor \log_2(l) \rfloor)$.
    The number of intervals of width $1$ that Quant builds, denoted by $N_i^{(1)}(l, d)$, is equal to:
    $$
        N_i^{(1)}(l, d) := \begin{cases}
            0 & \text{if } k = d - 1 < \lfloor \log_2(l) \rfloor \\
            2^{k+1} - l & \text{if } k = \lfloor \log_2(l) \rfloor \text{ and } l < 1.5 \cdot 2^k \\
            2 \cdot \left( 2^{k+1} - l \right) & \text{if } k = \lfloor \log_2(l) \rfloor \text{ and } l \geq 1.5 \cdot 2^k
        \end{cases}
    $$
    Consequently, denoting by $N_i^{>1}(l, d) := N_i(l, d) - N_i^{(1)}(l, d)$ the number of intervals of width strictly greater than $1$, and by $W^{>1}(l, d) := W(l, d) - N_i^{(1)}(l, d)$ and $N_q^{>1}(l, d, \nu) := N_q(l, d, \nu) - N_i^{(1)}(l, d)$ the total width and total number of quantiles restricted to those intervals, all three quantities are fully determined by \autoref{theorem:1}, \autoref{theorem:2}, \autoref{theorem:3} and $N_i^{(1)}(l, d)$: each width-$1$ interval contributes exactly $1$ to $W(l, d)$ (its own width) and exactly $n_q(1, \nu) = 1$ to $N_q(l, d, \nu)$ (\autoref{theorem:3}'s proof), so subtracting $N_i^{(1)}(l, d)$ removes exactly their contribution from each.
\end{theorem}

\begin{proof}
    See the proof of \autoref{theorem:5} on page~\pageref{proof:theorem:5} of Appendix~\ref{secA1}.
\end{proof}

\subsubsection{Total sort}

In order to compute the total sort cost, we start with the ideal case in which the series length $l$ is exactly divisible by the maximum factor $2^k$, making all the intervals of each level have the exact same length.
The results are provided in \autoref{theorem:6}.
We provide a generalization for any arbitrary series length $l$ in \autoref{theorem:7}.
We finally provide the computational complexity of the total sort in \autoref{theorem:8}.

\begin{theorem}[Total sort cost for $l$ exactly divisible by $2^k$]
    \label{theorem:6}
    For any series length $l$ and any depth $d$, define $k = \min(d - 1, \lfloor \log_2(l) \rfloor)$.
    If $l$ is exactly divisible by $2^k$, under the sort-cost accounting where an interval of size $m$ costs $m \cdot \log_2(m)$, the total cost of sorting all the intervals that Quant builds for the raw representation of one series is:
    $$
        T_s^{r*}(l, d) := c_1 \cdot l \cdot \left[ \left( 2k + 2^{-k} \right) \cdot \log_2(l) - \left( k^2 + k - 2 + (k + 2) \cdot 2^{-k} \right) \right]
    $$
    with $c_1$ being the empirical time-per-unit-of-sort-work constant (in seconds), which is positive and hardware-specific.
\end{theorem}

\begin{proof}
    See the proof of \autoref{theorem:6} on page~\pageref{proof:theorem:6} of Appendix~\ref{secA1}.
\end{proof}

\begin{theorem}[Total sort cost for arbitrary $l$]
    \label{theorem:7}
    For any series length $l$ and any depth $d$, define $k = \min(d - 1, \lfloor \log_2(l) \rfloor)$.
    Under the sort-cost accounting where an interval of size $m$ costs $m \cdot \log_2(m)$, the total cost of sorting all the intervals that Quant builds for the raw representation of one series is:
    $$
        T_s^r(l, d) := c_1 \cdot l \cdot \left[ \left( 2k + 2^{-k} \right) \cdot \log_2(l) - \left( k^2 + k - 2 + (k + 2) \cdot 2^{-k} \right) \right] + \mathcal{O} \left( 2^k \right) + \mathcal{O}\left( k \cdot \log(l) \right)
    $$
\end{theorem}

\begin{proof}
    See the proof of \autoref{theorem:7} on page~\pageref{proof:theorem:7} of Appendix~\ref{secA1}.
\end{proof}

\begin{theorem}[Computational complexity of the total sort]
    \label{theorem:8}
    For any series length $l$ and any depth $d$, define $e = \min(d, \lfloor \log_2(l) \rfloor + 1)$.
    Under the sort-cost accounting where an interval of size $m$ costs $m \cdot \log_2(m)$, the computational complexity of sorting all the intervals that Quant builds for one series is:
    $$
        T_s^r(l, d) = \Theta \left( l \cdot e \cdot \log(l) \right)
    $$
    Equivalently, splitting by which term wins in the minimum in the formula of $e$:
    $$
        T_s^r(l, d) = \begin{cases}
            \Theta \left( d \cdot l \cdot \log(l) \right) & \textnormal{if } l \geq 2^{d-1} \text{ (saturated regime)}\\
            \Theta \left( l \cdot \left( \log(l) \right)^2 \right) & \textnormal{if } l < 2^{d-1} \text{ (unsaturated regime)}
        \end{cases}
    $$
\end{theorem}

\begin{proof}
    See the proof of \autoref{theorem:8} on page~\pageref{proof:theorem:8} of Appendix~\ref{secA1}.
\end{proof}

\subsubsection{Total extraction}

We now provide the results about the total extraction cost.
The general results are stated in \autoref{theorem:9}.
We finally provide the computational complexity of the total extraction in \autoref{theorem:10}.

\begin{theorem}[Total extraction cost]
    \label{theorem:9}
    The total extraction cost for the raw representation of a single series of length $l$, denoted by $T_e^r(l, d, \nu)$, is:
    $$
       T_e^r(l, d, \nu) := c_2 \cdot N_q^{>1}(l, d, \nu) + c_3 \cdot W^{>1}(l, d)
    $$
    with $N_q^{>1}(l, d, \nu)$ and $W^{>1}(l, d)$ as defined in \autoref{theorem:5}, $c_2$ being the empirical per-quantile extraction cost (in seconds/quantile), and $c_3$ being the empirical per-element extraction cost (in seconds/element), both constants being positive and hardware-specific.
\end{theorem}

\begin{proof}
    See the proof of \autoref{theorem:9} on page~\pageref{proof:theorem:9} of Appendix~\ref{secA1}.
\end{proof}

\begin{theorem}
    \label{theorem:10}
    For any series length $l$, any depth $d$, define $e = \min(d, \lfloor \log_2 l \rfloor + 1)$.
    For any positive integer $\nu$ (even as a function of $l$ and $d$), the computational complexity of the total extraction for one series does not depend on $\nu$ and is equal to:
    $$
        T_e^r(l, d, \nu) = \Theta(e \cdot l)
    $$
    Equivalently, splitting by which term wins in the minimum in the formula of $e$:
    $$
        T_e^r(l, d, \nu) = \begin{cases}
            \Theta(d \cdot l) & \textnormal{if } l \geq 2^{d-1} \text{ (saturated regime)}\\
            \Theta(l \cdot \log(l)) & \textnormal{if } l < 2^{d-1} \text{ (unsaturated regime)}
        \end{cases}
    $$
\end{theorem}

\begin{proof}
    See the proof of \autoref{theorem:10} on page~\pageref{proof:theorem:10} of Appendix~\ref{secA1}.
\end{proof}

\subsubsection{Total processing cost}

We first state, in \autoref{lemma:1}, the decomposition of the total processing cost into the total sort cost and the total extraction cost.
This assumption was implicit since the start of this section, but we make it explicit here so that it can be cross-referenced directly.
We then provide the computational complexity of the total processing in \autoref{theorem:11}.

\begin{lemma}[Total processing cost decomposition]
    \label{lemma:1}
    For any series length $l$, depth $d$, and quantile divisor $\nu$, the total processing cost for the raw representation of a single series is the sum of the total sort cost and the total extraction cost:
    $$
        T^r(l, d, \nu) = T_s^r(l, d) + T_e^r(l, d, \nu)
    $$
\end{lemma}

\begin{proof}
    See the proof of \autoref{lemma:1} on page~\pageref{proof:lemma:1} of Appendix~\ref{secA1}.
\end{proof}

\begin{theorem}[Computational complexity of the total processing]
    \label{theorem:11}
    For any series length $l$, any depth $d$, and any quantile divisor $\nu$, under the sort-cost accounting where an interval of size $m$ costs $m \cdot \log_2(m)$, the computational complexity of Quant's processing of the raw representation of a single series is:
    $$
        T^r(l, d, \nu) = \begin{cases}
            \Theta \left( d \cdot l \cdot \log(l) \right) & \textnormal{if } l \geq 2^{d-1} \text{ (saturated regime)}\\
            \Theta \left( l \cdot \left( \log(l) \right)^2 \right) & \textnormal{if } l < 2^{d-1} \text{ (unsaturated regime)}
        \end{cases}
    $$
\end{theorem}

\begin{proof}
    See the proof of \autoref{theorem:11} on page~\pageref{proof:theorem:11} of Appendix~\ref{secA1}.
\end{proof}

\subsubsection{Simplified expressions with more assumptions}

All the results presented above are general, for any value of $d$ and $\nu$.
Quant has default values for both its hyperparameters: the depth is equal to 6 ($d = 6$) and the quantile divisor is equal to 4 ($\nu = 4$).
These default values were used on all the data sets in the main experiments presented in \citep{dempsterQuantMinimalistInterval2024}.
In the rest of this section, we fix the values of both hyperparameters to their default values.
We only consider the saturated regime, which is attained for any series length $l \geq 32$, and also implies that $k = d - 1 = 5$.
Indeed, $32$ is a small value for series length, and one of the main advantages of Quant is to be fast, especially for long series.
Finally, to simplify the formulae even further, we assume the series length to be a power of $2$, that is $l = 2^b$ with $b \geq 5$ being an integer.

With these assumptions, we provide the simplified exact formulae for the total sort cost, the total extraction cost, and the total processing cost for the raw representation of a single series, in \autoref{theorem:12}, \autoref{theorem:13}, and \autoref{theorem:14} respectively.
They illustrate how the theoretical costs provided in \autoref{theorem:6}, \autoref{theorem:7}, and \autoref{theorem:9} can be turned into concrete runtimes.

\begin{theorem}[Total sort cost]
    \label{theorem:12}
    Assuming that $d=6$, and $l = 2^b$ for any positive integer $b$ greater than or equal to $5$, the exact formula for the total sort cost simplifies to:
    $$
        T_s^r \left( 2^b, 6 \right) = 3 \cdot c_1 \cdot 2^{b-5} \cdot \left( 107 \cdot b - 301 \right)
    $$
\end{theorem}

\begin{proof}
    See the proof of \autoref{theorem:12} on page~\pageref{proof:theorem:12} of Appendix~\ref{secA1}.
\end{proof}

\begin{theorem}[Total extraction cost]
    \label{theorem:13}
    Assuming that $d=6$, $\nu=4$ and $l = 2^b$ for any positive integer $b$ greater than or equal to $5$, the exact formula for the total extraction cost simplifies to:
    $$
        T_e^r \left(2^b, 6, 4 \right) = \begin{cases}
            \displaystyle 80 \cdot c_2 + 258 \cdot c_3 & \text{if } b = 5\\
            \displaystyle 192 \cdot c_2 + 642 \cdot c_3 & \text{if } b = 6\\
            \displaystyle 321 \cdot \left( 2^{b-7} \cdot c_2 + 2^{b-5} \cdot c_3 \right) & \text{if } b \geq 7
        \end{cases}
    $$
\end{theorem}

\begin{proof}
    See the proof of \autoref{theorem:13} on page~\pageref{proof:theorem:13} of Appendix~\ref{secA1}.
\end{proof}

\begin{theorem}[Total processing cost]
    \label{theorem:14}
    Assuming that $d=6$, $\nu=4$ and $l = 2^b$ for any positive integer $b$ greater than or equal to $5$, the exact formula for the total processing cost simplifies to:
    $$
        T^r\left( 2^b, 6, 4 \right) = \begin{cases}
            \displaystyle 702 \cdot c_1 + 80 \cdot c_2 + 258 \cdot c_3 & \text{if } b = 5\\
            \displaystyle 2046 \cdot c_1 + 192 \cdot c_2 + 642 \cdot c_3 & \text{if } b = 6\\
            \displaystyle 2^{(b-5)} \cdot \left[ 3 \cdot c_1 \cdot (107 \cdot b - 301) + \frac{321}{4} \cdot c_2 + 321 \cdot c_3 \right] & \text{if } b \geq 7
        \end{cases}
    $$
\end{theorem}

\begin{proof}
    See the proof of \autoref{theorem:14} on page~\pageref{proof:theorem:14} of Appendix~\ref{secA1}.
\end{proof}

\subsection{Generalizing to the four representations}

So far, every result in this section has been stated for the raw representation of a single series.
Quant actually extracts features from four representations of each series: the raw series itself, a smoothed first-order difference, the second-order difference, and the magnitude of the real discrete Fourier transform, processed independently and sequentially.
We now generalize \autoref{theorem:1} through \autoref{theorem:14} to the total cost across all four representations.

\begin{lemma}[Representation lengths]
    \label{lemma:2}
    For a raw series of length $l \geq 3$, index Quant's four representations by $p \in \{1, 2, 3, 4\}$ (raw, smoothed first-difference, second-difference, and FFT-magnitude, respectively), and denote by $l_p(l)$ the length of representation $p$. Then:
    $$
        l_1(l) = l, \qquad l_2(l) = l - 1, \qquad l_3(l) = l - 2, \qquad l_4(l) = \left\lfloor \frac{l}{2} \right\rfloor + 1
    $$
\end{lemma}

\begin{proof}
    See the proof of \autoref{lemma:2} on page~\pageref{proof:lemma:2} of Appendix~\ref{secA1}.
\end{proof}

\begin{remark}
    \label{remark:representation-generalization-exact}
    None of \autoref{theorem:1} through \autoref{theorem:11} use any property of the raw series beyond its length $l$: the ``$r$'' superscript on $T_s^r$, $T_e^r$, and $T^r$ is bookkeeping (indicating that these results were first stated for the raw representation), not a restriction on what the results apply to.
    Every one of those results holds, unchanged, for any single sequence that Quant partitions into intervals via the same depth-$d$, divisor-$\nu$ procedure and processes by sorting and extracting quantiles.
    In particular, for each of the other three representations, we just have to replace this representation's own length $l_p(l)$ in place of $l$.
    Moreover, because the same sort-then-extract kernel processes every representation (only the input array differs), the hardware constants $c_1$, $c_2$, $c_3$ are shared across all four representations: they are not representation-specific, unlike the implementation-specific constants introduced in \autoref{theorem:19} and \autoref{theorem:20}.
\end{remark}

\begin{theorem}[Total sort cost across all four representations]
    \label{theorem:15}
    For any series length $l \geq 3$ and depth $d$, under the sort-cost accounting where an interval of size $m$ costs $m \cdot \log_2(m)$, the total cost of sorting all the intervals that Quant builds across all four representations of one series is:
    $$
        T_s(l, d) = \sum_{p=1}^4 T_s^{r_p}\left( l_p(l), d \right)
    $$
    with each term given by \autoref{theorem:7} (applied, per \autoref{lemma:2} and \autoref{remark:representation-generalization-exact}, with $l$ replaced by $l_p(l)$).
\end{theorem}

\begin{proof}
    See the proof of \autoref{theorem:15} on page~\pageref{proof:theorem:15} of Appendix~\ref{secA1}.
\end{proof}

\begin{theorem}[Total extraction cost across all four representations]
    \label{theorem:16}
    For any series length $l \geq 3$, depth $d$, and quantile divisor $\nu$, the total extraction cost across all four representations of one series is:
    $$
        T_e(l, d, \nu) = \sum_{p=1}^4 T_e^{r_p}\left( l_p(l), d, \nu \right) = c_2 \cdot \sum_{p=1}^4 N_q^{>1}\left( l_p(l), d, \nu \right) + c_3 \cdot \sum_{p=1}^4 W^{>1}\left( l_p(l), d \right)
    $$
    with $N_q^{>1}(l_p(l), d, \nu)$ and $W^{>1}(l_p(l), d)$ exactly given by \autoref{theorem:5}, evaluated at $l_p(l)$.
\end{theorem}

\begin{proof}
    See the proof of \autoref{theorem:16} on page~\pageref{proof:theorem:16} of Appendix~\ref{secA1}.
\end{proof}

\begin{theorem}[Total processing cost across all four representations]
    \label{theorem:17}
    For any series length $l \geq 3$, depth $d$, and quantile divisor $\nu$, the total processing cost across all four representations of one series is:
    $$
        T(l, d, \nu) = T_s(l, d) + T_e(l, d, \nu)
    $$
\end{theorem}

\begin{proof}
    See the proof of \autoref{theorem:17} on page~\pageref{proof:theorem:17} of Appendix~\ref{secA1}.
\end{proof}

\begin{theorem}[Computational complexity of the total, across all four representations]
    \label{theorem:18}
    For any series length $l$ and depth $d$, define $e = \min(d, \lfloor \log_2(l) \rfloor + 1)$ as in \autoref{theorem:8} and \autoref{theorem:10}. Then, for any positive $\nu$:
    $$
        T_s(l,d) = \Theta(l \cdot e \cdot \log(l)), \qquad T_e(l,d,\nu) = \Theta(e \cdot l), \qquad T(l,d,\nu) = \Theta(l \cdot e \cdot \log(l))
    $$
    Equivalently:
    $$
        T(l,d,\nu) = \begin{cases}
            \Theta \left( d \cdot l \cdot \log(l) \right) & \textnormal{if } l \geq 2^{d-1} \text{ (saturated regime)}\\
            \Theta \left( l \cdot \left( \log(l) \right)^2 \right) & \textnormal{if } l < 2^{d-1} \text{ (unsaturated regime)}
        \end{cases}
    $$
    i.e., summing across all four representations does not change the complexity class obtained for the raw representation alone in \autoref{theorem:8}, \autoref{theorem:10}, and \autoref{theorem:11}.
\end{theorem}

\begin{proof}
    See the proof of \autoref{theorem:18} on page~\pageref{proof:theorem:18} of Appendix~\ref{secA1}.
\end{proof}

\begin{remark}
    Unlike the extraction cost (\autoref{theorem:16}), which is exact for any $l$ via \autoref{theorem:2} and \autoref{theorem:3}, the closed forms of \autoref{theorem:12}, \autoref{theorem:13}, and \autoref{theorem:14} rely on $l$ being exactly divisible by $2^k$ (in particular, on $l = 2^b$ in the $d=6, \nu=4$ simplification).
    This holds for the raw representation's length $l_1(l) = l$ by assumption, but generally fails for the other three: $l_2(l) = 2^b - 1$ and $l_4(l) = 2^{b-1} + 1$ are odd for every $b \geq 1$, and $l_3(l) = 2^b - 2$, while even, is not itself a power of two for $b \geq 2$.
    Consequently, even under the simplifying assumptions of \autoref{theorem:12} through \autoref{theorem:14} ($d=6$, $\nu=4$, $l=2^b$), only the raw representation's contribution to $T_s(l,d)$, $T_e(l,d,\nu)$, and $T(l,d,\nu)$ admits a closed form as clean as \autoref{theorem:12} through \autoref{theorem:14}.
    The other three representations' contributions are governed by the fully general \autoref{theorem:7} (sort, with its $\mathcal{O}(\cdot)$ correction terms) and \autoref{theorem:9} (extraction, exact but without the simplification of the residual term $E(l,d,\nu)$ available when $l$ is exactly divisible by $2^k$, used in deriving \autoref{theorem:12} through \autoref{theorem:14}).
\end{remark}

\section{Making Quant faster for short series}\label{sec5}

Quant, applied to a data set of $n$ univariate time series of length $l$, consists in processing every (series, interval) ordered pair.
We consider a sequential processing of this data set, meaning that two nested for loops (one over all the series and one over all the intervals) are required to process the whole data set.
The question that remains to be answered is: Which loop should be the outer loop?
We analyze the computational complexities of both approaches to determine in which settings each approach is optimal.

Before comparing each approach, we also make explicit an assumption that was already implicit in \autoref{theorem:6} through \autoref{theorem:14}: the constants $c_1$, $c_2$, and $c_3$ were defined as \emph{empirical, hardware-specific} constants, and nothing in the proofs of \autoref{theorem:7}, \autoref{theorem:9}, or \autoref{theorem:11} assumed a particular way of executing the sort and extraction operations, as only the \emph{number} of operations was counted.
We now make this dependence on the implementation explicit, and derive the total-cost formula for a data set of $n$ series for each of Quant's two natural implementation strategies.

The total cost of either strategy is linear (with an intercept, i.e., affine) in the number of series $n$:
$$
    \text{total cost} = \text{setup cost} + (\text{marginal cost} \times \text{number of series})
$$
As an analogy, one can think of baking cookies: there are constant (i.e., independent of the number of cookies) costs such as preheating the oven, and there are marginal (i.e., proportional to the number of cookies) costs such as making each cookie from the dough and placing it on the baking sheet. Here, a cookie is a single series, and we need to derive the setup cost (coming from the structure of the code) and the marginal cost (processing a single series) of each strategy. As we show below, the two strategies differ sharply in how much of their cost is genuinely a shared, amortizable setup cost.

To state the marginal cost of either strategy, it is convenient to name the sort-work quantity that was left as an anonymous bracket expression in \autoref{theorem:7}. We define:
$$
    \Sigma(l, d) := l \cdot \left[ \left( 2k + 2^{-k} \right) \cdot \log_2(l) - \left( k^2 + k - 2 + (k+2) \cdot 2^{-k} \right) \right] + \mathcal{O}\left( 2^k \right) + \mathcal{O}\left( k \cdot \log(l) \right)
$$
with $k = \min(d-1, \lfloor \log_2(l) \rfloor)$, so that \autoref{theorem:7} reads $T_s^r(l,d) = c_1 \cdot \Sigma(l,d)$. For each implementation strategy $v \in \{I, S\}$ (interval-outer and series-outer, defined below), we write $c_1^{(v)}, c_2^{(v)}, c_3^{(v)}$ for that implementation's own hardware-specific sort and extraction constants, and define its per-series total processing cost as
$$
    T^{r,(v)}(l, d, \nu) := c_1^{(v)} \cdot \Sigma(l,d) + c_2^{(v)} \cdot N_q^{>1}(l,d,\nu) + c_3^{(v)} \cdot W^{>1}(l,d)
$$
i.e., \autoref{lemma:1}'s single-series total cost, evaluated with implementation $v$'s own constants ($N_q^{>1}$ and $W^{>1}$ as in \autoref{theorem:5}, per \autoref{theorem:9}).

\subsection{Outer loop over the intervals}

In Quant's reference implementation, the outer loop is the loop over the intervals: for each of the $N_i(l,d)$ intervals, a single vectorized (NumPy) call either sorts and extracts quantiles from that interval (if its width is greater than $1$) or directly copies its single value (if its width is $1$), across all $n$ series at once. We refer to this strategy as \emph{interval-outer}, or version $(I)$.

\begin{theorem}[Total cost of the interval-outer implementation]
    \label{theorem:19}
    For any number of series $n \geq 1$, series length $l$, depth $d$, and quantile divisor $\nu$, the total wall-clock cost of processing the raw representation of $n$ series with the interval-outer implementation is:
    $$
        T^{r,(I)}(n, l, d, \nu) = N_i^{>1}(l,d) \cdot \kappa^{(I)}(l) + n \cdot T^{r,(I)}(l,d,\nu)
    $$
    with $N_i^{>1}(l,d)$ as in \autoref{theorem:5}, and $\kappa^{(I)}(l) \geq 0$ being the hardware-specific per-interval overhead (the Python loop-iteration cost plus the NumPy call-dispatch cost of one vectorized sort-and-extract call, in seconds), a positive constant for a given $l$.
\end{theorem}

\begin{proof}
    See the proof of \autoref{theorem:19} on page~\pageref{proof:theorem:19} of Appendix~\ref{secA1}.
\end{proof}

\subsection{Outer loop over the series}

The alternative strategy makes the outer loop the loop over the $n$ series: for each series, a single compiled (Numba) call sorts and extracts quantiles from all $N_i(l,d)$ intervals of that series internally, with no further Python-level looping.
We refer to this strategy as \emph{series-outer}, or version $(S)$.

\begin{theorem}[Total cost of the series-outer implementation]
    \label{theorem:20}
    For any $n \geq 1$, $l$, $d$, $\nu$, the total wall-clock cost of processing the raw representation of $n$ series with the series-outer implementation is:
    $$
        T^{r,(S)}(n, l, d, \nu) = n \cdot \left[ \kappa^{(S)}(l) + T^{r,(S)}(l,d,\nu) \right]
    $$
    with $\kappa^{(S)}(l) \geq 0$ being the hardware-specific per-series call-dispatch overhead (in seconds) of invoking the compiled kernel at series length $l$, a positive constant for a given $l$. Unlike $\kappa^{(I)}(l)$, which counts a dispatch call issued once per non-trivial \emph{interval}, $\kappa^{(S)}(l)$ counts a dispatch call issued once per \emph{series}.
    We do not assume, a priori, that this quantity is independent of $l$ (see \autoref{remark:kappa-S-l-dependence} below).
\end{theorem}

\begin{proof}
    See the proof of \autoref{theorem:20} on page~\pageref{proof:theorem:20} of Appendix~\ref{secA1}.
\end{proof}

\begin{remark}
    Unlike $T^{r,(I)}$, the cost $T^{r,(S)}(n,l,d,\nu)$ is exactly proportional to $n$: the series-outer implementation has no analogue of the interval-outer implementation's $n$-independent setup term $N_i(l,d) \cdot \kappa^{(I)}(l)$, because every series requires its own call-dispatch overhead $\kappa^{(S)}(l)$. There is no cost shared, and thus amortizable, across series.
    Conversely, the interval-outer implementation amortizes its per-interval dispatch overhead across all $n$ series at once, at the cost of paying for $N_i(l,d)$ such dispatches regardless of how small $n$ is.
\end{remark}

\begin{remark}[$\kappa^{(S)}(l)$ is not, empirically, $l$-independent]
    \label{remark:kappa-S-l-dependence}
    During our initial experiments, we naturally treated $\kappa^{(S)}$ as a single constant, independent of $l$ and $d$, on the grounds that it counts the dispatch overhead of a single compiled function call, unlike $\kappa^{(I)}(l)$, which is paid once per interval and therefore has an obvious reason to scale with $l$.
    Fitting $\kappa^{(S)}$ this way, however, leaves a residual (observed minus predicted processing cost, divided out over the four representations) that grows smoothly and monotonically with $l$, by roughly two orders of magnitude across the fitting grid, with $r^2 \geq 0.999$ at every individual $l$ (thus a real trend, not measurement noise).
    We therefore model $\kappa^{(S)}(l)$ as depending on $l$, fit with the same $A - B/l$ closed form used for $\kappa^{(I)}(l)$ (\autoref{sec7} and \autoref{sec8}).
    The fit's own $r^2$ against this functional form is low, consistent with a small, close-to-constant residual rather than a strong trend of its own.
    This is different from the approximate algorithm's case, discussed in \autoref{remark:kappa-S-tilde-l-dependence}, where the dependence on $l$ is considerably stronger and this same functional form does not fit well.
    A plausible source of this dependence is that invoking the compiled kernel is not, in practice, a fixed-cost operation: it involves marshalling an $l$-dependent volume of data (the input series and the output feature array) across the Python/compiled-code boundary, a cost that the idealized ``one dispatch, fixed overhead'' argument of \autoref{theorem:20}'s proof does not account for.
    As with $\kappa^{(I)}(l)$'s own closed-form fit, we caution against extrapolating $\kappa^{(S)}(l)$ far outside the interior of the series-length grid that it is fit on: a held-out check of this style of fit shows degraded predictive accuracy near and beyond the grid's boundary.
\end{remark}

\subsection{Comparing both approaches}

\begin{theorem}[Crossover sample size, raw representation]
    \label{theorem:21}
    For the raw representation (\autoref{sec4}) and any $l$, $d$, $\nu$, define the per-series marginal costs $B_I(l,d,\nu) := T^{r,(I)}(l,d,\nu)$ and $B_S(l,d,\nu) := \kappa^{(S)}(l) + T^{r,(S)}(l,d,\nu)$, and the interval-outer setup cost $A_I(l,d) := N_i^{>1}(l,d) \cdot \kappa^{(I)}(l)$ (\autoref{theorem:5}).
    \begin{itemize}
        \item If $B_S(l,d,\nu) \leq B_I(l,d,\nu)$, then $T^{r,(S)}(n,l,d,\nu) \leq T^{r,(I)}(n,l,d,\nu)$ for every $n \geq 1$: the series-outer implementation is at least as fast for every sample size.
        \item If $B_S(l,d,\nu) > B_I(l,d,\nu)$, define the crossover sample size
        $$
            n^*(l,d,\nu) := \frac{A_I(l,d)}{B_S(l,d,\nu) - B_I(l,d,\nu)}
        $$
        Then $T^{r,(S)}(n,l,d,\nu) < T^{r,(I)}(n,l,d,\nu)$ for $n < n^*(l,d,\nu)$, and $T^{r,(I)}(n,l,d,\nu) < T^{r,(S)}(n,l,d,\nu)$ for $n > n^*(l,d,\nu)$.
    \end{itemize}
    In either case, the interval-outer implementation is only worth using once $n$ exceeds a threshold that grows with the number of non-trivial intervals $N_i^{>1}(l,d)$ (hence, other things equal, with $l$).
    For a small enough data set, the series-outer implementation is never worse.
\end{theorem}

\begin{proof}
    See the proof of \autoref{theorem:21} on page~\pageref{proof:theorem:21} of Appendix~\ref{secA1}.
\end{proof}

\begin{remark}
    \autoref{theorem:21} formalizes this section's motivating question, for the raw representation in isolation.
    Whether the series-outer implementation is preferable for \emph{every} sample size, or only up to a crossover $n^*(l,d,\nu)$, is itself an empirical question.
    It depends on the sign of $B_S(l,d,\nu) - B_I(l,d,\nu)$, i.e., on which implementation has the lower per-series marginal cost once its own call-dispatch overhead is included.
    We estimate $\kappa^{(I)}(l)$, $\kappa^{(S)}(l)$, and the version-specific $c_1^{(v)}, c_2^{(v)}, c_3^{(v)}$ empirically in \autoref{sec7} and \autoref{sec8}, and use \autoref{theorem:21} to determine, for realistic values of $l$ and $\nu$, which strategy Quant should use as a function of the number of series $n$, and whether short series are better served by one strategy regardless of $n$.

    In practice, however, Quant processes four representations per series (\autoref{sec4}), not the raw representation alone, and \autoref{theorem:21} is deliberately not generalized to that setting the way \autoref{theorem:33} generalizes its approximate-algorithm counterpart: composing a four-representation crossover from \autoref{theorem:21} would require, for each representation $p$, its own separately-fitted $\kappa^{(I)}(l_p(l))$, and an $l$-dependent quantity such as $\kappa^{(I)}(\cdot)$ is comparatively hard to estimate reliably from a finite series-length grid, and unsafe to extrapolate beyond it.
    Our reference implementation's actual dispatch rule sidesteps this: rather than composing $\kappa^{(I)}(l_p(l))$ representation by representation, it estimates a single, pooled, effectively $l$-independent overhead directly from the wall-clock time of the full four-representation call, fitted directly against measured runtimes rather than assembled from the theorem's own per-operation constants.
    \autoref{theorem:21}'s crossover therefore serves in this article as an independent theoretical cross-check of that practical heuristic's soundness, not as the formula the heuristic itself evaluates.
    This pooled overhead is fit once per (mode, depth, div) triple, at the fixed depth and quantile divisor used throughout this article's experiments (\autoref{sec7}).
    It is not portable to a differently configured instance without being refit, and, since \autoref{remark:kappa-S-l-dependence} shows the underlying per-series overhead genuinely varies with $l$, pooling it into a single scalar is itself a simplification of the same kind \autoref{theorem:21} avoids by keeping $\kappa^{(I)}(l)$ explicit in $l$.
    This is an intentional, practical trade-off, made for the reasons given above, and not an oversight.
\end{remark}

\section{Making an approximate version of Quant for long series using moment-based quantiles}\label{sec6}

\autoref{sec3} identified sorting as the dominant cost for long series, and \autoref{theorem:8} confirmed it: the total sort cost is $\Theta(d \cdot l \cdot \log(l))$ in the saturated regime, strictly worse than linear in $l$.
\autoref{sec5} addressed this without changing what Quant computes, only how the computation is dispatched.
In this section, we instead change the computation itself: we propose an approximate variant of Quant that replaces each interval's exact order statistics, which require sorting, with quantile \emph{estimates} obtained from that interval's first four moments via the Cornish-Fisher expansion, computed in a single linear pass with no sorting at all.
This trades exactness for asymptotic speed, and is intended specifically for long series, where \autoref{theorem:8}'s extra $\log(l)$ factor is most costly and where a single long series already carries enough information for a moment-based approximation to be reasonable.
Whether the resulting approximation error is acceptable in practice is an empirical question addressed in \autoref{sec7} and \autoref{sec8}, not in this section.

\subsection{Background: moment-based quantile approximation}\label{sec6.1}

For a sample of $m$ real values, define its raw central moments $M_2 = \sum_i (x_i - \bar x)^2$, $M_3 = \sum_i (x_i - \bar x)^3$, and $M_4 = \sum_i (x_i - \bar x)^4$, where $\bar x$ is the sample mean.
These, together with $\bar x$ itself and $m$, can be computed in a single pass over the $m$ values using the Welford-Pebay online update.
Maintaining running values of $m$, $\bar x$, $M_2$, $M_3$, $M_4$, each new observation updates all five quantities using only the previous running values and the new observation, with a fixed (i.e., independent of $m$) number of arithmetic operations per update.
Unlike sorting, which requires the whole interval to be available at once and whose cost per element grows with the interval's size (an interval of size $m$ costs $m \cdot \log_2(m)$, i.e. $\log_2(m)$ per element, under the accounting used in \autoref{sec4}), this online update touches each element exactly once and performs the same constant amount of work regardless of $m$ or of how many elements have already been processed.
We additionally track the running minimum and maximum alongside the four moments, at the negligible extra cost of two comparisons per element.

From $(m, M_2, M_3, M_4)$, the sample variance, skewness, and excess kurtosis are obtained via:
$$
    \text{var} = \frac{M_2}{m} \qquad\qquad \text{skew} = \frac{M_3 / m}{\text{std}^3} \qquad\qquad \text{exkurt} = \frac{M_4 / m}{\text{var}^2} - 3
$$
with $\text{std} = \sqrt{\text{var}}$, each a fixed number of arithmetic operations, independent of $m$.

The Cornish-Fisher expansion, introduced in \citep{cornishMomentsCumulantsSpecification1938} and given its now-standard explicit form up to the fourth cumulant in \citep{fisherPercentilePointsDistributions1960}, approximates the quantile at probability $q \in (0, 1)$ of a distribution from estimates of its mean, variance, skewness, and excess kurtosis, by adjusting the corresponding standard normal quantile for the distribution's departure from normality.
Writing $z = \Phi^{-1}(q)$ for the standard normal quantile function evaluated at $q$ (itself computed via a fixed-cost rational approximation, independent of $m$ or of $q$'s specific value beyond selecting which of a small, constant number of branches of the approximation to evaluate), the Cornish-Fisher approximation used in this study is:
$$
    w(z) := z + \frac{z^2 - 1}{6} \cdot \text{skew} + \frac{z^3 - 3z}{24} \cdot \text{exkurt} - \frac{2z^3 - 5z}{36} \cdot \text{skew}^2
$$
$$
    \widehat{Q}(q) := \text{mean} + \text{std} \cdot w(z)
$$
Both $w(z)$ and $\widehat Q(q)$ are, like $\Phi^{-1}(q)$ itself, a fixed number of arithmetic operations, independent of $m$: once the moments of an interval are known, approximating any one quantile from them costs the same, regardless of how large the interval is.

The expansion is undefined at the boundary probabilities $q = 0$ and $q = 1$ (where $\Phi^{-1}$ is infinite), so Quant's exact minimum and maximum are used directly at these two positions instead, whenever they are requested.
They are both already available at no extra asymptotic cost, since they are computed in the same linear pass as the moments.
Only the strictly interior quantile positions of an interval's requested layout (i.e., every requested quantile except these two boundary ones, when present) go through the Cornish-Fisher expansion.

Both $w(z)$ and $\widehat Q(q)$ above are obtained by formally inverting an Edgeworth expansion of the distribution function, so the expansion's classical error theory is inherited from Edgeworth expansion theory.
\citet{hillGeneralizedAsymptoticExpansions1968} generalized the Cornish-Fisher expansion to arbitrary order and to limiting distributions other than the normal, and gave the asymptotic order of the truncation error as a function of the number of cumulant terms retained.
This fourth-moment expansion such as the one above is, in this classical framework, accurate up to a residual term that vanishes as the underlying distribution approaches normality (equivalently, as its standardized skewness and excess kurtosis approach zero).
This classical theory assumes the target distribution sits in a sequence approaching normality (as in a Central Limit Theorem refinement), which is not guaranteed to hold for the empirical distribution of an arbitrary interval's raw values, as used here.
\citet{ulyanovNonAsymptoticResultsCornishFisher2016} instead derive genuinely non-asymptotic error bounds for the Cornish-Fisher expansion, though still under regularity conditions on the target distribution (e.g., a bounded density and finite higher moments) that we neither verify nor assume hold for every interval in this study.

A further, well-documented deficiency of the truncated Cornish-Fisher expansion, independent of the error-magnitude questions above, is that the approximate quantile $\widehat Q(q)$ is not guaranteed to be monotonically increasing in $q$ for every combination of skewness and excess kurtosis.
Monotonicity only holds within a bounded region of the (skewness, excess kurtosis) plane known as the expansion's \emph{domain of validity} \citep{jaschkeCornishFisherExpansionContext2002, maillardUsersGuideCornishFisher2012}.
Using the raw sample skewness and excess kurtosis of a highly non-normal interval can, in principle, fall outside this region and yield a non-monotonic (hence invalid, as a quantile function) approximation on that interval.
\citet{maillardUsersGuideCornishFisher2012} additionally warns against a related, purely practical pitfall: for small samples, the skewness and excess kurtosis \emph{parameters} plugged into $w(z)$ are themselves only estimates of the interval's true skewness and excess kurtosis, and this estimation error compounds with the expansion's own truncation error.
We do not check membership in the domain of validity per interval in this study, and instead treat the resulting approximation error, including any contribution from non-monotonicity, as an empirical question, addressed alongside runtime in \autoref{sec7} and \autoref{sec8}.
Indeed, doing so would add a per-interval cost beyond the fixed number of arithmetic operations counted in \autoref{sec6.2}.
\citet{amedeemanesmeComputationCorrectedCornishFisher2019} propose a corrected variant that re-estimates an effective skewness and excess kurtosis via a response-surface method specifically to stay within the domain of validity and reduce this source of error.
We use the direct, uncorrected expansion of \autoref{sec6.1} throughout this study for its lower, still $O(1)$-per-quantile cost, and leave adopting such a correction to future work.

\subsection{Complexity of processing a single series}\label{sec6.2}

The interval structure Quant builds, and the set of quantile positions requested from each interval, do not depend on how each interval's quantiles are actually computed: \autoref{theorem:1} ($N_i(l,d)$, the number of intervals), \autoref{theorem:2} ($W(l,d)$, the total width), and \autoref{theorem:3} ($N_q(l,d,\nu)$, the total number of quantiles) therefore apply unchanged to the approximate algorithm of this section.
What changes is only how each interval's contribution to the total cost is computed: a linear moment pass instead of a sort, and a Cornish-Fisher evaluation instead of an interpolation from sorted values.
We mark every cost quantity specific to the approximate algorithm with a tilde, to distinguish it from its exact-algorithm counterpart of \autoref{sec4} (e.g. $\widetilde T^r$ versus $T^r$), and introduce three new hardware-specific constants, $\tilde c_1, \tilde c_2, \tilde c_3$, playing the same role for the approximate algorithm that $c_1, c_2, c_3$ played for the exact one.

\subsubsection{Total moment computation cost}

\begin{theorem}[Total moment computation cost]
    \label{theorem:22}
    For any series length $l$ and depth $d$, under the moment-cost accounting where an interval of size $m$ costs exactly $m$ (i.e., $\log_2(m)$ replaced by $1$, relative to the sort-cost accounting of \autoref{theorem:6}), the total cost of computing the moments of all $N_i(l,d)$ intervals that Quant builds for the raw representation of one series is exactly:
    $$
        \widetilde T_m^r(l, d) := \tilde c_1 \cdot W(l, d)
    $$
    with $\tilde c_1$ being the empirical time-per-element moment-update constant (in seconds/element), positive and hardware-specific.
\end{theorem}

\begin{proof}
    See the proof of \autoref{theorem:22} on page~\pageref{proof:theorem:22} of Appendix~\ref{secA1}.
\end{proof}

\begin{remark}
    Unlike \autoref{theorem:7}, \autoref{theorem:22} needs no separate treatment of the case where $l$ is not exactly divisible by $2^k$: the moment-cost accounting is exactly linear in each interval's width, with no analogue of sorting's convexity (the $f(m) = m \log_2(m)$ used in \autoref{theorem:7}'s proof is strictly convex, which is precisely what made rounding the interval widths to integers costly to bound.
    The linear cost model of this section has no such curvature, so \autoref{theorem:2}'s exact closed form for $W(l,d)$ transfers to $\widetilde T_m^r(l,d)$ without any correction term at all).
\end{remark}

\begin{theorem}[Computational complexity of the total moment computation]
    \label{theorem:23}
    For any series length $l$ and any depth $d$, define $e = \min(d, \lfloor \log_2(l) \rfloor + 1)$.
    Under the moment-cost accounting of \autoref{theorem:22}, the computational complexity of computing the moments of all the intervals that Quant builds for one series is:
    $$
        \widetilde T_m^r(l,d) = \Theta(e \cdot l)
    $$
    Equivalently, splitting by which term wins in the minimum in the formula of $e$:
    $$
        \widetilde T_m^r(l,d) = \begin{cases}
            \Theta(d \cdot l) & \textnormal{if } l \geq 2^{d-1} \text{ (saturated regime)}\\
            \Theta(l \cdot \log(l)) & \textnormal{if } l < 2^{d-1} \text{ (unsaturated regime)}
        \end{cases}
    $$
\end{theorem}

\begin{proof}
    See the proof of \autoref{theorem:23} on page~\pageref{proof:theorem:23} of Appendix~\ref{secA1}.
\end{proof}

\subsubsection{Total extraction cost}

\begin{theorem}[Total extraction cost, Cornish-Fisher]
    \label{theorem:24}
    The total extraction cost for the raw representation of a single series of length $l$, under the approximate (moment-based) algorithm, denoted by $\widetilde T_e^r(l,d,\nu)$, is:
    $$
        \widetilde T_e^r(l,d,\nu) := \tilde c_2 \cdot N_q^{>1}(l,d,\nu) + \tilde c_3 \cdot N_i(l,d)
    $$
    with $N_q^{>1}$ as in \autoref{theorem:5}, $\tilde c_2$ being the empirical per-quantile Cornish-Fisher evaluation cost (in seconds/quantile), and $\tilde c_3$ being the empirical per-interval cost of converting an interval's raw moments into (variance, skewness, excess kurtosis) (in seconds/interval), both constants being positive and hardware-specific.
\end{theorem}

\begin{proof}
    See the proof of \autoref{theorem:24} on page~\pageref{proof:theorem:24} of Appendix~\ref{secA1}.
\end{proof}

\begin{remark}[$\tilde c_2$ as an average cost]
    Every one of the $N_q^{>1}(l,d,\nu)$ quantile positions is charged the same rate $\tilde c_2$ in \autoref{theorem:24}, even though the (at most two, per non-trivial interval) boundary positions $q=0$ and $q=1$ use the interval's already-known exact minimum/maximum rather than a genuine Cornish-Fisher evaluation (\autoref{sec6.1}), and are therefore strictly cheaper.
    $\tilde c_2$ is best read as an average per-quantile cost over this mix.
    Since the number of boundary positions is at most $2 \cdot N_i^{>1}(l,d) = \mathcal{O}(l)$, while $N_q^{>1}(l,d,\nu) = \Theta(l)$ or larger (\autoref{theorem:10}'s proof), this simplification changes the constant $\tilde c_2$ but not the asymptotic complexity derived in \autoref{theorem:25} below.
\end{remark}

\begin{theorem}[Computational complexity of the total extraction, Cornish-Fisher]
    \label{theorem:25}
    For any series length $l$, any depth $d$, define $e = \min(d, \lfloor \log_2(l) \rfloor + 1)$.
    For any positive integer $\nu$ (even as a function of $l$ and $d$), the computational complexity of the total (approximate) extraction cost for one series does not depend on $\nu$ and is equal to:
    $$
        \widetilde T_e^r(l,d,\nu) = \Theta(e \cdot l)
    $$
    Equivalently, splitting by which term wins in the minimum in the formula of $e$:
    $$
        \widetilde T_e^r(l,d,\nu) = \begin{cases}
            \Theta(d \cdot l) & \textnormal{if } l \geq 2^{d-1} \text{ (saturated regime)}\\
            \Theta(l \cdot \log(l)) & \textnormal{if } l < 2^{d-1} \text{ (unsaturated regime)}
        \end{cases}
    $$
\end{theorem}

\begin{proof}
    See the proof of \autoref{theorem:25} on page~\pageref{proof:theorem:25} of Appendix~\ref{secA1}.
\end{proof}

\subsubsection{Total processing cost}

\begin{lemma}[Total processing cost decomposition, approximate algorithm]
    \label{lemma:3}
    For any series length $l$, depth $d$, and quantile divisor $\nu$, the total processing cost for the raw representation of a single series, under the approximate algorithm, is the sum of the total moment computation cost and the total extraction cost:
    $$
        \widetilde T^r(l,d,\nu) = \widetilde T_m^r(l,d) + \widetilde T_e^r(l,d,\nu)
    $$
\end{lemma}

\begin{proof}
    See the proof of \autoref{lemma:3} on page~\pageref{proof:lemma:3} of Appendix~\ref{secA1}.
\end{proof}

\begin{theorem}[Computational complexity of the total processing, approximate algorithm]
    \label{theorem:26}
    For any series length $l$, any depth $d$, and any quantile divisor $\nu$, under the moment-cost/Cornish-Fisher accounting of this section, the computational complexity of Quant's approximate processing of the raw representation of a single series is:
    $$
        \widetilde T^r(l,d,\nu) = \begin{cases}
            \Theta(d \cdot l) & \textnormal{if } l \geq 2^{d-1} \text{ (saturated regime)}\\
            \Theta(l \cdot \log(l)) & \textnormal{if } l < 2^{d-1} \text{ (unsaturated regime)}
        \end{cases}
    $$
\end{theorem}

\begin{proof}
    See the proof of \autoref{theorem:26} on page~\pageref{proof:theorem:26} of Appendix~\ref{secA1}.
\end{proof}

\begin{remark}[Asymptotic speedup from moment-based quantiles]
    \label{remark:speedup}
    Combining \autoref{theorem:11} (exact algorithm) and \autoref{theorem:26} (approximate algorithm), for any $l, d, \nu$:
    $$
        \frac{T^r(l,d,\nu)}{\widetilde T^r(l,d,\nu)} = \Theta(\log(l))
    $$
    in both the saturated regime ($\Theta(d \cdot l \cdot \log(l)) / \Theta(d \cdot l)$) and the unsaturated regime ($\Theta(l \cdot (\log(l))^2) / \Theta(l \cdot \log(l))$).
    In both regimes, the moment-based algorithm removes exactly one factor of $\log(l)$ relative to the exact algorithm, and this speedup grows, without bound, as $l$ grows.
    This is consistent with this section's motivation (\autoref{sec3}): the longer the series, the more this trade-off favors the approximate algorithm, while for short series the resulting speedup is small in absolute terms and may not be worth its approximation error, an orthogonal question addressed empirically in \autoref{sec7} and \autoref{sec8}.
\end{remark}

\subsubsection{Generalizing to the four representations}

So far, every result in this subsection has been stated for the raw representation of a single series.
As in \autoref{sec4}, Quant applies the same moment-based procedure to all four representations of a series (the raw series itself, the smoothed first-order difference, the second-order difference, and the magnitude of the real discrete Fourier transform) processed independently and sequentially, with representation $p$'s own length $l_p(l)$ given by \autoref{lemma:2}.
We now generalize \autoref{theorem:22} through \autoref{theorem:26} to the total cost across all four representations, mirroring \autoref{theorem:15} through \autoref{theorem:18}.

\begin{remark}
    \label{remark:representation-generalization-approx}
    None of \autoref{theorem:22} through \autoref{theorem:26} use any property of the raw series beyond its length $l$: exactly as for the exact algorithm (\autoref{remark:representation-generalization-exact}), the ``$r$'' superscript on $\widetilde T_m^r$, $\widetilde T_e^r$, and $\widetilde T^r$ is bookkeeping, not a restriction.
    Every one of those results holds, unchanged, for any single sequence that Quant partitions into intervals via the same depth-$d$, divisor-$\nu$ procedure and processes by computing moments and evaluating the Cornish-Fisher expansion.
    In particular, for each of the other three representations, the only change is to use that representation's own length $l_p(l)$ in place of $l$.
    The moment-based kernel is shared across all four representations (only the input array differs), so the hardware constants $\tilde c_1$, $\tilde c_2$, $\tilde c_3$ are likewise shared, not representation-specific.
\end{remark}

\begin{theorem}[Total moment computation cost across all four representations]
    \label{theorem:27}
    For any series length $l \geq 3$ and depth $d$, the total cost of computing the moments of all the intervals that Quant builds across all four representations of one series is:
    $$
        \widetilde T_m(l,d) = \sum_{p=1}^4 \widetilde T_m^{r_p}\left( l_p(l), d \right)
    $$
    with each term given by \autoref{theorem:22} (applied, per \autoref{lemma:2} and \autoref{remark:representation-generalization-approx}, with $l$ replaced by $l_p(l)$).
\end{theorem}

\begin{proof}
    See the proof of \autoref{theorem:27} on page~\pageref{proof:theorem:27} of Appendix~\ref{secA1}.
\end{proof}

\begin{theorem}[Total extraction cost across all four representations, Cornish-Fisher]
    \label{theorem:28}
    For any series length $l \geq 3$, depth $d$, and quantile divisor $\nu$, the total extraction cost across all four representations of one series is:
    $$
        \widetilde T_e(l,d,\nu) = \sum_{p=1}^4 \widetilde T_e^{r_p}\left( l_p(l), d, \nu \right) = \tilde c_2 \cdot \sum_{p=1}^4 N_q^{>1}\left( l_p(l), d, \nu \right) + \tilde c_3 \cdot \sum_{p=1}^4 N_i\left( l_p(l), d \right)
    $$
    with $N_q^{>1}(l_p(l),d,\nu)$ as in \autoref{theorem:5}, and $N_i(l_p(l),d)$ exactly given by \autoref{theorem:1}, both evaluated at $l_p(l)$.
\end{theorem}

\begin{proof}
    See the proof of \autoref{theorem:28} on page~\pageref{proof:theorem:28} of Appendix~\ref{secA1}.
\end{proof}

\begin{theorem}[Total processing cost across all four representations, approximate algorithm]
    \label{theorem:29}
    For any series length $l \geq 3$, depth $d$, and quantile divisor $\nu$, the total processing cost across all four representations of one series, under the approximate algorithm, is:
    $$
        \widetilde T(l,d,\nu) = \widetilde T_m(l,d) + \widetilde T_e(l,d,\nu)
    $$
\end{theorem}

\begin{proof}
    See the proof of \autoref{theorem:29} on page~\pageref{proof:theorem:29} of Appendix~\ref{secA1}.
\end{proof}

\begin{theorem}[Computational complexity of the total, across all four representations, approximate algorithm]
    \label{theorem:30}
    For any series length $l$ and depth $d$, define $e = \min(d, \lfloor \log_2(l) \rfloor + 1)$ as in \autoref{theorem:23} and \autoref{theorem:25}. Then, for any positive $\nu$:
    $$
        \widetilde T_m(l,d) = \Theta(e \cdot l), \qquad \widetilde T_e(l,d,\nu) = \Theta(e \cdot l), \qquad \widetilde T(l,d,\nu) = \Theta(e \cdot l)
    $$
    Equivalently:
    $$
        \widetilde T(l,d,\nu) = \begin{cases}
            \Theta(d \cdot l) & \textnormal{if } l \geq 2^{d-1} \text{ (saturated regime)}\\
            \Theta(l \cdot \log(l)) & \textnormal{if } l < 2^{d-1} \text{ (unsaturated regime)}
        \end{cases}
    $$
    Summing across all four representations does not change the complexity class obtained for the raw representation alone in \autoref{theorem:23}, \autoref{theorem:25}, and \autoref{theorem:26}.
\end{theorem}

\begin{proof}
    See the proof of \autoref{theorem:30} on page~\pageref{proof:theorem:30} of Appendix~\ref{secA1}.
\end{proof}

\begin{remark}[Asymptotic speedup from moment-based quantiles, across all four representations]
    \label{remark:speedup-full}
    Combining \autoref{theorem:18} (exact algorithm, four representations) and \autoref{theorem:30} (approximate algorithm, four representations), for any $l, d, \nu$:
    $$
        \frac{T(l,d,\nu)}{\widetilde T(l,d,\nu)} = \Theta(\log(l))
    $$
    in both regimes, exactly as in \autoref{remark:speedup} for the raw representation alone: summing across all four representations preserves the single factor of $\log(l)$ that the moment-based algorithm removes.
    Since Quant always processes all four representations of every series, this is the practically relevant comparison, more so than \autoref{remark:speedup}.
\end{remark}

\subsection{Outer loop over the intervals}\label{sec6.3}

As in \autoref{sec5} for the exact algorithm, the approximate algorithm admits two natural implementations, differing in which loop (over the intervals or over the $n$ series) is the outer one:
\begin{itemize}
    \item a NumPy-vectorized implementation that dispatches one call per interval, processing all $n$ series at once for a given representation (\emph{interval-outer}, version $(I)$), and
    \item a compiled (Numba) implementation that dispatches one call per representation, processing all $n$ series and all their intervals internally (\emph{series-outer}, version $(S)$).
\end{itemize}
Both implementations process the four representations independently and sequentially, exactly as in \autoref{theorem:27} through \autoref{theorem:30}.
We reuse \autoref{sec5}'s total-cost decomposition (a per-implementation setup cost, independent of $n$, plus a per-implementation marginal cost, proportional to $n$), applied to the approximate algorithm's own constants: for $v \in \{I, S\}$, define the per-representation marginal cost
$$
    \widetilde T^{r,(v)}(l,d,\nu) := \tilde c_1^{(v)} \cdot W(l,d) + \tilde c_2^{(v)} \cdot N_q^{>1}(l,d,\nu) + \tilde c_3^{(v)} \cdot N_i(l,d)
$$
(i.e., \autoref{lemma:3}'s single-series, single-representation total cost, evaluated with implementation $v$'s own constants, with $N_q^{>1}$ as in \autoref{theorem:5}: only the Cornish-Fisher evaluation term excludes trivial intervals, since both the moment computation (\autoref{theorem:22}, using the inclusive $W(l,d)$) and the moments-to-(variance, skewness, excess kurtosis) conversion (using the inclusive $N_i(l,d)$) are dispatched as single calls covering every interval, trivial or not, see \autoref{theorem:24}), and the total marginal cost across all four representations
$$
    \widetilde T^{(v)}_\Sigma(l,d,\nu) := \sum_{p=1}^4 \widetilde T^{r,(v)}\left( l_p(l), d, \nu \right)
$$

\begin{theorem}[Total cost of the interval-outer implementation, approximate algorithm, across all four representations]
    \label{theorem:31}
    For any number of series $n \geq 1$, series length $l$, depth $d$, and quantile divisor $\nu$, the total wall-clock cost of processing all four representations of $n$ series with the interval-outer implementation of the approximate algorithm is:
    $$
        \widetilde T^{(I)}(n,l,d,\nu) = \sum_{p=1}^4 N_i^{>1}\left( l_p(l), d \right) \cdot \tilde\kappa^{(I)}\left( l_p(l) \right) + n \cdot \widetilde T^{(I)}_\Sigma(l,d,\nu)
    $$
    with $N_i^{>1}$ as in \autoref{theorem:5}, and $\tilde\kappa^{(I)}(l) \geq 0$ being the hardware-specific per-interval overhead (the Python loop-iteration cost plus the NumPy call-dispatch cost of one vectorized Cornish-Fisher extraction call, in seconds), a positive constant for a given $l$.
\end{theorem}

\begin{proof}
    See the proof of \autoref{theorem:31} on page~\pageref{proof:theorem:31} of Appendix~\ref{secA1}.
\end{proof}

\subsection{Outer loop over the series}

\begin{theorem}[Total cost of the series-outer implementation, approximate algorithm, across all four representations]
    \label{theorem:32}
    For any $n \geq 1$, $l$, $d$, $\nu$, the total wall-clock cost of processing all four representations of $n$ series with the series-outer implementation of the approximate algorithm is:
    $$
        \widetilde T^{(S)}(n,l,d,\nu) = n \cdot \left[ \tilde\kappa^{(S)}_\Sigma(l) + \widetilde T^{(S)}_\Sigma(l,d,\nu) \right], \qquad \tilde\kappa^{(S)}_\Sigma(l) := \sum_{p=1}^4 \tilde\kappa^{(S)}\left( l_p(l) \right)
    $$
    with $\tilde\kappa^{(S)}(l) \geq 0$ being the hardware-specific per-series, per-representation call-dispatch overhead (in seconds) of invoking the compiled kernel on a representation of length $l$, a positive constant for a given $l$.
    As with $\kappa^{(S)}(l)$ (\autoref{sec5}, \autoref{remark:kappa-S-l-dependence}), we do not assume $\tilde\kappa^{(S)}(l)$ is independent of $l$ or $d$ (see \autoref{remark:kappa-S-tilde-l-dependence} below).
\end{theorem}

\begin{proof}
    See the proof of \autoref{theorem:32} on page~\pageref{proof:theorem:32} of Appendix~\ref{secA1}.
\end{proof}

\begin{remark}
    As in \autoref{sec5}, $\widetilde T^{(S)}(n,l,d,\nu)$ is exactly proportional to $n$, with no analogue of the interval-outer implementation's amortizable, $n$-independent setup term $\sum_p N_i^{>1}(l_p(l),d) \cdot \tilde\kappa^{(I)}(l_p(l))$: every series requires its own call-dispatch overhead $\tilde\kappa^{(S)}_\Sigma(l)$, paid once per representation, under the series-outer implementation, while the interval-outer implementation amortizes its per-interval dispatch overhead across all $n$ series at once.
\end{remark}

\begin{remark}[$\tilde\kappa^{(S)}(l)$ is not $l$-independent, and does not converge to a constant]
    \label{remark:kappa-S-tilde-l-dependence}
    Unlike the exact algorithm's $\kappa^{(S)}(l)$ (\autoref{remark:kappa-S-l-dependence}), the approximate algorithm's four-representation dispatch cost (\autoref{theorem:32}) is proven, not merely conjectured by analogy from the raw-representation case.
    Fitting the implied per-representation residual $\tilde\kappa^{(S)}(l)$ against measured runtimes shows that treating $\tilde\kappa^{(S)}$ as a single constant does not hold empirically, and more severely than for the exact algorithm: the residual grows smoothly with $l$ by roughly two orders of magnitude across the fitting grid and shows no sign of converging to a constant as $l$ grows.
    Instead, it follows an empirical power law $\tilde\kappa^{(S)}(l) \approx A + C \cdot l^{P}$, with $P \approx 0.74$, validated by leave-one-out cross-validation rather than in-sample fit alone (since it has one more free parameter than the $A - B/l$ form that continues to describe $\kappa^{(S)}(l)$ adequately): the $A - B/l$ form reaches $r^2 \approx 0.18$ in-sample and $\approx 0.10$ under cross-validation on this residual, against $r^2 \approx 0.99$ both in-sample and under cross-validation for the power law (\autoref{sec7} and \autoref{sec8}).
    The exponent $P$ has no theoretical derivation offered here: we report it as an empirical finding, not a new theorem, and state \autoref{theorem:32} above with $\tilde\kappa^{(S)}(l)$ left as a general function of $l$ rather than asserting its independence from $l$.
    We flag two possible, non-exclusive explanations, neither confirmed here.
    First, the same call-marshalling argument as \autoref{remark:kappa-S-l-dependence}, here compounded across four representations of differing lengths.
    Second, that part of the true $l$-dependent moment-computation or Cornish-Fisher extraction cost (\autoref{sec6.2}) is not fully captured by $\tilde c_1, \tilde c_2, \tilde c_3$ and is instead absorbed into the residual this fit labels $\tilde\kappa^{(S)}(l)$.
\end{remark}

\subsection{Comparing both approaches}

\begin{theorem}[Crossover sample size, approximate algorithm, across all four representations]
    \label{theorem:33}
    For any $l$, $d$, $\nu$, define the per-series marginal costs $\widetilde B_I(l,d,\nu) := \widetilde T^{(I)}_\Sigma(l,d,\nu)$ and $\widetilde B_S(l,d,\nu) := \tilde\kappa^{(S)}_\Sigma(l) + \widetilde T^{(S)}_\Sigma(l,d,\nu)$, and the interval-outer setup cost $\widetilde A_I(l,d) := \sum_{p=1}^4 N_i^{>1}(l_p(l),d) \cdot \tilde\kappa^{(I)}(l_p(l))$ (\autoref{theorem:5}).
    \begin{itemize}
        \item If $\widetilde B_S(l,d,\nu) \leq \widetilde B_I(l,d,\nu)$, then $\widetilde T^{(S)}(n,l,d,\nu) \leq \widetilde T^{(I)}(n,l,d,\nu)$ for every $n \geq 1$.
        \item If $\widetilde B_S(l,d,\nu) > \widetilde B_I(l,d,\nu)$, define the crossover sample size
        $$
            \tilde n^*(l,d,\nu) := \frac{\widetilde A_I(l,d)}{\widetilde B_S(l,d,\nu) - \widetilde B_I(l,d,\nu)}
        $$
        Then $\widetilde T^{(S)}(n,l,d,\nu) < \widetilde T^{(I)}(n,l,d,\nu)$ for $n < \tilde n^*(l,d,\nu)$, and $\widetilde T^{(I)}(n,l,d,\nu) < \widetilde T^{(S)}(n,l,d,\nu)$ for $n > \tilde n^*(l,d,\nu)$.
    \end{itemize}
\end{theorem}

\begin{proof}
    See the proof of \autoref{theorem:33} on page~\pageref{proof:theorem:33} of Appendix~\ref{secA1}.
\end{proof}

\begin{remark}
    As in \autoref{sec5}, whether the series-outer implementation of the approximate algorithm is preferable for every sample size, or only up to a crossover $\tilde n^*(l,d,\nu)$, depends on the sign of $\widetilde B_S(l,d,\nu) - \widetilde B_I(l,d,\nu)$, itself an empirical question.
    We estimate $\tilde\kappa^{(I)}(l)$, $\tilde\kappa^{(S)}(l)$, and the version-specific $\tilde c_1^{(v)}, \tilde c_2^{(v)}, \tilde c_3^{(v)}$ empirically in \autoref{sec7} and \autoref{sec8}, alongside the exact algorithm's own constants, and use \autoref{theorem:33} together with \autoref{remark:speedup-full} to determine, for realistic values of $l$, $n$, and $\nu$, whether the exact or the approximate algorithm should be used, and with which outer loop.
\end{remark}

\section{Experimental setup}\label{sec7}

\subsection{Data sets and algorithms}

We use the UCR time series archive \citep{dauUCRTimeSeries2019e} to assess the speed and predictive performance of all the algorithms.
More precisely, we consider the $142$ univariate time series classification data sets from this archive: the 112 univariate equal-length no-missing-value data sets, plus the $30$ data sets of the ``bake off redux'' study \citep{middlehurstBakeReduxReview2024}.
We use the same 30 resamples, for each of the 142 data sets, that have been used in the most recent studies on this topic, making our results comparable to the ones from such studies.

We consider the following eight algorithms:
\begin{itemize}
    \item \mintinline{py}{MomentQuant("exact", "samples")} is our implementation of Quant with exact quantiles using a series-outer loop.
    \item \mintinline{py}{MomentQuant("exact", "intervals")} is our implementation of Quant with exact quantiles using an interval-outer loop.
    \item \mintinline{py}{MomentQuant("exact", "auto")} is our implementation of Quant with exact quantiles using automatically either the series-outer or interval-outer loop.
    \item \mintinline{py}{MomentQuant("approx", "samples")} is our implementation of Quant with approximate moment-based quantiles using a series-outer loop.
    \item \mintinline{py}{MomentQuant("approx", "intervals")} is our implementation of Quant with approximate moment-based quantiles using an interval-outer loop.
    \item \mintinline{py}{MomentQuant("approx", "auto")} is our implementation of Quant with approximate moment-based quantiles using automatically either the series-outer or interval-outer loop.
    \item \mintinline{py}{QuantFloat64} is the original implementation of Quant, in PyTorch, with the only change being the use of double precision instead of single precision.
    \item \mintinline{py}{QuantNumpy} is the one-to-one translation of \mintinline{py}{QuantFloat64} in NumPy, notably using the \mintinline{py}{numpy.quantile} function.
\end{itemize}
For MomentQuant, if the outer loop is not specified, it implicitly means that we refer to the automatic modes: \mintinline{py}{MomentQuant("exact")} refers to \mintinline{py}{MomentQuant("exact", "auto")}, and \mintinline{py}{MomentQuant("approx")} refers to \mintinline{py}{MomentQuant("approx", "auto")}.
Indeed, the whole point of the automatic modes is to try to automatically pick the faster version based on the size of the data set, and they are the default modes in our implementation.
The choice for the outer-loop (\mintinline{py}{"samples"}, \mintinline{py}{"intervals"} or \mintinline{py}{"auto"}) has no impact on the transformation output, and thus on the classification performance, but only on the runtime.
Finally, for classification performance, MomentQuant refers to \mintinline{py}{MomentQuant("approx")}, while Quant refers to either \mintinline{py}{MomentQuant("exact")} or \mintinline{py}{QuantFloat64} (which are theoretically identical).

Compared to Quant's reference implementation, we made several changes in order to improve runtime performance and to perform fair comparisons:
\begin{itemize}
    \item We changed the library used to perform all the mathematical operations, replacing PyTorch with NumPy. Indeed, we had to use Numba to obtain optimal performances for mathematical operations that could not be easily vectorized, and Numba was designed to be used with NumPy arrays. It would not be fair to have two different libraries in the comparisons as these two libraries might potentially have different optimizations and performances.
    \item The PyTorch function doing the most work in Quant's reference implementation, that is \mintinline{py}{torch.quantiles}, was known to be significantly slower than NumPy's equivalent function,\footnote{\url{https://github.com/pytorch/pytorch/issues/64947}} although very recent improvements have been made to improve the performance of \mintinline{py}{torch.quantiles} on CPU.\footnote{\url{https://github.com/pytorch/pytorch/pull/188394}}
    \item We actually did not use the \mintinline{py}{numpy.quantile} function in practice because it is very suboptimal when many quantiles are computed, which is the case in Quant for large $l$ because, at any level, the number of quantiles is proportional to the length of each interval, which itself is proportional to the series length $l$.
    We reported this issue.\footnote{\url{https://github.com/numpy/numpy/issues/32187}}
    \item Furthermore, we changed the precision used, replacing the simple precision with double precision.
    Simple precision is the default data type in PyTorch (and in deep learning in general) as it divides by two the random-access memory required and makes floating-point arithmetic slightly faster.
    In non-deep machine learning, double precision is more common.
    In order to have fair comparisons, we had to use the same precision everywhere, and made the arbitrary decision to use double precision.
\end{itemize}
We provide numerical results highlighting the impact of these changes in \autoref{sec8.1}.

For the classification step, we naturally use extremely randomized trees \citep{geurtsExtremelyRandomizedTrees2006a}, since it is the algorithm used with Quant \citep{dempsterQuantMinimalistInterval2024}.
We also use the default values for the hyperparameters of Quant (depth $d = 6$ and quantile divisor $\nu = 4$) and the same values for the hyperparameters of extremely randomized trees as in \citep{dempsterQuantMinimalistInterval2024}.

\subsection{Metrics and comparisons}

We use the accuracy (ACC) as the main metric in our comparative analyses, but we also report the scores for four other metrics on the 142 UCR data sets in order to make our algorithms easily comparable to existing ones.
These four metrics are balanced accuracy (BALACC), area under the receiver operating characteristic curve (AUROC), negative log-likelihood (NLL) and F1-score (F1).

In order to compare the performance of two algorithms on the 142 data sets, we use paired t-tests and Wilcoxon tests to test the equality of means and medians respectively, with a significance level of $\alpha = 0.05$.

\subsection{Experiments, implementation and reproducibility}

All the experiments were run on a MacBook Air (M1, 2020) with 16 GB RAM using Python 3.13.14.
We used the following Python packages to implement MomentQuant:
\begin{itemize}
    \item \emph{Numba} (0.63.1) \citep{lamNumbaLLVMbasedPython2015} is an open source just-in-time compiler that translates a subset of Python and NumPy code into fast machine code.
    \item \emph{NumPy} (2.3.5) \citep{harrisArrayProgrammingNumPy2020} is a standard Python package for manipulating $n$-dimensional arrays.
    \item \emph{scikit-learn} (1.8.0) \citep{pedregosaScikitlearnMachineLearning2011} is a popular Python package dedicated to machine learning, providing implementations of many algorithms as well as utility tools.
\end{itemize}

Additionally, to perform all the experiments, including saving the results and generating the figures, we also used the following Python packages:
\begin{itemize}
    \item \emph{aeon} (1.5.0) \citep{middlehurstAeonPythonToolkit2024} is a popular Python package for time series machine learning tasks such as classification, regression, clustering, anomaly detection, segmentation and similarity search.
    \item \emph{joblib} (1.5.3) is a set of tools to provide lightweight pipelining in Python, notably transparent disk-caching of functions and lazy re-evaluation, as well as easy simple parallel computing.
    \item \emph{Matplotlib} (3.11.1) \citep{hunterMatplotlib2DGraphics2007} is a comprehensive library for creating static, animated, and interactive visualizations in Python.
    \item \emph{pandas} (2.3.3) \citep{mckinney-proc-scipy-2010} is a fast, powerful, flexible and easy to use open source data analysis and manipulation tool.
    \item \emph{SciPy} (1.17.1) \citep{virtanenSciPy10Fundamental2020}
    \item \emph{seaborn} (0.13.2) \citep{waskomSeabornStatisticalData2021} is a Python data visualization library based on Matplotlib, providing a high-level interface for drawing attractive and informative statistical graphics.
    \item \emph{PyTorch} (2.13.0) \citep{paszkePyTorchImperativeStyle2019} is an optimized tensor library for deep learning.
\end{itemize}

The source code is publicly available on a GitHub repository.\footnote{\url{https://github.com/johannfaouzi/moment-quant}}
We provide detailed instructions so that our results can be easily reproduced by anyone.

\section{Results}\label{sec8}

We present our results in several sections.
\autoref{sec8.1} highlights the impact of the implementation choice on the runtime.
In~\autoref{sec8.2}, we provide empirical results comparing theory and practice.
\autoref{sec8.3} compares the transformation runtimes for different algorithms.
In~\autoref{sec8.4}, we compare the classification performance and the end-to-end runtimes for different algorithms.
\autoref{sec8.5} provides empirical comparisons between the true and approximated, moment-based quantiles.

\subsection{The impact of the implementation choice}\label{sec8.1}

We mentioned in \autoref{sec5} that we reimplemented Quant in NumPy.
We now provide quantitative results justifying this decision.
Our benchmark includes $300$ series at each of $l \in \{512, 1024, \ldots, 32768\}$, and three implementations: \mintinline{py}{QuantFloat64}, \mintinline{py}{QuantNumpy}, and \mintinline{py}{MomentQuant("exact", "intervals")}.
\mintinline{py}{MomentQuant("exact", "intervals")} is implemented in NumPy only, and is very similar to \mintinline{py}{QuantNumpy}, but we compute the quantiles \emph{manually} instead of using the \mintinline{py}{numpy.quantile} function: we use the \mintinline{py}{numpy.sort} function to sort the (sub)series and compute the quantiles based on the sorted values using linear interpolation.
We insist on the fact that all three implementations are theoretically and algorithmically strictly identical.

\autoref{table:implementation_choice_runtimes} provides the measured runtimes.
\mintinline{py}{QuantNumpy} was only timed up to $l=4096$, since it was already unambiguously the slowest implementation well before that point, and its relative runtimes kept growing with $l$.
We also omit multi-threading runtimes for NumPy-based implementations, since the functions involved do not have native multi-threading, so the runtimes for single-threading and multi-threading are identical.

\begin{table}[tbp]
    \caption{Wall-clock runtime of \mintinline{py}{MomentQuant("exact", "intervals")}, \mintinline{py}{QuantFloat64} (1 thread and $8$ threads), and \mintinline{py}{QuantNumpy}, as a function of series length $l$, and each implementation's runtime relative to \mintinline{py}{MomentQuant("exact", "intervals")}'s runtimes.}
    \label{table:implementation_choice_runtimes}
    \centering
    \begin{tabular}{rcccc}
        \toprule
        \makecell{$l$} & \makecell{\mintinline{py}{MomentQuant}\\\mintinline{py}{("exact", "intervals")}} & \makecell{\mintinline{py}{QuantFloat64}\\ (1 thread)} & \makecell{\mintinline{py}{QuantFloat64}\\ (8 threads)} & \makecell{\mintinline{py}{QuantNumpy}} \\
        \midrule
        512   & 0.073s & 0.156s  & 0.137s & 0.348s \\
        1024  & 0.175s & 0.340s  & 0.254s & 1.077s \\
        2048  & 0.397s & 0.768s  & 0.463s & 3.621s \\
        4096  & 0.895s & 1.731s  & 0.786s & 12.915s \\
        8192  & 1.982s & 3.842s  & 1.388s & -- \\
        16384 & 4.329s & 8.490s  & 2.609s & -- \\
        32768 & 9.373s & 18.537s & 4.905s & -- \\
        \midrule
        512   & 1.00$\times$ & 2.14$\times$ & 1.89$\times$ & 4.79$\times$ \\
        1024  & 1.00$\times$ & 1.95$\times$ & 1.45$\times$ & 6.17$\times$ \\
        2048  & 1.00$\times$ & 1.94$\times$ & 1.17$\times$ & 9.13$\times$ \\
        4096  & 1.00$\times$ & 1.93$\times$ & 0.88$\times$ & 14.43$\times$ \\
        8192  & 1.00$\times$ & 1.94$\times$ & 0.70$\times$ & -- \\
        16384 & 1.00$\times$ & 1.96$\times$ & 0.60$\times$ & -- \\
        32768 & 1.00$\times$ & 1.98$\times$ & 0.52$\times$ & -- \\
        \bottomrule
    \end{tabular}
\end{table}

Although the three implementations are theoretically strictly identical, their runtimes are much different.
Focusing on the single-threaded runtimes first, we see that our implementation is around $2$ times faster than the one in PyTorch across all the value of $l$, and much faster than the implementation using the \mintinline{py}{numpy.quantile} function.
After digging into the source code of both libraries, we managed to understand the source of both these differences:
\begin{itemize}
    \item PyTorch's implementation of sorting always performs the sorting of both the values (\mintinline{py}{sort}) and the indices (\mintinline{py}{argsort}), which is not the case in NumPy.
    We show in \autoref{sec9} that, contrary to a naive swap-counting argument, this roughly $2\times$ gap does not come from PyTorch performing more comparisons or swaps than NumPy.
    \item The \mintinline{py}{numpy.quantile} function does not sort all the values, and performs partitions instead.
    Indeed, in order to compute quantiles, it is not necessary to sort all the values: the only values needed are the values before and after each quantile of interest (in order to perform the linear interpolations).
    However, the cost of a partition for a vector of size $m$ is $\Theta(m)$, so the cost of $a$ independent partitions is $\Theta(m \cdot a)$, which is much worse than $\Theta(m \cdot \log(m))$ (the cost of a full sort) for large values of $a$, which is the case in Quant's transformation. Indeed, the number of quantiles computed in Quant is (approximately) proportional to the length of each subseries, so the complexity of using independent partitions is $\Theta(m^2)$, hence the non-linear relative runtime increase for \mintinline{py}{QuantNumpy} when $l$ increases.
\end{itemize}

Regarding multi-threading, we observe that the benefit becomes bigger as the series length increases.
These results are logical: the larger the series, the longer the sorting takes, while the overhead remains constant.
On our machine with 8 threads, \mintinline{py}{QuantFloat64} with multi-threading becomes faster than \mintinline{py}{MomentQuant} for a series length between $2048$ and $4096$.

Taken together, these results illustrate how large a difference implementation choice alone can make to several functions computing the exact same quantities.

\subsection{Theory versus practice}\label{sec8.2}

We derived closed-form cost models for MomentQuant's exact and approx modes in \autoref{sec4} and \autoref{sec6}, and for
both outer-loop strategies (interval-outer and series-outer) in \autoref{sec5} and \autoref{sec6.3}.
These cost models depend on multiple hardware-specific constants:
\begin{itemize}
    \item the sort and extraction constants for the exact mode: $c_1^{(v)}, c_2^{(v)}, c_3^{(v)}$,
    \item the moment update, Cornish-Fisher evaluation and moment conversion for the approx mode: $\tilde c_1^{(v)}, \tilde c_2^{(v)}, \tilde c_3^{(v)}$, and
    \item the per-call dispatch overheads $\kappa^{(I)}(l)$, $\kappa^{(S)}(l)$, $\tilde\kappa^{(I)}(l)$, and $\tilde\kappa^{(S)}(l)$ for the interval-outer and series-outer versions, and the exact and approx modes respectively.
\end{itemize}
These constants say nothing about whether the resulting formulas actually predict real wall-clock time, only that the formulas are internally consistent.
We checked the former directly, for all five estimators that this project has a cost model for: MomentQuant's four (mode, version) combinations and QuantFloat64 (structurally identical to exact/intervals).

\subsubsection{Estimating the hardware-specific constants}

All constants above were estimated from a single, dedicated calibration sweep: series length $l$ crossed with number of series $n$, over a grid spanning $l \in \{16, \ldots, 8192\}$ and $n \in \{1, \ldots, 10\,000\}$, timed single-threaded.
At each $l$, the measured runtime as a function of $n$ is affine (an intercept plus a term linear in $n$, see \autoref{theorem:19} and \autoref{theorem:20}).
Regressing runtime against $n$ at each $l$ separately recovers, per $l$, a per-series marginal cost (the slope) and a fixed per-call overhead (the intercept).

The constants $c_1^{(v)}, c_2^{(v)}, c_3^{(v)}$ are then obtained by a single ordinary-least-squares regression of the per-$l$ marginal costs against the theorem's own basis functions (sort work, number of quantiles extracted, and total interval width), computed exactly from the real interval layout MomentQuant builds internally (not from the power-of-2-restricted closed forms of \autoref{sec4}, which would not apply to every $l$ on the grid).
$\kappa^{(I)}(l)$ and $\kappa^{(S)}(l)$ are obtained from the same per-$l$ intercepts, modeled as $\kappa(l) = A - B/l$ (\autoref{remark:kappa-S-l-dependence}).

The approx-mode constants $\tilde c_1, \tilde c_2, \tilde c_3$ are estimated the same way in principle, but from a dedicated microbenchmark that times the moment-computation and Cornish-Fisher-extraction phases separately, rather than only ever as a fused total.
The two natural regressors for the fused fit are collinear enough (correlation $>0.999$ in our experiments) that a joint fit leaves $\tilde c_2$ and $\tilde c_3$ only weakly identified, occasionally producing a physically implausible negative coefficient.
Timing the two phases separately avoids this by regressing against a far less collinear pair of basis functions.

All constants are specific to the fixed configuration used throughout this study ($d=6$, $\nu=4$) and to the machine that they were measured on.

\subsubsection{Holdout validation}

We also performed holdout validation of these constants by comparing the theoretical and empirical runtimes for new experiments.
We used a grid of series lengths spanning $l \in \{16, \ldots, 4096\}$ and a fixed $n = 698$ number of series.
Furthermore, we used a grid of numbers of series spanning $n \in \{64, \ldots, 16\,384\}$ and a fixed $l = 384$ series length.
Not only none of the $(n, l)$ combination in this holdout validation appears in the grid used to estimate the hardware-specific constants, but these runtimes are also obtained from new experiments.

\begin{table}[tbp]
    \caption{Relative errors between the theoretical and empirical runtimes.}
    \label{table:theory_vs_actual_errors}
    \centering
    \begin{tabular}{lcccc}
        \toprule
        Estimator & Mean & \makecell{Standard\\ deviation} & Minimum & Maximum \\
        \midrule
        \mintinline{py}{MomentQuant("exact", "samples")}    & $-1.5\%$ & $1.8\%$  & $-5.4\%$ & $+0.6\%$ \\
        \mintinline{py}{MomentQuant("exact", "intervals")}  & $+3.6\%$ & $15.1\%$ & $-37.0\%$ & $+23.3\%$ \\
        \mintinline{py}{MomentQuant("approx", "samples")}   & $-2.8\%$ & $24.5\%$ & $-68.4\%$ & $+24.5\%$ \\
        \mintinline{py}{MomentQuant("approx", "intervals")} & $+5.7\%$ & $24.8\%$ & $-24.4\%$ & $+38.0\%$ \\
        \mintinline{py}{QuantFloat64}                       & $+15.4\%$ & $12.6\%$ & $-6.6\%$ & $+45.2\%$ \\
        \bottomrule
    \end{tabular}
\end{table}

\autoref{table:theory_vs_actual_errors} shows that the cost model transfers well to held-out $(l,n)$ combinations, with a mean error within a few percents of zero for four of the five estimators.
\mintinline{py}{MomentQuant("exact", "samples")} is the tightest fit by a wide margin (mean $-1.5\%$, standard deviation $1.8\%$, never off by more than $5.4\%$ in either direction).
The series-outer kernel's compiled, single-call-per-series structure is evidently the easiest of the five to model precisely.
\mintinline{py}{QuantFloat64} is the outlier: despite sharing the same functional form as \mintinline{py}{MomentQuant("exact", "intervals")}, its predictions carry the largest systematic bias of any estimator ($+15.4\%$ on average, consistently in the same direction),
These results suggest that fitting its own $c_1, c_2, c_3$ and $\kappa^{(I)}(l)$ directly, rather than reusing MomentQuant's exact-mode constants, still leaves some PyTorch-specific overhead unaccounted for that this cost model was not designed to capture.
\mintinline{py}{MomentQuant("approx", "intervals")}, \mintinline{py}{MomentQuant("exact", "intervals")}, \mintinline{py}{MomentQuant("approx", "samples")} show wider spreads, with individual worst-case errors reaching $37$ to $68\%$.
Inspecting these directly shows the largest errors cluster at the smallest $l$ or $n$ values in each sweep, exactly where the fixed per-call dispatch overhead $\kappa(l)$ (the hardest-to-model term) makes up the largest share of the total predicted cost.

\subsubsection{Dispatch crossover}

Beyond checking whether the cost model predicts raw runtime, we can check whether it predicts the right \emph{decision}: at what sample count $n$ should the automatic switch from the series-outer to the interval-outer kernel.
\autoref{theorem:21} and \autoref{theorem:33} give a formal crossover sample size $n^*(l,d,\nu)$, assembled bottom-up from the same per-operation constants validated above, evaluated here with their all-four-representations rather than the raw-representation-only theorem statement.
The calibrated threshold actually used by the automatic switch is fit a different way entirely: directly against measured runtimes, as a single pooled, effectively $l$-independent overhead, precisely because \autoref{theorem:21}'s own remark argues that composing $n^*$ representation-by-representation, and $l$-value by $l$-value, from the theorem's own $l$-dependent constants is impractical to estimate reliably.
\autoref{table:dispatch_crossover_comparison} compares the two crossovers directly, across the calibration grid's own $l \in \{16, \ldots, 8192\}$.

\begin{table}[tbp]
    \caption{
        Comparisons between the calibrated dispatch $n_{\text{heuristic}}$ and theorem-derived $n^*_{\text{theorem}}$ (\autoref{theorem:21} and \autoref{theorem:33} extended to all four representations) crossovers, by mode and series length $l$, single-threading only.
    }
    \label{table:dispatch_crossover_comparison}
    \centering
    \begin{tabular}{llccc}
        \toprule
        Mode & $l$ & $n_{\text{heuristic}}$ & $n^*_{\text{theorem}}$ & Ratio \\
        \midrule
        \multirowcell{6}{\mintinline{py}{MomentQuant("exact")}} & 16   & 23\,170 & 20.0   & 1161$\times$ \\
         & 64   & 2\,482  & 44.8   & 55$\times$ \\
         & 256  & 395     & 51.2   & 7.7$\times$ \\
         & 1024 & 72.4    & 12.8   & 5.7$\times$ \\
         & 4096 & 14.3    & 2.9    & 5.0$\times$ \\
         & 8192 & 6.5     & 1.4    & 4.7$\times$ \\
        \midrule
        \multirowcell{6}{\mintinline{py}{MomentQuant("approx")}} & 16   & 39\,964 & 1\,057 & 38$\times$ \\
         & 64   & 9\,991  & 610    & 16$\times$ \\
         & 256  & 2\,498  & 216    & 11.6$\times$ \\
         & 1024 & 624     & 75.8   & 8.2$\times$ \\
         & 4096 & 156     & 28.5   & 5.5$\times$ \\
         & 8192 & 78.1    & 17.9   & 4.3$\times$ \\
        \bottomrule
    \end{tabular}
\end{table}

Theory and practice disagree sharply, and always in the same direction.
At every $l$ tested and in both modes, we have $n_{\text{heuristic}} \gg n^*_{\text{theorem}}$: the calibrated rule keeps using the series-outer kernel for a far larger sample count than the theorem says is optimal.
However, the size of the disagreement is not stable.
It is largest at the smallest $l$ (a factor of $1161\times$ for exact mode at $l=16$, $38\times$ for approx mode) and shrinks steadily, by two to three orders of magnitude, as $l$ grows, settling into a roughly constant $4$ to $5\times$ (exact) or $4$ to $6\times$ (approx) residual gap that persists even at $l=8192$, the largest $l$ on the calibration grid.
It lines up with a mechanism that we already anticipated rather than a new one: the calibrated heuristic fits a single, pooled overhead with no $l$-dependence of its own, while $\kappa^{(I)}(l)$ and $\tilde\kappa^{(I)}(l)$, the quantities that the theorem actually plugs in, are themselves strongly $l$-dependent and largest in relative terms at small $l$.
A pooled, $l$-independent stand-in for a genuinely $l$-dependent quantity should disagree with the theorem most where that quantity varies fastest (small $l$) and least where it flattens out (large $l$), which is exactly the shape \autoref{table:dispatch_crossover_comparison} shows.
This is consistent with the pooling simplification being the dominant source of disagreement between the two models, rather than a more basic misspecification of either one.
\autoref{theorem:21} and \autoref{theorem:33}'s crossovers should be read as validating the existence and functional form of a crossover, not as a formula that can substitute for the calibrated, measurement-fit threshold at any particular $l$.

In our implementation of \mintinline{py}{MomentQuant}, the automatic version uses the calibrated dispatch crossovers instead of the theorem-derived crossovers, both in the single-threaded and multithreaded setups.

\subsection{Transformation runtimes}\label{sec8.3}

We now move from controlled, single-length microbenchmarks to a full, realistic workload: transforming both the training and test sets of all $142$ data sets of the UCR archive, for all six combinations of \mintinline{py}{MomentQuant}, alongside \mintinline{py}{QuantFloat64} as an external reference, in both single-threaded and multithreaded (8 threads) setups.

\begin{table}[tbp]
    \caption{Total wall-clock runtime to transform all $142$ UCR archive data sets (training and test set merged), using 1 thread and $8$ threads, and the resulting speedup.}
    \label{table:ucr142_transform_runtimes}
    \centering
    \begin{tabular}{lrrr}
        \toprule
        Estimator & 1 thread & 8 threads & Speedup \\
        \midrule
        \mintinline{py}{MomentQuant("exact", "samples")}    & 117.97s & 26.67s & 4.42$\times$ \\
        \mintinline{py}{MomentQuant("exact", "intervals")}  & 64.80s  & 65.27s & 0.99$\times$ \\
        \mintinline{py}{MomentQuant("exact", "auto")}       & 69.15s  & 26.56s & 2.60$\times$ \\
        \mintinline{py}{MomentQuant("approx", "samples")}   & 38.45s  & 23.41s & 1.64$\times$ \\
        \mintinline{py}{MomentQuant("approx", "intervals")} & 39.44s  & 40.25s & 0.98$\times$ \\
        \mintinline{py}{MomentQuant("approx", "auto")}      & 38.74s  & 23.36s & 1.66$\times$ \\
        \mintinline{py}{QuantFloat64}                       & 105.55s & 65.63s & 1.61$\times$ \\
        \bottomrule
    \end{tabular}
\end{table}

\autoref{table:ucr142_transform_runtimes} provides the corresponding runtimes.
In the single-threaded setup, we obtain the same ordering as the one established in \autoref{table:implementation_choice_runtimes}.
\mintinline{py}{MomentQuant("exact", "intervals")} ($64.80$s) is faster than both \mintinline{py}{MomentQuant("exact", "samples")} ($117.97$s, $1.8\times$ slower) and \mintinline{py}{QuantFloat64} ($105.55$s, $1.6\times$ slower).
These results are not surprising, since the archive mixes many $(n, l)$ combinations, some of which genuinely favor the interval kernel.
The approx mode is faster still, as expected from its better asymptotic complexity: all three variants cluster tightly between $38$ and $39$s, roughly $1.7$ to $1.8\times$ faster than the best exact-mode variant and $2.7$ to $3.1\times$ faster than \mintinline{py}{QuantFloat64}.

The multithreaded setup benefits the two kernels very differently.
The interval-outer kernel shows no speedup at all from $8$ threads, since it is not parallelized in this implementation, while the series-outer kernel (parallelized across series via Numba) speeds up substantially: $4.42\times$ for exact mode, a more modest $1.64\times$ for approx mode.
\mintinline{py}{QuantFloat64} also benefits from PyTorch's own default multi-threading ($1.61\times$), consistent with the results in \autoref{table:implementation_choice_runtimes}.
We tried to parallelize the interval-outer kernel, but obtained worse results than in the single-threaded setup, which is why the interval-outer kernel always use single-threading.

The automatic variants track whichever manual variant is actually faster in each regime closely, without ever being the worst choice.
In the single-threaded setup, \mintinline{py}{MomentQuant("exact", "auto")} ($69.2$s) lands within $7\%$ of the best single-threaded choice (\mintinline{py}{MomentQuant("exact", "intervals")}, $64.8$s), so it does not always pick the faster variant.
At $8$ threads, \mintinline{py}{MomentQuant("exact", "auto")} ($26.56$s) essentially matches the best choice outright (\mintinline{py}{MomentQuant("exact", "samples")}, $26.67$s), since the parallel-specific threshold correctly steers the large majority of the archive toward the now much faster samples kernel.
The same holds for the approx mode in both settings (\mintinline{py}{MomentQuant("approx", "auto")} within $1\%$ of the best manual variant single-threaded, and matching it at $8$ threads).
This is an end-to-end validation of the calibrated dispatch rule described above, at a scale and data set diversity well beyond the single-length microbenchmark that it was calibrated on.

\subsection{Classification performance and end-to-end runtimes}\label{sec8.4}

We now evaluate the automatic versions of both the approximate (\mintinline{py}{MomentQuant("approx")}) and exact (\mintinline{py}{MomentQuant("exact")}) modes, alongside \mintinline{py}{QuantFloat64} as an external reference, on all $142$ UCR archive data sets, using $30$ resamples per data set.
Regarding the values of the hyperparameters of the classification algorithms, we use the same ones as in \citep{dempsterQuantMinimalistInterval2024}.

\subsubsection{Classification performance}

First, we compare three classification algorithms built on top of MomentQuant: extremely randomized trees \citep{geurtsExtremelyRandomizedTrees2006a}, random forests \citep{breimanRandomForests2001a}, and Ridge \citep{hoerlRidgeRegressionApplications1970}.
In \citep{dempsterQuantMinimalistInterval2024}, the authors showed that Quant performed significantly better with extremely randomized trees than with the other two algorithms.
\autoref{fig:classifier_comparison} shows the pairwise accuracy between the three algorithms, with extremely randomized trees as the reference.
MomentQuant is significantly better with extremely randomized trees (mean accuracy $0.8505$) than with random forests (mean accuracy $0.8396$, $p < 0.001$) and Ridge (mean accuracy $0.7678$, $p < 0.001$).
These results are consistent with the ones in \citep{dempsterQuantMinimalistInterval2024}.

\begin{figure}
    \begin{subfigure}{0.49\textwidth}
        \includegraphics[width=\textwidth]{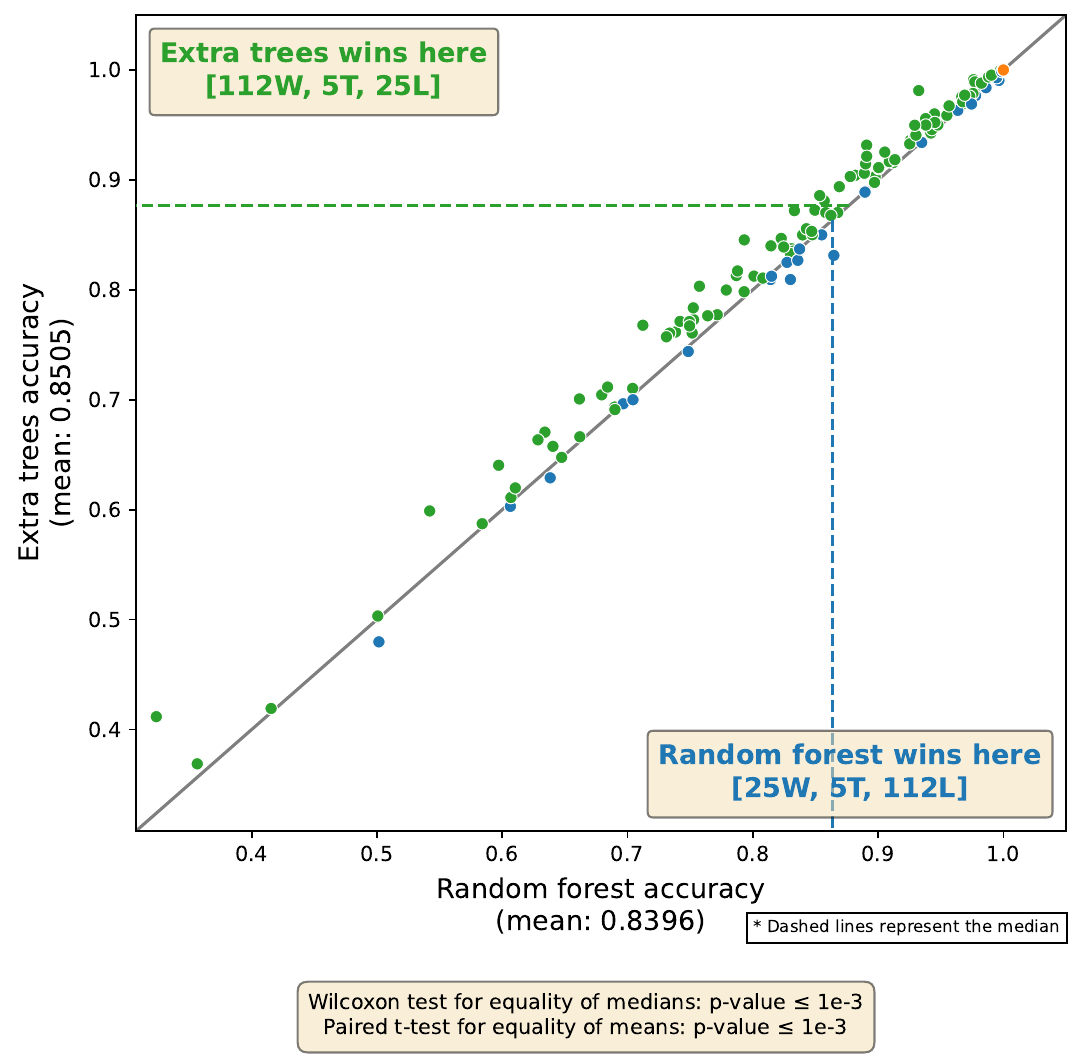}
    \end{subfigure}%
    \hfill
    \begin{subfigure}{0.49\textwidth}
        \includegraphics[width=\textwidth]{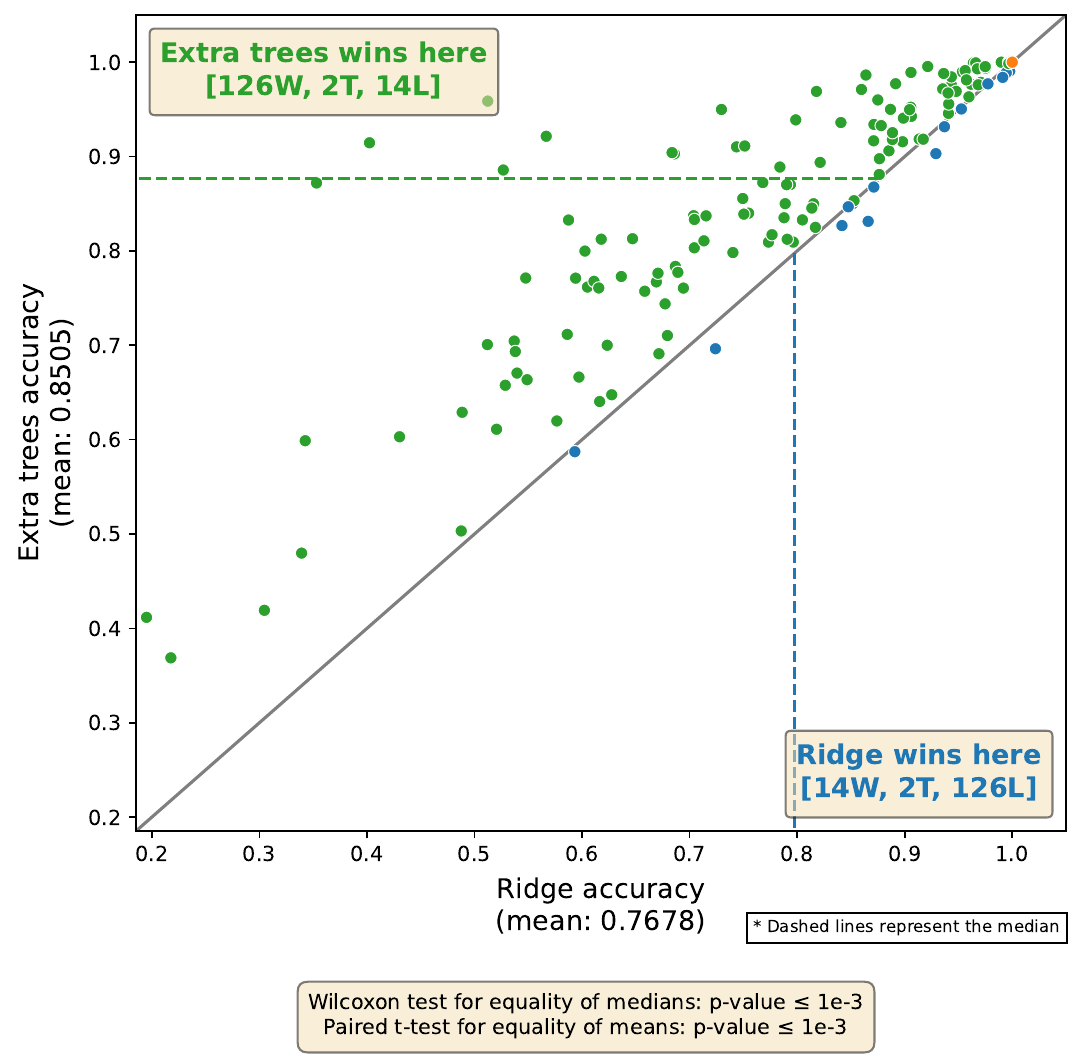}
    \end{subfigure}%
    \caption{
        Pairwise accuracy for MomentQuant with extremely randomized trees (default), compared to with random forest (left) and Ridge (right) in terms of accuracy on the 142 UCR data sets.
        The mean accuracy scores are computed over 30 resamples for each data set.
    }
    \label{fig:classifier_comparison}
\end{figure}

We now compare the classification performances of MomentQuant with both \mintinline{py}{MomentQuant("exact")} and \mintinline{py}{QuantFloat64}.
\autoref{table:classification_performance} provides the classification performance of \mintinline{py}{MomentQuant("approx")}, \mintinline{py}{MomentQuant("exact")} and \mintinline{py}{QuantFloat64}, on the 142 UCR data sets.
\mintinline{py}{MomentQuant("approx")} and \mintinline{py}{QuantFloat64} have (almost) the same performance (at least up to 4 decimals, except for the negative log-likelihood).
\mintinline{py}{MomentQuant("approx")} is slightly behind \mintinline{py}{MomentQuant("exact")} for all five metrics, but with less than $0.01$ for each metric.

\begin{table}[tbp]
    \caption{
        Classification performance of \mintinline{py}{MomentQuant("approx")}, \mintinline{py}{MomentQuant("exact")} and \mintinline{py}{QuantFloat64}, on the 142 UCR data sets.
        For each algorithm and each metric, the mean score is computed over the $142 \times 30 $ (data set, resample) pairs.
        The direction of each arrow indicates in which direction a better score is for each metric.
    }
    \label{table:classification_performance}
    \centering
    \begin{tabular}{lccccc}
        \toprule
        Estimator & ACC $\uparrow$ & BALACC $\uparrow$ & AUROC $\uparrow$ & NLL $\downarrow$ & F1 $\uparrow$ \\
        \midrule
        \mintinline{py}{MomentQuant("approx")} & 0.8505 & 0.8264 & 0.9541 & 0.5345 & 0.8262 \\
        \mintinline{py}{MomentQuant("exact")}  & 0.8551 & 0.8301 & 0.9559 & 0.5332 & 0.8305 \\
        \mintinline{py}{QuantFloat64}          & 0.8551 & 0.8301 & 0.9559 & 0.5333 & 0.8305 \\
        \bottomrule
    \end{tabular}
\end{table}

\autoref{fig:approx_vs_exact_vs_quant_accuracy} shows the pairwise accuracy between the three algorithms, with MomentQuant as the reference.
MomentQuant has a slightly lower mean accuracy ($0.8505$) than both \mintinline{py}{MomentQuant("exact")} ($0.8551$) and \mintinline{py}{QuantFloat64} ($0.8551$), but the differences are significant ($p = 0.002$).
The worst raw difference in terms of accuracy occurs on the ShapeletSim data set, where accuracy drops from $0.9852$ to $0.8269$ ($-0.1583$).
Otherwise, the raw accuracy difference lies between $-0.0470$ (ACSF1) and $+0.0339$ (Lightning2).
This outlier (ShapeletSim) is easily visible in \autoref{fig:approx_vs_exact_vs_quant_accuracy}.

Although \mintinline{py}{MomentQuant("exact")} and \mintinline{py}{QuantFloat64} are theoretically identical, they actually slightly differ in practice due to the different libraries used (NumPy and PyTorch).
We identified two reasons that can make the results very slightly different:
\begin{itemize}
    \item \ul{Different mean computation}: NumPy and PyTorch compute the mean via different summation orders, and double-precision addition is known to be non-associative because of rounding errors possibly occurring at each step. This difference can make centered quantiles very slightly different depending on the library.
    \item \ul{Discrete Fourier transform}: We compared the results of the discrete Fourier transform for multiple series lengths between both libraries and identified one value ($2047$) where the results were very slightly different. More values of series lengths leading to slightly different results is a possibility.
\end{itemize}
On the 142 data sets, \mintinline{py}{MomentQuant("exact")} has 10 wins, 129 ties and 3 losses compared to \mintinline{py}{QuantFloat64}, and the difference in mean accuracy ($0.855099$ vs $0.855077$) is not significant ($p = 0.087$).

\begin{figure}
    \begin{subfigure}{0.49\textwidth}
        \includegraphics[width=\textwidth]{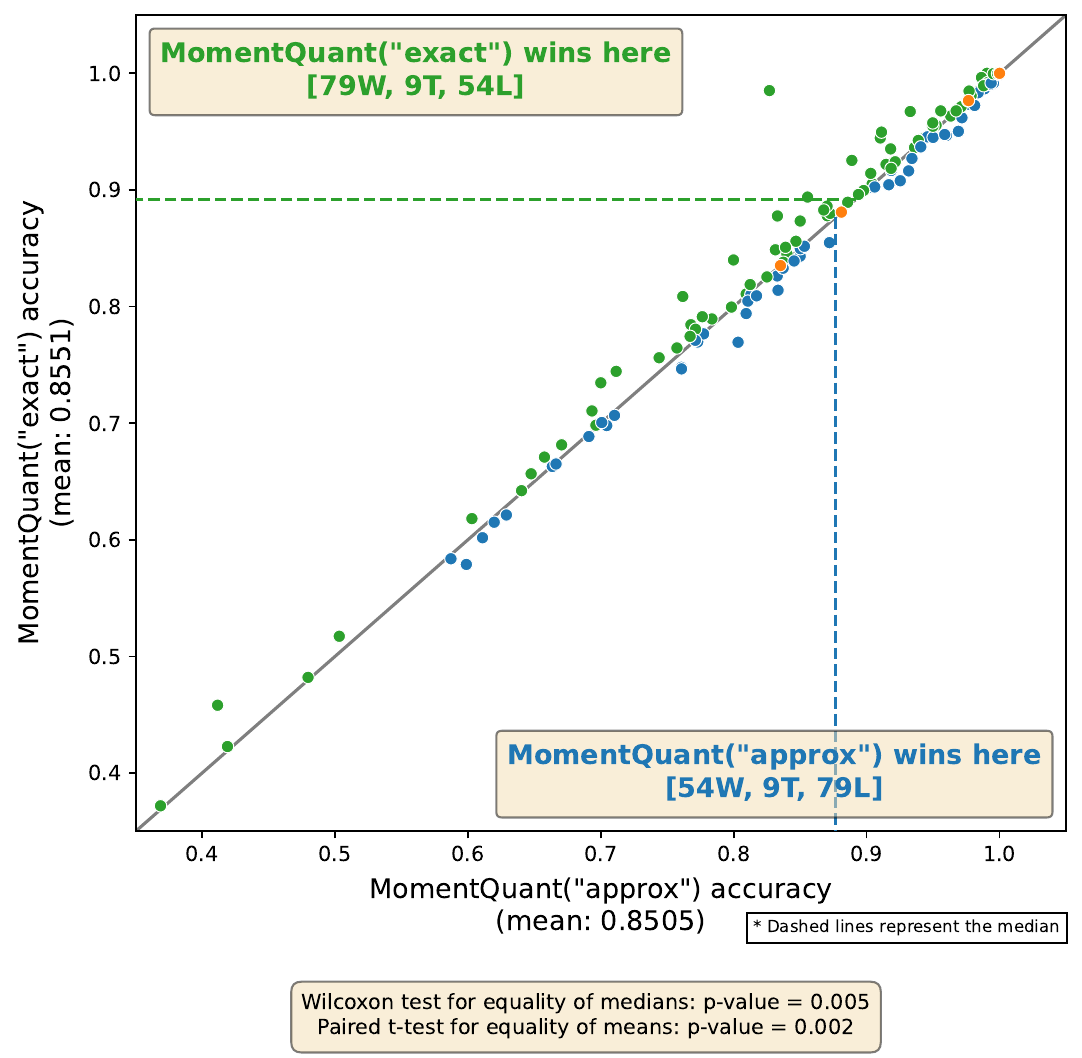}
    \end{subfigure}%
    \hfill
    \begin{subfigure}{0.49\textwidth}
        \includegraphics[width=\textwidth]{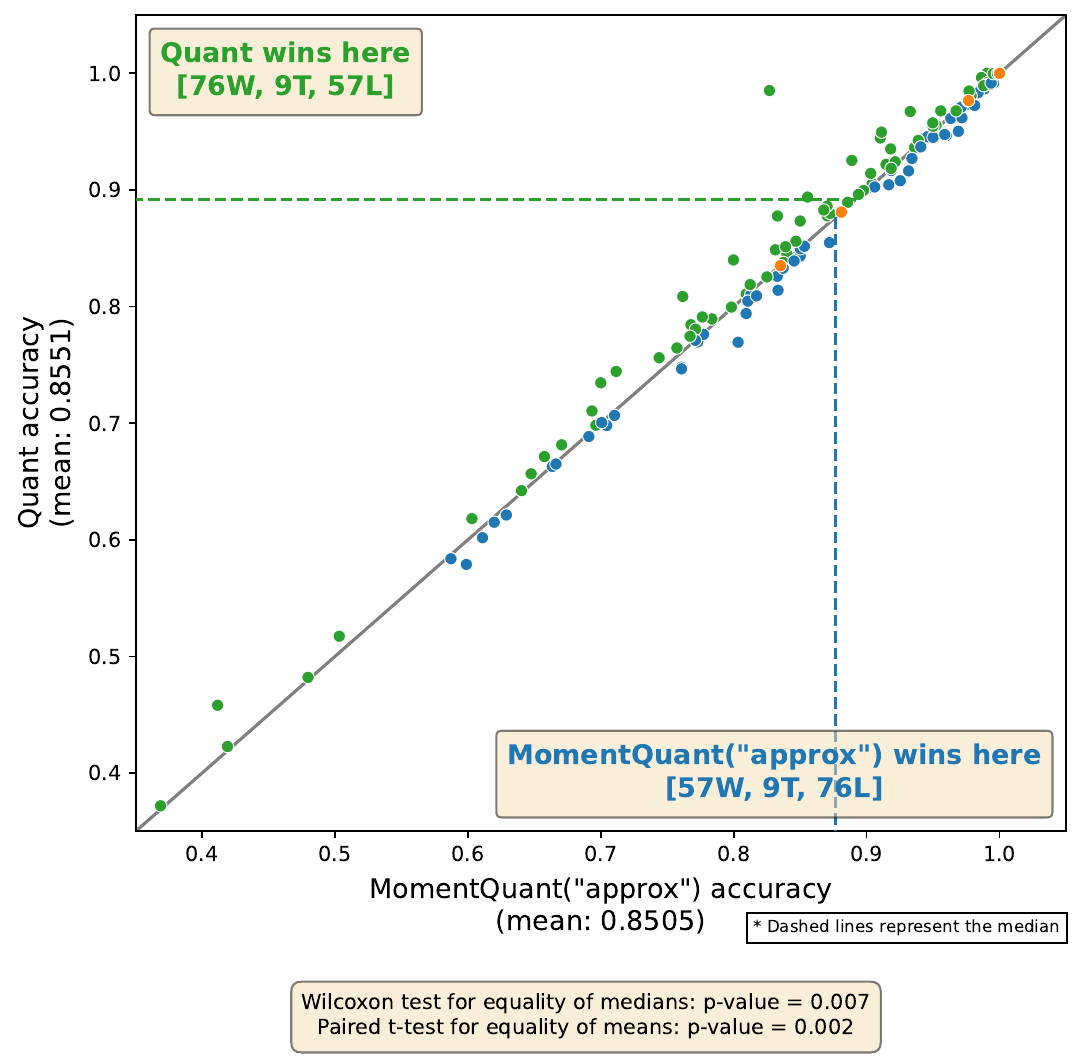}
    \end{subfigure}%
    \caption{
        Pairwise accuracy for \mintinline{py}{MomentQuant("approx")} compared to \mintinline{py}{MomentQuant("exact")} (left) and \mintinline{py}{QuantFloat64} (right) in terms of accuracy on the 142 UCR data sets.
        The mean accuracy scores are computed over 30 resamples for each data set.
    }
    \label{fig:approx_vs_exact_vs_quant_accuracy}
\end{figure}

\subsubsection{End-to-end and inference-only runtimes}

We now compare the runtimes of the whole pipeline (transformation step + classification step) instead of the transformation step only.
\autoref{table:total_and_inference_runtimes} provides the total runtimes of \mintinline{py}{MomentQuant("approx")}, \mintinline{py}{MomentQuant("exact")}, and \mintinline{py}{QuantFloat64}, summed over the 142 UCR data sets and averaged over the 30 resamples, in the single- and multithreaded setups.

\begin{table}[tbp]
    \caption{Total and inference-only wall-clock runtime, summed over $142$ data sets and averaged over the $30$ resamples, in the single-threaded and multithreaded setups.}
    \label{table:total_and_inference_runtimes}
    \centering
    \begin{tabular}{lcccc}
        \toprule
        Estimator & \makecell{Total\\ (1 thread)} & \makecell{Inference-only \\ (1 thread)} & \makecell{Total\\ (8 threads)} & \makecell{Inference-only\\ (8 threads)} \\
        \midrule
        \mintinline{py}{MomentQuant("approx")} & 401s & 29s & 156s & 22s \\
        \mintinline{py}{MomentQuant("exact")}  & 436s & 47s & 161s & 25s \\
        \mintinline{py}{QuantFloat64}          & 474s & 68s & 200s & 48s \\
        \bottomrule
    \end{tabular}
\end{table}

Starting with the total (training + inference) runtimes, the gains of both \mintinline{py}{MomentQuant("approx")} and \mintinline{py}{MomentQuant("exact")} are mild compared to \mintinline{py}{QuantFloat64}, but relatively higher in the multithreaded setup than in the single-threaded setup.
In the single-threaded setup, the total runtimes are respectively 401s ($-15\%$), 436s ($-8\%$), and 474s.
In the multithreaded setup, the total runtimes are respectively 156s ($-22\%$), 161s ($-19\%$), and 200s.
Indeed, as noted in \citep{dempsterQuantMinimalistInterval2024}, most of the total runtime comes from training the classification algorithm.
This is even more exacerbated by the runtime improvements of \mintinline{py}{MomentQuant("approx")} and \mintinline{py}{MomentQuant("exact")}.
\autoref{fig:stage_breakdown} shows the breakdown of the total single-threaded runtime of \mintinline{py}{MomentQuant("approx")}, \mintinline{py}{MomentQuant("exact")}, and \mintinline{py}{QuantFloat64} into its four stages: training (transformation, classification) and inference (transformation, classification).
The breakdown is averaged over the 142 UCR data sets, but split into four bins for the series length: short ($l < 150$, $35$ data sets), medium-short ($150 \leq l < 320$, $36$ data sets), medium-long ($320 \leq l < 720$, $38$ data sets) and long ($l > 720$, $33$ data sets).
In every situation, training the classification algorithm is the step taking the most (relative) time.
The differences between \mintinline{py}{MomentQuant("approx")}, \mintinline{py}{MomentQuant("exact")}, and \mintinline{py}{QuantFloat64} are the biggest for long series, which is consistent with our previous results.
Indeed, the computational complexities for \mintinline{py}{MomentQuant("approx")} and \mintinline{py}{MomentQuant("exact")} in the saturated regime are $\Theta(d \cdot l)$ and $\Theta(d \cdot l \cdot \log(l))$ respectively.
The absolute difference becomes bigger for larger values of $\log(l)$, that is for larger values of $l$.

\begin{figure}
    \includegraphics[width=\textwidth]{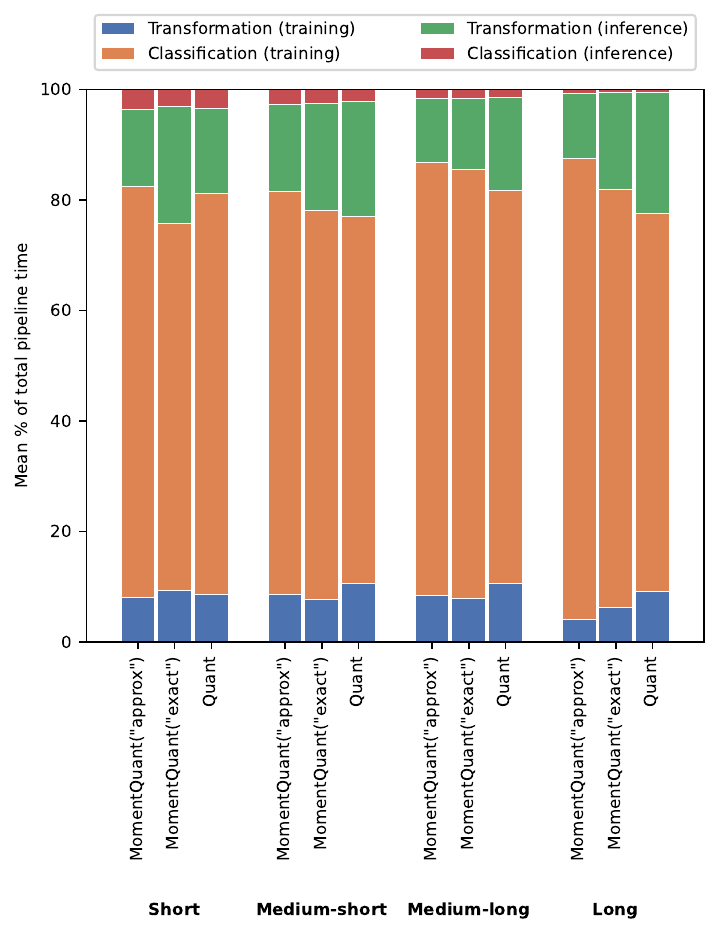}
    \caption{
        Breakdown of the total wall-clock single-threaded runtime of \mintinline{py}{MomentQuant("approx")}, \mintinline{py}{MomentQuant("exact")}, and \mintinline{py}{QuantFloat64} into its four stages: training (transformation, classification) and inference (transformation, classification).
        The breakdown is averaged over the 142 UCR data sets, but split into four bins for the series length: short, medium-short, medium-long and long.
    }
    \label{fig:stage_breakdown}
\end{figure}

We believe that benchmarking time series classification algorithms using the total runtime (training + inference) only is inappropriate.
Indeed, in actual real-life applications, training an algorithm is performed once (in a while), while its use in production is much more common, potentially daily.
Therefore, we believe that time series classification algorithms should also be benchmarked using the inference-only runtime.
\autoref{fig:stage_breakdown_inference} shows the corresponding breakdown of the inference-only single-threaded runtime of \mintinline{py}{MomentQuant("approx")}, \mintinline{py}{MomentQuant("exact")}, and \mintinline{py}{QuantFloat64}.
In every situation, the transformation step is taking the most (relative) time.
For all three estimators, the bigger the series length, the higher the proportion of the transformation step in the inference-only runtime.
These results prove that optimizing the transformation steps of Quant and MomentQuant is actually relevant.

\begin{figure}
    \includegraphics[width=\textwidth]{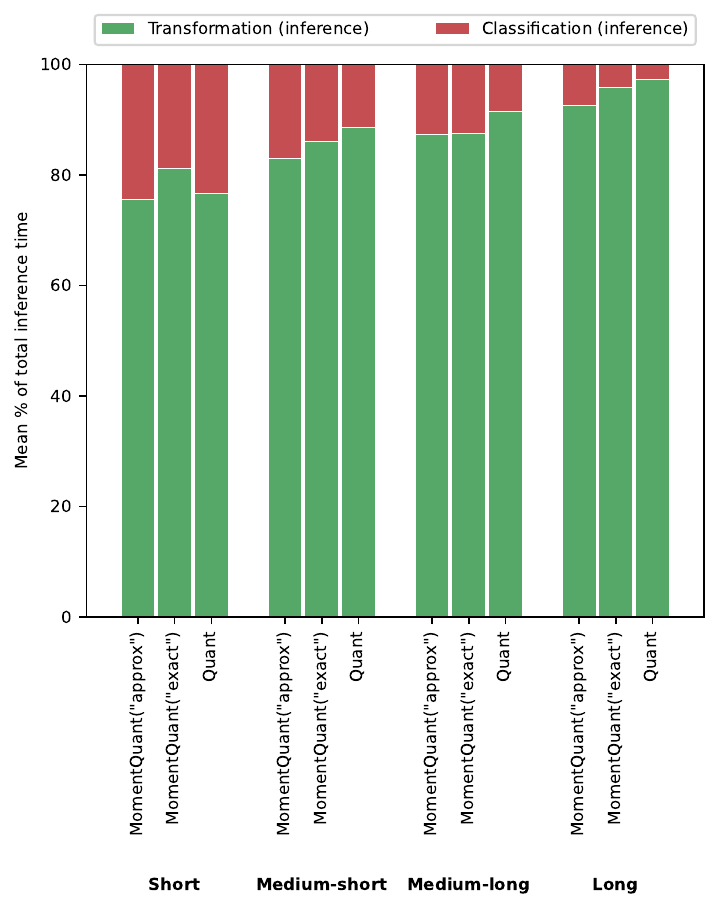}
    \caption{
        Breakdown of the inference-only wall-clock single-threaded runtime of \mintinline{py}{MomentQuant("approx")}, \mintinline{py}{MomentQuant("exact")}, and \mintinline{py}{QuantFloat64} into its two stages: transformation and classification.
        The breakdown is averaged over the 142 UCR data sets, but split into four bins for the series length: short, medium-short, medium-long and long.
    }
    \label{fig:stage_breakdown_inference}
\end{figure}

We now provide theoretical explanations for these differences in relative runtimes between the transformation and classification steps in the training and inferences phases.
To make the analysis simpler, we will focus on the asymptotic complexities and overlook the constants terms (e.g., for the overheads).
For a given series length, the transformation steps of Quant and MomentQuant are always proportional to the number of series in the data set, independently of the phase (training or inference).
The only additional work that the transformation step of the training phase has to perform is to compute the intervals based on the series length, which has a $\Theta(\min(l, 2^d))$ computational complexity and is thus negligible.
On the other hand, classification algorithms take much longer to train than to infer from.
For instance, the computational complexity of training an extremely randomized trees model is $\mathcal{O}(m \cdot k \cdot n \cdot \log(n))$, where $m$ is the number of trees, $k$ is the number of randomly selected features at each node, and $n$ is the number of training samples \citep{geurtsExtremelyRandomizedTrees2006a}.
However, its inference computational complexity is only $\mathcal{O}(m \cdot p \cdot \log(n))$, with $p$ being the number of samples.
Assuming that the training and test sets have the same number of samples ($n = p$), we can see that the training computational complexity has an extra factor $k$.
In our experiments, we set $k = 0.1$ as it is done in \citep{dempsterQuantMinimalistInterval2024}, meaning that the number of randomly selected features at each node is proportional to the number of features, which is proportional to the series length $l$.
Thus, for training and test sets with the same number of samples $n$, the computational complexity of training the extremely randomized trees model is $\mathcal{O}(m \cdot l \cdot n \cdot \log(n))$, which has an extra $l$ factor compared to its inference computational complexity, which is $\mathcal{O}(m \cdot n \cdot \log(n))$.

\subsubsection{Trade-offs}

We now investigate the trade-offs that MomentQuant makes compared to Quant.

We first compare the accuracy differences and runtime speedups to the series lengths.
Indeed, \autoref{sec6} motivates MomentQuant specifically for long series, where its better asymptotic complexity should matter most.
\autoref{table:accuracy_runtime_by_length_bin} provides the mean accuracy scores and the median single-threaded end-to-end runtimes between the \mintinline{py}{"approx"} and \mintinline{py}{"exact"} versions of MomentQuant, for each four bin of series lengths.
There is no clean, monotonic pattern in these results.
The biggest deficits for the \mintinline{py}{"approx"} version, both in terms of accuracy difference ($-0.94$) and runtime speedup ($0.997\times$), occur for the medium-long series length bin.
For the other three bins, the \mintinline{py}{"approx"} version performs worse, but to a lesser extent (accuracy differences ranging from $-0.58$ to $-0.08$), and is faster (speedup ratios ranging from $1.06$ to $1.21$), than the \mintinline{py}{"exact"} version.
However, the runtimes are end-to-end (i.e., including the classification steps) and the numbers of samples are not taken into account, which might explain the absence of monotonicity.
\autoref{fig:accuracy_and_speedup_vs_series_length} complements \autoref{table:accuracy_runtime_by_length_bin} with the results for the 142 data sets.

\begin{table}[tbp]
    \caption{
        Mean accuracy and median single-threaded end-to-end runtime between the \mintinline{py}{"approx"} and \mintinline{py}{"exact"} versions of MomentQuant, by series-length bin.
        $\Delta$ is the mean, over the data sets in the bin, of each data set's own (\mintinline{py}{"approx"} $-$ \mintinline{py}{"exact"}) accuracy difference.
        Speedup is likewise the mean, over the data sets in the bin, of each data set's own \mintinline{py}{"exact"}/\mintinline{py}{"approx"} runtime ratio.
        This is a per-data-set average, not the ratio of the two (bin-aggregated) Runtime columns shown here, so it does not necessarily match a naive division of these two columns.
    }
    \label{table:accuracy_runtime_by_length_bin}
    \centering
    \begin{tabular}{lcccccc}
        \toprule
        \multirowcell{2}{Series length bin} & \multicolumn{3}{c}{Accuracy} & \multicolumn{3}{c}{Runtime} \\
        \cmidrule(lr){2-4} \cmidrule(lr){5-7}
        & \mintinline{py}{"approx"} & \mintinline{py}{"exact"} & $\Delta$ & \mintinline{py}{"approx"} & \mintinline{py}{"exact"} & Speedup \\
        \midrule
        Short        & 0.8818 & 0.8826 & $-0.0008$ & 0.415 & 0.470 & 1.206$\times$ \\
        Medium-short & 0.8663 & 0.8683 & $-0.0020$ & 0.592 & 0.670 & 1.064$\times$ \\
        Medium-long  & 0.8279 & 0.8373 & $-0.0094$ & 0.623 & 0.648 & 0.997$\times$ \\
        Long         & 0.8262 & 0.8320 & $-0.0058$ & 5.115 & 5.344 & 1.162$\times$ \\
        \bottomrule
    \end{tabular}
\end{table}

\begin{figure}
    \begin{subfigure}{0.49\textwidth}
        \includegraphics[width=\textwidth]{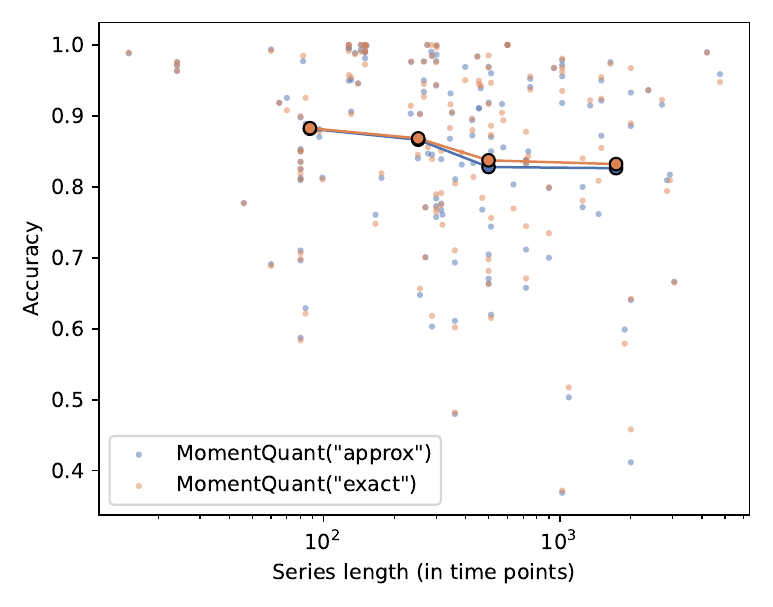}
    \end{subfigure}%
    \hfill
    \begin{subfigure}{0.49\textwidth}
        \includegraphics[width=\textwidth]{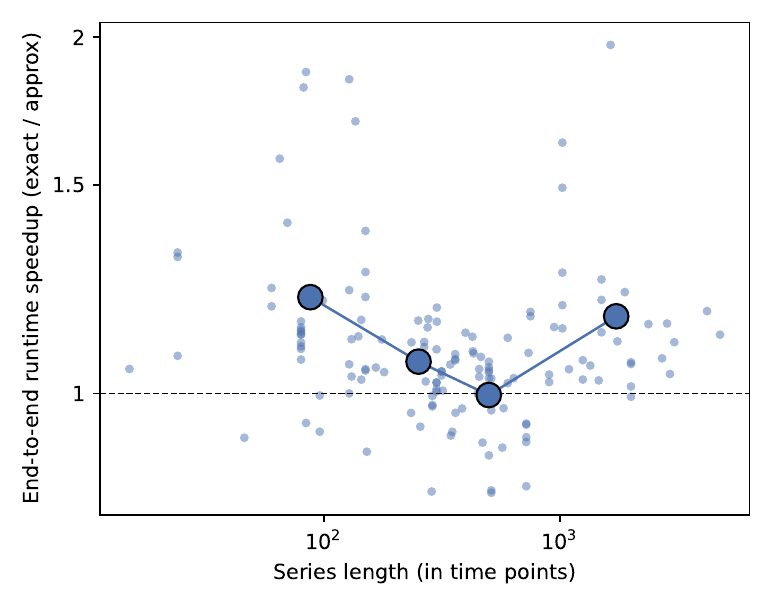}
    \end{subfigure}%
    \caption{
        Accuracy (left) and speedup ratios (right) compared to series length for \mintinline{py}{MomentQuant("approx")} and \mintinline{py}{MomentQuant("exact")}.
        For each of the 142 UCR data sets, the mean accuracy is computed over 30 resamples and plotted, the median runtime is computed over 30 resamples, and the ratio of the medians is plotted.
        Additionally, the mean accuracy and speedup ratio for each bin of series length (short, medium-short, medium-long and long) are also plotted.
    }
    \label{fig:accuracy_and_speedup_vs_series_length}
\end{figure}

We also compare the trade-offs in accuracy and transform-runtime between \mintinline{py}{MomentQuant("approx")} and \mintinline{py}{MomentQuant("exact")}.
\autoref{fig:approx_vs_exact_dataset_tradeoff} shows the results for the 142 data sets.
\mintinline{py}{MomentQuant("approx")} is faster for 108 data sets, and better for 40 out of the 108 data sets.
On the other hand, \mintinline{py}{MomentQuant("exact")} is faster for 34 data sets, and better for 20 out of the 34 data sets.
Interestingly, the proportion of times when \mintinline{py}{MomentQuant("approx")} is better is similar when \mintinline{py}{MomentQuant("approx")} is faster ($37\%$) and slower ($41\%$).
Likewise, the proportion of times when \mintinline{py}{MomentQuant("approx")} is faster is similar when \mintinline{py}{MomentQuant("approx")} is better ($74\%$) and worse ($77\%$).
The Pearson correlation coefficient ($0.0884$) confirms the weak linear relationship between the accuracy differences and the transform-runtime speedup ratios.

\begin{figure}
    \includegraphics[width=\textwidth]{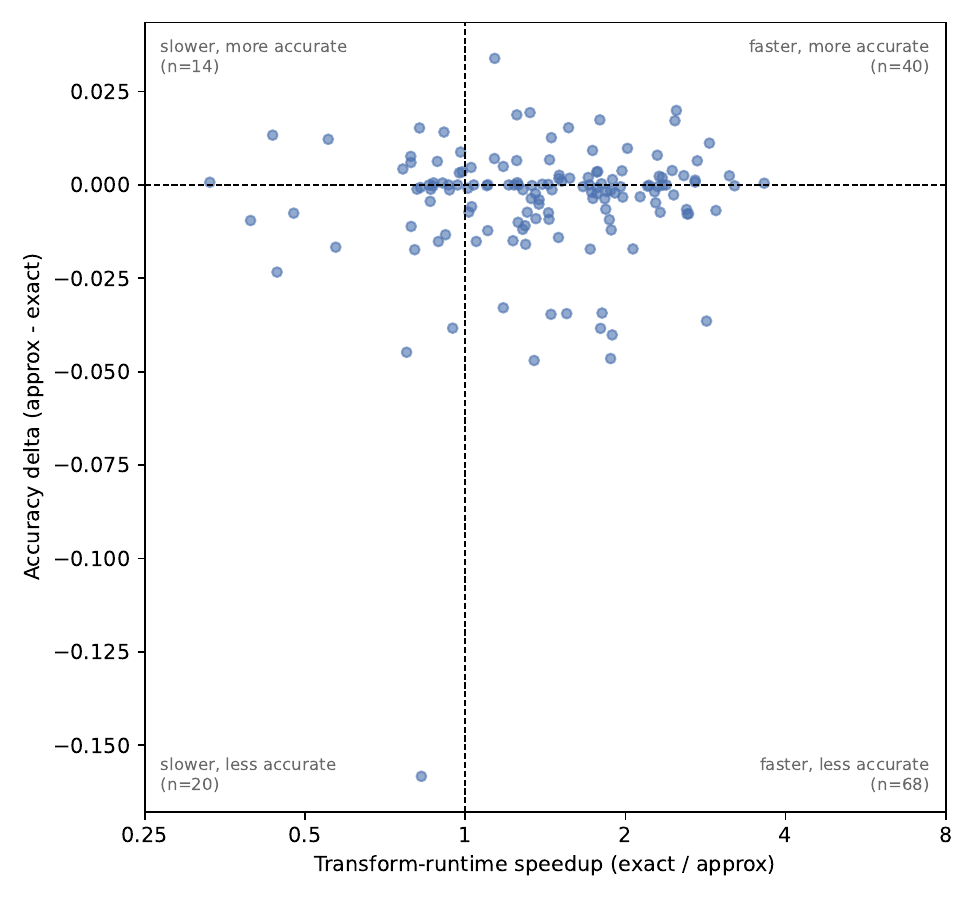}
    \caption{
        Accuracy difference compared to transform-runtime speedup between \mintinline{py}{MomentQuant("approx")} and \mintinline{py}{MomentQuant("exact")}.
        For each of the 142 UCR data sets, the mean accuracy is computed over 30 resamples.
        To the right of the vertical line, \mintinline{py}{MomentQuant("approx")} is faster than \mintinline{py}{MomentQuant("exact")}.
        Above the horizontal line, \mintinline{py}{MomentQuant("approx")} is better than \mintinline{py}{MomentQuant("exact")}.
    }
    \label{fig:approx_vs_exact_dataset_tradeoff}
\end{figure}

\subsection{Fidelity of the moment-based quantile approximation}\label{sec8.5}

All the previous comparisons regarding the approx mode are downstream.
They measure its effect on classification accuracy and runtime, but never the quality of the Cornish-Fisher approximation itself, at the level of individual quantile values.
We close this gap in this section by comparing the true and approximate moment-based quantiles using Pearson correlation.
We investigate the distributions of the Pearson correlation coefficients at each level on six length-diverse UCR data sets.
Naturally, we only include columns for which approximate quantiles are computed (because even MomentQuant computes exact quantiles for the minimum and maximum).

\autoref{table:quantile_correlation_stats} provides information about the series length and descriptive statistics on the distribution of the Pearson correlation coefficients on the six data sets.
Median correlation is consistently high ($0.598$ to $0.988$ depending on the data set), but every data set's distribution has a long negative tail (minima from $-0.049$ to $-0.980$), pulling the mean substantially below the median in every case.

\begin{table}[tbp]
    \caption{Descriptive statistics of the distribution of the Pearson correlation coefficients between the true and approximate quantiles, per data set, pooled over every genuinely-approximated column.}
    \label{table:quantile_correlation_stats}
    \centering
    \begin{tabular}{lllccc}
        \toprule
        Data set & $l$ & Number of columns & Mean & Median & Minimum \\
        \midrule
        ItalyPowerDemand & 24   & 112  & 0.900 & 0.988 & $-0.049$ \\
        ECG200           & 96   & 570  & 0.820 & 0.881 & $-0.849$ \\
        GunPoint         & 150  & 870  & 0.689 & 0.878 & $-0.959$ \\
        Wafer            & 152  & 858  & 0.521 & 0.598 & $-0.584$ \\
        Yoga             & 426  & 3\,033 & 0.702 & 0.847 & $-0.980$ \\
        StarLightCurves  & 1024 & 8\,034 & 0.819 & 0.943 & $-0.968$ \\
        \bottomrule
    \end{tabular}
\end{table}

\autoref{fig:starlightcurves-correlation} shows the distribution of Pearson correlation coefficients at each level for the StarLightCurves data set, the largest one among the six considered, making its per-depth histograms the least noisy of the six.
The correlation rises monotonically and substantially with depth, from a median of $0.650$ at depth $0$ to $0.729$, $0.885$, $0.962$, $0.981$, and finally $0.992$ at depth $5$.
The deepest, shortest intervals are approximated more reliably than the shallowest, longest ones, the opposite of what shrinking sample size per interval alone would suggest.
A plausible, but not independently confirmed, explanation lies in which quantiles are actually being approximated at each depth.
At the shallowest levels, the extreme quantiles are closer to the tails ($q$ near $0$ or $1$), which the Cornish-Fisher expansion is well known to approximate worst.
On the other hand, at the deepest levels, the extreme quantiles are less close to the tails, making the approximation possibly less inaccurate.
The remaining five data sets (Appendix~\ref{secB}) show the same qualitative direction wherever their own depth range is wide enough to check it, though with visibly more sampling noise given their smaller column counts per depth.

However, it is interesting to note that a Pearson correlation coefficient close to $-1$ is not necessarily bad for a downstream classification task.
Indeed, changing the sign of a variable has no impact for several families of classification algorithms such as linear models and tree-based algorithms.
Moreover, the algorithm used for MomentQuant and Quant, extremely randomized trees, is a tree-based method.

\begin{figure}
    \centering
    \includegraphics[width=\textwidth]{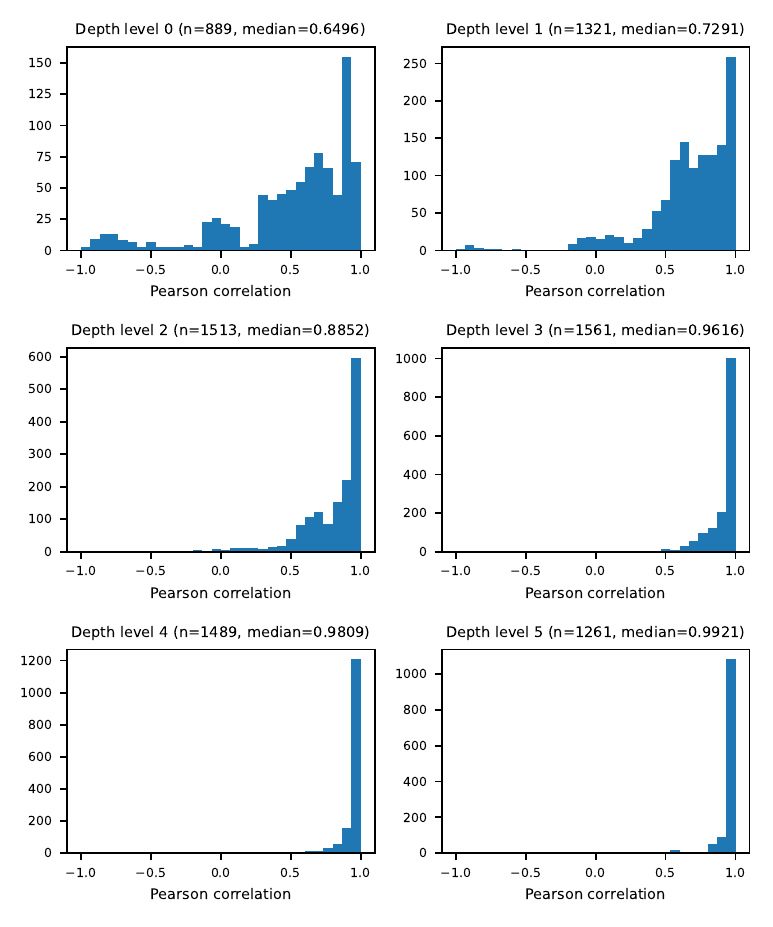}
    \caption{
        Distribution of the Pearson correlation coefficients between the exact and approximated moment-based quantile values on the StarLightCurves data set ($l = 1024$), with one histogram per depth level $0$ to $5$.
        Each panel's title reports the number of genuinely-approximated columns at that depth and their median correlation.
    }
    \label{fig:starlightcurves-correlation}
\end{figure}

\section{One final experiment}\label{sec9}

In the original publication of Quant \citep{dempsterQuantMinimalistInterval2024}, the authors proposed the following optimization:
\begin{quote}
    \itshape
    ``In principle, we could sort each input representation once, keeping track of the indices of the sorted values, and then form any interval by selecting the already-sorted values using their indices.
    In practice, even the naive approach incurs negligible overall computational cost.
    Median transform time over 142 data sets in the expanded UCR archive is less than one second.
    The majority of compute time is spent in training the classifier.
    In other words, any attempt at optimizing total compute time should concentrate on reducing the size of the feature space, and/or improving the efficiency of classifier training.
    We leave further optimization for future work.''
\end{quote}
While our results confirmed that most of the runtime is dominated by the classification algorithm and not Quant's transformation, we focus in this study on the optimization of Quant's transformation.
We were thus interested in the proposed optimization.
However, it is only presented at the end of this article because the results are mixed.

As shown by \autoref{theorem:8}, the total cost of sorting every interval separately is $\Theta(d \cdot l \cdot \log(l))$ in the saturated regime: each of the $d$ depth levels re-sorts (a subdivision of) the whole series from scratch, even though the intervals' endpoints are fixed and data-independent.
Quant's authors anticipated this and suggested, without implementing it, a fix: sort the series once, whose cost is $\mathcal{O}(l \cdot \log(l))$, and reuse that single global order to read off every interval's sorted values directly, instead of sorting each interval's subsequence again.
Since a given depth level's (base or shifted) intervals partition $[0, l)$ into a small, fixed number of non-overlapping pieces, recovering their sorted values from the global order only requires grouping each of the $l$ globally-sorted positions by which interval it falls into.
This is a bucket assignment known in advance from the interval boundaries alone.
If that grouping step could be done in $\mathcal{O}(l)$ per level, as a genuine counting (bucket) sort would (since the number of distinct buckets per level is small and fixed, so a comparison-based sort is not needed to group them), the total cost would fall to $\mathcal{O}(l \cdot \log(l) + d \cdot l)$, which is asymptotically better than $\Theta(d \cdot l \cdot \log(l))$ as soon as $d \geq 2$.

We implemented this optimization exactly as described, both in NumPy and PyTorch.
In both cases we first checked that the optimization produces numerically identical output to the naive per-interval computation, before benchmarking it.
The two implementations tell different stories, for reasons specific to how each library is built.

\begin{table}[tbp]
    \caption{Wall-clock runtime of the naive and proposed implementations, and their ratio, as a function of series length $l$, in NumPy and PyTorch.}
    \label{table:presort_optimization_runtimes}
    \centering
    \begin{tabular}{rrccc}
        \toprule
        Library & $l$ & Naive version & Proposed version & Proposed / Naive \\
        \midrule
        \multirowcell{7}{NumPy}                & 512   & 0.072s & 0.206s  & 2.85$\times$ \\
                                               & 1024  & 0.171s & 0.456s  & 2.66$\times$ \\
                                               & 2048  & 0.393s & 0.990s  & 2.52$\times$ \\
                                               & 4096  & 0.904s & 2.090s  & 2.31$\times$ \\
                                               & 8192  & 1.989s & 4.567s  & 2.30$\times$ \\
                                               & 16384 & 4.295s & 9.761s  & 2.27$\times$ \\
                                               & 32768 & 9.264s & 21.067s & 2.27$\times$ \\
        \midrule
        \multirowcell{7}{PyTorch\\ (1 thread)} & 512   & 0.153s  & 0.187s  & 1.22$\times$ \\
                                               & 1024  & 0.337s  & 0.389s  & 1.15$\times$ \\
                                               & 2048  & 0.745s  & 0.820s  & 1.10$\times$ \\
                                               & 4096  & 1.703s  & 1.663s  & 0.98$\times$ \\
                                               & 8192  & 3.762s  & 3.726s  & 0.99$\times$ \\
                                               & 16384 & 8.325s  & 7.817s  & 0.94$\times$ \\
                                               & 32768 & 18.288s & 16.934s & 0.93$\times$ \\
        \midrule
        \multirowcell{7}{PyTorch\\ (8 threads)} & 512   & 0.130s & 0.107s & 0.82$\times$ \\
                                                & 1024  & 0.244s & 0.177s & 0.73$\times$ \\
                                                & 2048  & 0.449s & 0.338s & 0.75$\times$ \\
                                                & 4096  & 0.754s & 0.594s & 0.79$\times$ \\
                                                & 8192  & 1.293s & 1.170s & 0.90$\times$ \\
                                                & 16384 & 2.318s & 2.208s & 0.95$\times$ \\
                                                & 32768 & 4.722s & 4.647s & 0.98$\times$ \\
        \bottomrule
    \end{tabular}
\end{table}

\autoref{table:presort_optimization_runtimes} provides the wall-clock runtimes of the naive and proposed implementations for multiple series lengths in both libraries.
In NumPy, we do not distinguish the single-threaded and multithreaded setups as both the NumPy functions involved, \mintinline{py}{numpy.sort()} and \mintinline{py}{numpy.argsort()}, do not benefit from multi-threading.
The results in NumPy are underwhelming: the proposed optimization is consistently slower, by a factor of roughly $2.4$ to $2.7\times$, with the gap narrowing slowly as $l$ grows but never closing.
The results in PyTorch are more positive, without being outstanding.
In the single-threaded setup, the proposed optimization is slower for $l \leq 2048$, roughly matches the naive implementation around $l = 4096$--$8192$, and is consistently $6$ to $9\%$ faster for $l \geq 16384$.
In the multithreaded setup (with $8$ threads in our experiments), the proposed optimization is faster than the naive implementation at every $l$ tested.
However, its advantage now shrinks as $l$ grows, from roughly $27\%$ faster at $l = 1024$ down to roughly $2\%$ faster at $l = 32768$, which is the opposite trend from the single-threaded setup.
Two library-level differences explain these results.

First, there is a library-level difference in the cost of obtaining the permutation, which explains why the proposed optimization sometimes works in PyTorch, but never works in NumPy.
The proposed optimization's one mandatory global sort must be an indexed sort (\mintinline{py}{argsort}): it needs not only the sorted values but the original position of each one, so that every value can later be assigned to the interval that it came from.
We already showed in \autoref{sec8.1} that the naive ``swap-doubling'' argument does not hold up under isolation: pairing a value with an index adds only a modest overhead, not a full $2\times$.
The actual cost of an indexed sort relative to a plain sort in NumPy depends strongly on $l$.
For $l \lesssim 1024$--$2048$, \mintinline{py}{numpy.argsort} is roughly as fast as \mintinline{py}{numpy.sort}, or even slightly faster.
Beyond that point the ratio grows steadily, reaching roughly $2\times$ in double precision and $4$ to $4.5\times$ in single precision by $l=32768$.
This is consistent with the $2.3\times$ overhead of the proposed optimization over the naive one observed in \autoref{table:presort_optimization_runtimes} at the same $l$.

This is not primarily a matter of doing more work: an indexed and a non-indexed sort perform the exact same number of comparisons and swaps at every $l$.
Rather, it is a matter of how that work accesses memory, compounded by a second, larger effect specific to real-world implementations.
To separate the two effects, we implemented and compiled with Numba three quicksort variants sharing the exact same introsort skeleton (median-of-three pivot, insertion-sort cutoff), so that Numba never invokes a vectorized \emph{single instruction, multiple data} (SIMD) fast path for any of them, and so that all three perform identical comparison and swap counts at every $l$ (verified exactly):
\begin{itemize}
    \item \mintinline{py}{direct}: a plain value sort.
    Comparisons and swaps both act directly and sequentially on the values array.
    This mirrors the kernel behind \mintinline{py}{numpy.sort}.
    \item \mintinline{py}{indirect}: an indexed sort in which the values array is read-only, and only an index array is permuted, with comparisons dereferencing through it.
    This mirrors the kernel behind \mintinline{py}{numpy.argsort}.
    \item \mintinline{py}{keyvalue}: an indexed sort in which the values and the indices are permuted together on every swap, so comparisons stay direct and sequential on the values array, exactly like \mintinline{py}{direct}.
    This mirrors PyTorch's sorting kernel, used by both \mintinline{py}{sort} and \mintinline{py}{argsort}.
\end{itemize}
\autoref{table:sort_argsort_mechanism_isolation} reports the wall-clock runtime of the three variants for every $l$ in our grid, along with the \mintinline{py}{keyvalue}/\mintinline{py}{direct} and \mintinline{py}{indirect}/\mintinline{py}{direct} ratios.

\begin{table}[tbp]
    \caption{
        Wall-clock runtime of the \mintinline{py}{direct}, \mintinline{py}{keyvalue}, and \mintinline{py}{indirect} quicksort variants, as a function of $l$, along with the \mintinline{py}{keyvalue}/\mintinline{py}{direct} and \mintinline{py}{indirect}/\mintinline{py}{direct} ratios.
        Units ($\mu$s or ms) are given per row.
        The four largest values of $l$ (from $65536$ onward, below the second horizontal rule) are well beyond any series length this paper's pipeline ever processes.
        They are included only to show where the \mintinline{py}{indirect}/\mintinline{py}{direct} ratio is headed asymptotically.
    }
    \label{table:sort_argsort_mechanism_isolation}
    \centering
    \begin{tabular}{rccccc}
        \toprule
        $l$ & \mintinline{py}{direct} & \mintinline{py}{keyvalue} & \mintinline{py}{indirect} & \mintinline{py}{keyvalue}/\mintinline{py}{direct} & \mintinline{py}{indirect}/\mintinline{py}{direct} \\
        \midrule
        4       & 0.351$\mu$s & 0.600$\mu$s & 0.594$\mu$s & 1.71$\times$ & 1.69$\times$ \\
        8       & 0.360$\mu$s & 0.607$\mu$s & 0.613$\mu$s & 1.69$\times$ & 1.70$\times$ \\
        16      & 0.408$\mu$s & 0.662$\mu$s & 0.679$\mu$s & 1.62$\times$ & 1.66$\times$ \\
        32      & 0.480$\mu$s & 0.762$\mu$s & 0.784$\mu$s & 1.59$\times$ & 1.63$\times$ \\
        64      & 0.676$\mu$s & 0.959$\mu$s & 1.07$\mu$s  & 1.42$\times$ & 1.58$\times$ \\
        128     & 1.16$\mu$s  & 1.55$\mu$s  & 1.71$\mu$s  & 1.33$\times$ & 1.47$\times$ \\
        256     & 2.15$\mu$s  & 2.71$\mu$s  & 3.03$\mu$s  & 1.26$\times$ & 1.41$\times$ \\
        512     & 4.35$\mu$s  & 5.18$\mu$s  & 5.88$\mu$s  & 1.19$\times$ & 1.35$\times$ \\
        1024    & 9.20$\mu$s  & 10.7$\mu$s  & 12.2$\mu$s  & 1.17$\times$ & 1.32$\times$ \\
        2048    & 22.5$\mu$s  & 29.5$\mu$s  & 28.7$\mu$s  & 1.31$\times$ & 1.28$\times$ \\
        4096    & 123$\mu$s   & 130$\mu$s   & 157$\mu$s   & 1.06$\times$ & 1.28$\times$ \\
        8192    & 365$\mu$s   & 378$\mu$s   & 457$\mu$s   & 1.04$\times$ & 1.25$\times$ \\
        16384   & 0.843ms     & 0.875ms     & 1.06ms      & 1.04$\times$ & 1.26$\times$ \\
        32768   & 1.85ms      & 1.96ms      & 2.34ms      & 1.06$\times$ & 1.26$\times$ \\
        \midrule
        65536   & 3.95ms      & 4.25ms      & 5.04ms      & 1.08$\times$ & 1.27$\times$ \\
        262144  & 18.1ms      & 18.9ms      & 23.0ms      & 1.05$\times$ & 1.28$\times$ \\
        1048576 & 81.4ms      & 85.1ms      & 105ms       & 1.04$\times$ & 1.29$\times$ \\
        4194304 & 356ms       & 372ms       & 501ms       & 1.04$\times$ & 1.41$\times$ \\
        \bottomrule
    \end{tabular}
\end{table}

The \mintinline{py}{keyvalue}/\mintinline{py}{direct} ratio stays modest throughout, from about $1.7\times$ at the smallest $l$ (dominated by fixed per-call overhead) down to $1.0$--$1.1\times$ for $l \geq 4096$, and it stays in that same narrow band all the way out to $l=4\,194\,304$: pairing a value with an index, on its own, is nowhere near twice the cost of a plain swap, at any scale that we tested.
This is directly relevant to the roughly $2\times$ gap between \mintinline{py}{QuantFloat64} and \mintinline{py}{MomentQuant("exact", "intervals")} observed earlier in \autoref{table:implementation_choice_runtimes} (\autoref{sec8.1}): since PyTorch's sorting kernel follows the \mintinline{py}{keyvalue} strategy while NumPy's \mintinline{py}{sort} follows \mintinline{py}{direct}, the \mintinline{py}{keyvalue}/\mintinline{py}{direct} ratio here rules out swap-counting as the explanation for that gap as well.
The gap in \autoref{table:implementation_choice_runtimes} is more likely dominated by PyTorch's per-call dispatch and tensor-allocation overhead than by the sorting algorithm itself.
The \mintinline{py}{indirect}/\mintinline{py}{direct} ratio is somewhat larger and follows a mild U-shape: high at very small $l$ (again, fixed overhead), settling into a broad plateau of roughly $1.25$--$1.29\times$ from $l=4096$ up to $l=262144$, and then climbing again at the two largest sizes tested, reaching $1.41\times$ at $l=4\,194\,304$.
This ratio is past the plateau, but still well short of the several-fold gap that a purely cache-miss-driven effect would be expected to reach at even larger scales.
The entries at $l=1024$, $2048$, and $4096$ are also visibly noisier across repeated runs than their neighbors (up to a $3.5\times$ spread between the fastest and slowest of $40$ repeats, against roughly $1.1$--$1.2\times$ for most other $l$).
This spread did not shrink when we quadrupled the number of repeats, which suggests a repeatable effect tied to this transition region.
A plausible explanation, but not verified, is neighboring jobs of very different sizes in the shuffled run order interacting with the CPU's frequency scaling, rather than one-off measurement noise.
This confirms that indirection has a real, cache-driven cost that keeps growing with $l$ well beyond this paper's operational range ($l \leq 32768$), but one that remains far too small, at least at the scales tested here, to explain the roughly $2$ to $4.5\times$ gap observed with the real \mintinline{py}{numpy.argsort} at the same $l$.
The rest of that gap comes from a second effect that this SIMD-free isolation deliberately excludes: NumPy's \mintinline{py}{sort} has a vectorized SIMD fast path with no equivalent for \mintinline{py}{argsort}.
SIMD is true simultaneous parallel hardware-level execution and is different from multi-threading.
This fast path only pays off once $l$ is large enough, which is why the real \mintinline{py}{argsort}/\mintinline{py}{sort} gap is negligible, or even inverted, below roughly $l=1024$--$2048$, and only opens up beyond that.
This is a much sharper transition than the SIMD-free \mintinline{py}{indirect}/\mintinline{py}{direct} ratio in \autoref{table:sort_argsort_mechanism_isolation}, which grows far more gradually across the same range and remains well under $2\times$ even at $l=4\,194\,304$.

Second, PyTorch's kernels generally parallelize across independent rows of a batch, and the naive and proposed implementations are not equally well suited to this parallelism.
The naive implementation performs many independent sorts, one per interval, most of them short, especially near the leaves of the depth hierarchy.
In this case, there is little work per call for several threads to split up, so more threads barely help at small $l$ (its wall-clock time only fell by about $1.2\times$ going from $1$ to $8$ threads at $l=512$).
The proposed optimization performs a handful of large operations per representation instead (one global sort, plus one redistribution pass per depth level).
This is coarser-grained work that several threads can split up efficiently even when $l$ is small (about $1.75\times$ faster at $l=512$ going from $1$ to $8$ threads).
This asymmetry is largest exactly where the proposed optimization's multithreaded advantage is biggest (small $l$), and shrinks as $l$ grows and the naive implementation's individual interval sorts become large enough to parallelize well too.

\section{Conclusion}\label{sec10}

The start of this research work set out from a narrow observation: Quant's reference implementation, despite being fast and simple overall, commits to a single computational structure that is not well-matched to either end of the series-length spectrum on a single CPU core.
We addressed both ends.
For short series, a theoretical cost model showed that choosing between two loop orderings, rather than committing to Quant's fixed one, removes most of the avoidable dispatch overhead.
For long series, our main contribution, we replaced exact per-interval sorting with a moment-based approximation of the same quantiles via the Cornish-Fisher expansion, trading its $\mathcal{O}(l \cdot \log(l))$ sorting cost for an $\mathcal{O}(l)$ one.
Moreover, for each mode, a calibrated dispatch heuristic automatically tries to choose the faster loop order, although it does not always succeed.

We derived a complete analysis of the computational complexities of Quant and MomentQuant.
Thanks to this analysis, it becomes much easier to compare these two algorithms to other time series classification algorithms in terms of computational complexity.
Indeed, we provided evidence that empirical comparisons have important limitations: the choice of a library, or even a single function from a library, can have a substantial impact on runtime, independently of the theoretical computational complexity.

At full UCR-archive scale, MomentQuant's better asymptotic complexity translates into a real, measurable advantage: roughly $1.7\times$ faster feature extraction than exact mode, and further still against Quant's own reference implementation, at a small but genuine accuracy cost (under half a percentage point on average), and one that it does not pay uniformly, since it still wins outright on over a third of the data sets tested.
That said, the practical picture is more nuanced than the asymptotic argument alone would suggest: once downstream classifier fitting is counted, which dominates end-to-end cost on this archive, most of the runtime gap disappears, only to re-emerge once inference-only cost is isolated.
Splitting the comparison by series length shows a real but non-monotonic pattern rather than the clean, uniformly-widening advantage a length-only reading of the cost model would predict.
MomentQuant's advantage is therefore most relevant where feature extraction itself, not classifier fitting, is the bottleneck: at inference time, in deployment, or wherever series are long enough that sorting cost stops being a rounding error.

Several questions remain open.
We did not enforce the Cornish-Fisher expansion's domain of validity per interval, nor test whether a genuinely linear-time counting sort would close the residual gap identified for the exact mode's alternative kernel.
Both are natural targets for tightening MomentQuant's worst-case behavior further.
More broadly, this study evaluated MomentQuant as a drop-in replacement for Quant's own quantile computation.
Whether the same moment-based idea transfers to other interval- or quantile-based feature extractors is left to future work.
This study is also restricted to univariate series, following both Quant's own reference implementation and the univariate UCR archive used throughout this paper.
Extending our cost model to multivariate series is not immediate: treating each channel independently and concatenating its own quantile features, as most interval-based methods do, would simply multiply every result derived here by the number of channels, leaving the complexity class in $l$ unchanged, but the dispatch-overhead constants of \autoref{sec5} and \autoref{sec6.3}, and the resulting crossover sample size, would likely need to be recalibrated, since channels are a natural additional axis to batch alongside the number of series $n$.
A multivariate design that builds intervals jointly across channels, rather than independently per channel, would instead need a genuinely new cost model.

Overall, we believe this work shows that a small, principled amount of approximation, applied where it costs the least and helps the most, is a practical way to make an already fast method for time series classification even faster.

\section{Declarations}

All the data sets used in this study are publicly available from the UCR Time Series Archive \citep{dauUCRTimeSeries2019e}.
We would like to thank Professor Eamonn Keogh and all the people who have contributed to this archive.
The data sets can be downloaded in several ways, either manually from the website\footnote{\url{https://timeseriesclassification.com/dataset.php}} or automatically using libraries such as the \emph{aeon} \citep{middlehurstAeonPythonToolkit2024} Python package.
The whole source code supporting this study is publicly available on a GitHub repository,\footnote{\url{https://github.com/johannfaouzi/moment-quant}} with detailed instructions to reproduce all the experiments.
The results are also directly available, so that other researchers can easily compare their algorithms with ours.
The repository is under the BSD 3-Clause License, making it reusable by other researchers.

Generative artificial intelligence has been used in this study, both for generating code and analyzing the results.
Nonetheless, the authors agree to be accountable for all the materials associated with this study (this manuscript and the provided public GitHub repository).

No funding was received for conducting this study.
The authors have no relevant financial or non-financial interests to disclose.
The authors have no conflicts of interest to declare that are relevant to the content of this article.
The authors certify that they have no affiliations with or involvement in any organization or entity with any financial interest or non-financial interest in the matter or materials discussed in this manuscript.
The authors have no financial or proprietary interests in any material discussed in this article.

\newpage

\begin{appendices}

\section{Proofs}\label{secA1}

This appendix collects the proofs of every theorem and lemma stated in \autoref{sec4}, \autoref{sec5}, and \autoref{sec6}, in the order in which they appear in the main text.

\begin{proof}[Proof of \autoref{theorem:1}]
    \label{proof:theorem:1}
    At any level $r \in \{0, \ldots, k\}$, Quant builds $2^r$ base intervals.
    When $r > 0$, Quant also builds $2^r - 1$ shifted intervals.
    In total, Quant builds $1 + \sum_{r=1}^{k} 2^r + 2^r - 1 = 2^{k+2} - k - 3$ intervals.
    However, at level $r = k$, the shifted intervals are actually built if and only if the median of the widths of the base intervals at level $k$ is greater than $1$, which occurs if and only if $l \geq 1.5 \cdot 2^k$.
    Therefore, in this case, Quant builds $2^{k+2} - k - 3 - \left( 2^k - 1 \right) = 3 \cdot 2^k - k - 2$.
    In conclusion, Quant exactly builds $N_i(l, d)$ intervals with:
    $$
        N_i(l, d) = \begin{cases}
            2^{k+2} - k - 3 & \text{if } l \geq 1.5 \cdot 2^k \\
            3 \cdot 2^k - k - 2 & \text{if } l < 1.5 \cdot 2^k
        \end{cases}
    $$
    In both cases, the dominant term is proportional to $2^k$.
    Therefore:
    $$
        N_i(l, d) = \Theta \left( 2^k \right)
    $$
\end{proof}

\begin{proof}[Proof of \autoref{theorem:2}]
    \label{proof:theorem:2}
    At any level $r \in \{0, \ldots, k\}$, the widths of the base intervals sum to exactly $l$ and the widths of the shifted intervals sum to $l - \lceil l \cdot 2^{-r} \rceil$, except if $r = k$ and $l < 1.5 \cdot 2^k$, in which case there are no shifted intervals, so their widths sum to $0$.
    For any level $r \in \{0, \ldots, k\}$, define $\epsilon_r = \lceil l \cdot 2^{-r} \rceil - l \cdot 2^{-r}$.
    Summing over the levels, we have two distinct cases:
    \begin{itemize}
        \item if $l \geq 1.5 \cdot 2^k$, we have:
        $$
            \sum_{r=0}^{k} \left[ 2l - l \cdot 2^{-r} - \epsilon_r \right] = \left( 2k + 2^{-k} \right) \cdot l - \sum_{r=1}^k \epsilon_r
        $$
        \item if $l < 1.5 \cdot 2^k$, we have:
        $$
            \sum_{r=0}^{k-1} \left[ 2l - l \cdot 2^{-r} - \epsilon_r \right] + l = \left( 2k - 1 + 2^{-(k-1)} \right) \cdot l - \sum_{r=1}^{k-1} \epsilon_r
        $$
    \end{itemize}
    Therefore, the total width $W(l, d)$ is equal to:
    $$
        W(l, d) = \begin{cases}
            \displaystyle \left( 2k + 2^{-k} \right) \cdot l - \sum_{r=1}^k \left( \left\lceil l \cdot 2^{-r} \right\rceil - l \cdot 2^{-r} \right) & \text{if } l \geq 1.5 \cdot 2^k \\
            \displaystyle \left( 2k - 1 + 2^{-(k-1)} \right) \cdot l - \sum_{r=1}^{k-1} \left( \left\lceil l \cdot 2^{-r} \right\rceil - l \cdot 2^{-r} \right) & \text{if } l < 1.5 \cdot 2^k
        \end{cases} = \Theta \left( k \cdot l \right)
    $$
    In both cases, the first term is $\Theta(k \cdot l)$, and the correction term is negative and $\mathcal{O}(k)$, thus dominated by the first term:
    $$
        -k \leq - \sum_{r=1}^k \underbrace{\left\lceil l \cdot 2^{-r} \right\rceil - l \cdot 2^{-r}}_{\in [0, 1)} \leq - \sum_{r=1}^{k-1} \underbrace{\left\lceil l \cdot 2^{-r} \right\rceil - l \cdot 2^{-r}}_{\in [0, 1)}\leq 0
    $$
    Therefore:
    $$
        W(l, d) = \Theta \left( k \cdot l \right)
    $$
\end{proof}

\begin{proof}[Proof of \autoref{theorem:3}]
    \label{proof:theorem:3}
    The number of quantiles extracted for an interval of any length $m$, denoted by $n_q(m)$ is:
    $$
        n_q(m, \nu) = 1 + \left\lfloor \frac{(m - 1)}{\nu} \right\rfloor = 1 + \frac{(m - 1)}{\nu} - \left\{ \frac{(m - 1)}{\nu} \right\}
    $$
     Summing over all the $N_i(l, d)$ intervals (\autoref{theorem:1}) and using the fact that $\sum_{j} m_j = W(l, d)$ (\autoref{theorem:2}), the total number of quantiles is equal to:
    $$
        N_q(l, d, \nu) = \sum_{j} n_q(m_j, \nu) = N_i(l, d) + \frac{W(l, d) - N_i(l, d)}{\nu} - \sum_j \left\{ \frac{m_j - 1}{\nu} \right\}
    $$
\end{proof}

\begin{proof}[Proof of \autoref{theorem:4}]
    \label{proof:theorem:4}
    We split the proof in three parts (one for each case).

    \noindent\ul{Case $\nu =1$}. It is trivial since $\left\{ \frac{(m - 1)}{\nu} \right\} = 0$ for any $m \geq 1$.

    \noindent\ul{Case $d =1$}. It is trivial because there is a single interval (for the whole series) and the single condition is $\left\{ \frac{l - 1}{\nu} \right\} = 0$, which is obtained if and only if $l - 1$ is exactly divisible by $\nu$.

    \noindent\ul{Case $d \geq 2$ and $\nu \geq 2$}.
    For any level $r \in \{1, \ldots, d - 1\}$:
    \begin{itemize}
        \item if $l$ is exactly divisible by $2^r$, then all the intervals have the same width $l \cdot 2^{-r}$, and
        \item if $l$ is not exactly divisible by $2^r$, then there are $2^r - s_r$ base intervals of length $q_r = \lfloor l \cdot 2^{-r} \rfloor $ and $s_r$ base intervals of length $q_r + 1$, with $s_r > 0$.
    \end{itemize}
    Thus, we have the two following cases:
    \begin{itemize}
        \item If there exists any $r \in \{1, \ldots, d - 1\}$ such that $l$ is not exactly divisible by $2^r$, then both $q_r - 1$ and $q_r$ would have to be exactly divisible by $\nu$, which implies that 1 would need to be divisible by $\nu$, which is impossible for any $\nu \geq 2$.
        \item If $l$ is exactly divisible by $2^{d-1}$, then write $l = a \cdot 2^{d-1}$ with $a$ being a positive integer. Looking at the two deepest levels ($d - 1$ and $d - 2$), we would need both $a - 1$ and $2a - 1$ to be exactly divisible by $\nu$, which implies that 1 would need to be divisible by $\nu$, which is impossible for $\nu \geq 2$.
    \end{itemize}

    For each individual term, we have:
    $$
        0 \leq \left\{ \frac{m_j - 1}{\nu} \right\} \leq \frac{\nu - 1}{\nu}
    $$
    Summing over all the $N_i(l, d)$ provides the bound:
    $$
        0 \leq E(l, d, \nu) \leq \frac{\nu - 1}{\nu} \cdot N_i(l, d)
    $$

    For $l = a \cdot 2^{d-1} \cdot \nu$, with $a$ being any positive integer, for any level $r \in \{0, \ldots, d - 1\}$, all the intervals of level $r$ have the same width $a \cdot 2^{d - r - 1} \cdot \nu$, and we have:
    $$
        \left\{ \frac{a \cdot 2^{d - r - 1} \cdot \nu - 1}{\nu} \right\} = \frac{\nu - 1}{\nu}
    $$
    Summing over all the intervals, we have:
    $$
        E(a \cdot 2^{d-1} \cdot \nu, d, \nu) = \frac{\nu - 1}{\nu} \cdot N_i(l, d)
    $$
\end{proof}

\begin{proof}[Proof of \autoref{theorem:5}]
    \label{proof:theorem:5}
    An interval built at level $r$ has width $1$ only if $r = k$: for any level $r < k$, both $2^r \leq 2^{k - 1}$ and $l \geq 2^k$ hold, so $l \cdot 2^{-r} \geq 2$, hence every base interval at level $r$ has width $\lfloor l \cdot 2^{-r} \rfloor \geq 2$.
    Every shifted interval at level $r$ reuses a base interval's width at that level (shown below), hence also has width at least $2$.

    If $k = d - 1 < \lfloor \log_2(l) \rfloor$, then $l \geq 2^{k+1}$, so $\lfloor l \cdot 2^{-k} \rfloor \geq 2$ as well, and every interval at level $k$ (base or shifted) has width at least $2$: $N_i^{(1)}(l, d) = 0$.

    Otherwise, $k = \lfloor \log_2(l) \rfloor$, so $2^k \leq l < 2^{k+1}$. Write $l = 2^k + \rho$ with $0 \leq \rho < 2^k$. The base intervals at level $k$ are given by $\text{indices}_i = \lfloor i \cdot l \cdot 2^{-k} \rfloor$ for $i = 0, \ldots, 2^k$, which partitions $[0, l]$ into $2^k$ pieces of which exactly $\rho$ have width $2$ and the remaining $2^k - \rho = 2^{k+1} - l$ have width $1$.

    If $l < 1.5 \cdot 2^k$ (i.e., $\rho < 2^{k-1}$), no shifted intervals are built at level $k$ (\autoref{theorem:1}), so $N_i^{(1)}(l, d) = 2^{k+1} - l$.

    If $l \geq 1.5 \cdot 2^k$ (i.e., $\rho \geq 2^{k-1} \geq 1$, since $k \geq 1$ whenever shifted intervals exist), shifted intervals are built at level $k$ by shifting and then dropping the last of the $2^k$ base intervals (\autoref{theorem:1}). The last base interval has width $\text{indices}_{2^k} - \text{indices}_{2^k - 1} = l - \left( l - \lceil l \cdot 2^{-k} \rceil \right) = \lceil l \cdot 2^{-k} \rceil = 2$ (using $\lfloor x - y \rfloor = x - \lceil y \rceil$ for integer $x$, and $\rho > 0$). Consequently, the $2^k - 1$ shifted intervals reuse the base intervals' widths minus this one dropped width-$2$ instance: $\rho - 1$ of width $2$ and $2^k - \rho$ of width $1$, contributing $2^k - \rho = 2^{k+1} - l$ further length-one intervals. In total: $N_i^{(1)}(l, d) = 2 \cdot \left( 2^{k+1} - l \right)$.
\end{proof}

\begin{proof}[Proof of \autoref{theorem:6}]
    \label{proof:theorem:6}
    At any level $r \in \{0, \ldots, k\}$, there are exactly $2^r$ intervals of length $l \cdot 2^{-r}$.
    Their total sort cost is thus:
    $$
        \text{base\_ideal}(r) = 2^r \cdot l \cdot 2^{-r} \cdot \log_2(l \cdot 2^{-r}) = l \cdot \left( \log_2(l) - r \right)
    $$
    At any level $r \in \{1, \ldots, k - 1\}$, there are exactly $2^r - 1$ shifted intervals of length $l \cdot 2^{-r}$.
    Their total sort cost is thus:
    $$
        \text{shift\_ideal}(r) = \left( 2^r - 1 \right) \cdot l \cdot 2^{-r} \cdot \log_2(l \cdot 2^{-r}) = \left( 1 - 2^{-r} \right) \cdot l \cdot \left( \log_2(l) - r \right)
    $$
    At level $r = k$, shifted intervals are only included if the median width of the base intervals at level $k$ is greater than $1$, which occurs if and only if $l \geq 1.5 \cdot 2^k$.
    The above formula for $\text{shift\_ideal}(r)$ actually still holds for $r = k$ for $l = 2^k$ because $\log_2(2^k) - k = 0$.
    Thus, we can use it for any $l$ in this setting because we assumed that $l$ is exactly divisible by $2^k$.

    Therefore, at any level $r \in \{1, \ldots, k\}$:
    $$
        \text{base\_ideal}(r) + \text{shift\_ideal}(r) = l \cdot \left( \log_2(l) - r \right) \left[ 1 + (1 - 2^{-r}) \right] = l \cdot \left( \log_2(l) - r \right) \cdot \left( 2 - 2^{-r} \right)
    $$
    Thus, the total sort cost $T_s^{r*}(l, d)$ is proportional to:
    \begin{align*}
        T_s^{r*}(l, d) &\propto \text{base\_ideal}(0) + \sum_{r=1}^k \text{base\_ideal}(r) + \text{shift\_ideal}(r)\\
        &\propto l \cdot \log_2(l) + \sum_{r=1}^k l \left( \log_2(l) - r \right) \cdot \left( 2 - 2^{-r} \right)\\
        &\propto l \cdot \log_2(l) \cdot \left[ 1 + \sum_{r=1}^k (2 - 2^{-r}) \right] - l \cdot \sum_{r=1}^k r \cdot (2 - 2^{-r})\\
        &\propto l \cdot \log_2(l) \cdot \left[ 2k + 2^{-k} \right] - l \left[ k^2 + k - 2 + (k + 2) \cdot 2^{-k} \right]\\
        T_s^{r*}(l, d) &\propto l \cdot \left[ \left( 2k + 2^{-k} \right) \cdot \log_2(l) - \left( k^2 + k - 2 + (k + 2) \cdot 2^{-k} \right) \right]
    \end{align*}
    Defining $c_1$ as the empirical time-per-unit-of-sort-work constant, we obtain the desired result:
    $$
        T_s^{r*}(l, d) = c_1 \cdot l \cdot \left[ \left( 2k + 2^{-k} \right) \cdot \log_2(l) - \left( k^2 + k - 2 + (k + 2) \cdot 2^{-k} \right) \right]
    $$
\end{proof}

\begin{proof}[Proof of \autoref{theorem:7}]
    \label{proof:theorem:7}
    For any level $r \in \{1, \ldots, k\}$, all the interval lengths might not be equal to $l \cdot 2^{-r}$ since it might not be an integer.
    Let $q_r = \lfloor l \cdot 2^{-r} \rfloor$ and $s_r$ (with $0 \leq s_r < 2^r$) be the quotient and the remainder of the Euclidean division of $l$ by $2^r$ respectively.
    With the implementation chosen in Quant, there are exactly $2^r - s_r$ base intervals of length $q_r$ and $s_r$ base intervals of length $q_r + 1$.
    The width of the last base interval, denoted by $w_r$, is equal to $q_r + 1$ if $s_r > 0$ else $q_r$.
    We mention this information because the total width of all the shifted intervals is equal to the total width of the first $2^r - 1$ base intervals (i.e., all the base intervals except the last one).
    Indeed, the $2^r - 1$ shifted intervals are the first $2^r - 1$ base intervals shifted by $\lceil l \cdot 2^{-r-1} \rceil$.

    To simplify the equations, we define $f(m) = m \cdot \log_2(m)$ as the sort cost of an interval of length $m$.
    For any level $r \in \{1, \ldots, k\}$, the total sort cost of the base intervals is:
    $$
        \text{base}(r) = \left( 2^r - s_r \right) \cdot f(q_r) + s_r \cdot f(q_r + 1)
    $$
    For any level $r \in \{1, \ldots, k - 1\}$, the total sort cost of the shifted intervals is:
    $$
        \text{shift}(r) = \text{base}(r) - f(w_r)
    $$
    Similarly to the proofs of \autoref{theorem:1} and \autoref{theorem:2}, we need to distinguish two cases:
    \begin{itemize}
        \item if $l < 1.5 \cdot 2^k$, there are no shifted intervals at level $r = k$, thus:
        $$
            \text{shift}(k) = 0
        $$
        \item if $l \geq 1.5 \cdot 2^k$, there are shifted intervals at level $r = k$, and the formula above still holds, thus:
        $$
            \text{shift}(k) = \text{base}(k) - f(w_k)
        $$
    \end{itemize}

    \noindent\ul{Case $l \geq 1.5 \cdot 2^k$}. Summing over the levels, the total sort cost $T(l)$ is equal to:
    $$
        T_s^r(l, d) = f(l) + \sum_{r=1}^k 2 \cdot \text{base}(r) - f(w_r)
    $$

    Now, we need to bound the gap between the ideal case (when $l$ is exactly divisible by $2^k$) and the arbitrary case (when $l$ is not exactly divisible by $2^r$).
    Let's recall the total sort cost of the all the base intervals at any level $r \in \{0, \ldots, k\}$ in the ideal case:
    $$
        \text{base\_ideal}(r) = 2^r \cdot f(l \cdot 2^{-r})
    $$
    At level $r=0$, the whole series is used, so the ideal and arbitrary cases exactly match:
    $$
        \text{base}(0) = \text{base\_ideal}(0) = f(l)
    $$
    Since $f$ is convex, by the standard Lagrange remainder for linear interpolation, for $\theta = s_r \cdot 2^{-r} \in [0, 1)$, there exists $\xi \in (q_r, q_r + 1)$ such that:
    $$
        (1 - \theta) \cdot f(q_r) + \theta \cdot f(q_r + 1) - f(q_r + \theta) = \frac{f''(\xi)}{2} \cdot \theta \cdot (1 - \theta)
    $$
    Multiplying by $2^r$ both sides of the equation, we have:
    $$
        0 \leq \text{base}(r) - \text{base\_ideal}(r) = 2^r \cdot \frac{f''(\xi)}{2} \cdot \theta \cdot (1 - \theta)
    $$
    Using the facts that $\theta \cdot (1 - \theta) \leq 1/4$ (because $1/4$ is the maximum value of function $x \mapsto x \cdot (1-x)$, attained at $1/2$) and that $f''(\xi) \leq f''(q_r) \leq f''(1) = 1 / \log(2)$ (because $f''(m) = 1 / (m \cdot \log(2))$ is decreasing), we have:
    $$
        0 \leq \text{base}(r) - \text{base\_ideal}(r) \leq \frac{2^r}{8 \log(2)}
    $$
    Summing over all the levels, we have:
    $$
        0 \leq \sum_{r=0}^k \text{base}(r) - \text{base\_ideal}(r) \leq \sum_{r=1}^k \frac{2^r}{8 \log(2)} = \mathcal{O}\left( 2^k \right)
    $$

    Now, let's focus on the shifted intervals, and more specifically on $f(w_r)$, and let's compare it to its ideal counterpart $f(l \cdot 2^{-r})$.
    Recall that $q_r = \lfloor l \cdot 2^{-r} \rfloor$ and that $w_r$ is equal to either $q_r$ (if $s_r = 0$) or $q_r + 1$ (if $s_r > 0$).
    Thus, we have $l \cdot 2^{-r} \in [q_r, q_r + 1]$ and $w_r \in [q_r, q_r + 1]$.
    Since $f$ is differentiable, using the mean value inequality and the fact that $f$ is convex, we have:
    $$
        \lvert f(w_r) - f(l \cdot 2^{-r}) \rvert \leq \max_{[q_r, q_r + 1]} \lvert f' \rvert = f'(q_r + 1) = \log_2(q_r + 1) + \frac{1}{\log(2)}
    $$
    Since $q_r = \lfloor l \cdot 2^{-r} \rfloor \leq l \cdot 2^{-r}$, we have:
    $$
        \log_2(q_r + 1) \leq \log_2(l \cdot 2^{-r} + 1) \leq \log_2(l \cdot 2^{-r}) + \log_2(2) = \log_2(l) - r + 1
    $$
    Summing over all the levels, we have:
    $$
        \sum_{r=1}^k \left[ \log_2(q_r + 1) + \frac{1}{\log(2)} \right] \leq k \cdot \left( \log_2(l) + 1 + \frac{1}{\log(2)} \right) - \frac{k \cdot (k + 1)}{2} = \mathcal{O}\left( k \cdot \log(l) \right)
    $$

    Putting it together, the total deviation from the ideal case is:
    \begin{align*}
        T_s^r(l, d) - T_s^{r*}(l, d) &= c_1 \cdot \left( \sum_{r=1}^k (\text{base}(r) - \text{base\_ideal}(r)) + (\text{shift}(r) - \text{shift\_ideal}(r)) \right)\\
        &= 2 \cdot c_1 \cdot \left( \sum_{r=1}^k (\text{base}(r) - \text{base\_ideal}(r)) + \sum_{r=1}^k (f(l \cdot 2^{-r}) - f(w_r)) \right)\\
        T_s^r(l, d) - T_s^{r*}(l, d) &= c_1 \cdot \mathcal{O} \left( 2^k \right) + \mathcal{O}\left( k \cdot \log(l) \right)
    \end{align*}
    Thus, the total sort cost for an arbitrary length $l$ is:
    $$
        T_s^r(l, d) = c_1 \cdot l \cdot \left[ \left( 2k + 2^{-k} \right) \cdot \log_2(l) - \left( k^2 + k - 2 + (k + 2) \cdot 2^{-k} \right) \right] + \mathcal{O} \left( 2^k \right) + \mathcal{O}\left( k \cdot \log(l) \right)
    $$

    \noindent\ul{Case $l < 1.5 \cdot 2^k$}. In this case, the correction term only removes the following positive term:
    $$
        \text{base}(k) - f(w_k) = \left( 2^k - s_k \right) \cdot f(q_k) + s_k \cdot f(q_k + 1) - f(w_k)
    $$
    Therefore, it does not change the bound obtained when $l \geq 1.5 \cdot 2^k$.

    \noindent\ul{General case, any arbitrary $l$}. Since the result holds for both cases, we have:
    $$
        T_s^r(l, d) = c_1 \cdot l \cdot \left[ \left( 2k + 2^{-k} \right) \cdot \log_2(l) - \left( k^2 + k - 2 + (k + 2) \cdot 2^{-k} \right) \right] + \mathcal{O} \left( 2^k \right) + \mathcal{O}\left( k \cdot \log(l) \right)
    $$
\end{proof}

\begin{proof}[Proof of \autoref{theorem:8}]
    \label{proof:theorem:8}
    The proof starts with the results of \autoref{theorem:7}:
    $$
        T_s^r(l, d) = l \cdot \left[ \left( 2k + 2^{-k} \right) \cdot \log_2(l) - \left( k^2 + k - 2 + (k + 2) \cdot 2^{-k} \right) \right] + \mathcal{O} \left( 2^k \right) + \mathcal{O}\left( k \cdot \log(l) \right)
    $$
    with $k = e - 1$.
    We split the proof in two parts (one for each case).

    \noindent\textbf{Saturated regime}\qquad In the saturated regime ($l \geq 2^{d-1}$), we have $k = d - 1 \leq \log_2(l)$. We derive both an upper bound and a lower bound.
    For the first term, we have:
    \begin{align*}
        \left( 2k + 2^{-k} \right) \cdot \log_2(l) \leq (2k + 1) \log_2(l) = \mathcal{O} (d \cdot \log(l))\\
        0 \leq \left( k^2 + k - 2 + (k + 2) \cdot 2^{-k} \right) = \mathcal{O} \left( k^2 \right) \subseteq \mathcal{O} (d \cdot \log(l))
    \end{align*}
    Therefore, we have:
    $$
        l \cdot \left[ \left( 2k + 2^{-k} \right) \cdot \log_2(l) - \left( k^2 + k - 2 + (k + 2) \cdot 2^{-k} \right) \right] = \mathcal{O} \left( d \cdot l \cdot \log(l) \right)
    $$
    Both correction terms being dominated by the first term, we obtain the desired upper bound:
    $$
        T_s^r(l, d) = \mathcal{O} \left( d \cdot l \cdot \log(l) \right)
    $$

    For the lower bound, let's view the bracketed expression as a function of $x = \log_2(l)$:
    $$
        g(x) = \left( 2k + 2^{-k} \right) \cdot x - \left( k^2 + k - 2 + (k + 2) \cdot 2^{-k} \right)
    $$
    Since $g$ is affine in $x$ with slope $2k + 2^{-k} \geq k$, and since $k \leq \log_2(l)$ in the saturated regime, for any $x \geq k$ we have:
    $$
        g(x) = g(k) + (x - k) \cdot \left( 2k + 2^{-k} \right) \geq g(k) + (x - k) \cdot k
    $$
    We now lower-bound $g(k) = k^2 - k + 2 - 2^{1-k}$ for every integer $k \geq 0$.
    For $k = 0$, we have $g(0) = 0 = 0^2/2$.
    For $k = 1$, we have $g(1) = 1 \geq 1/2 = 1^2/2$.
    For $k \geq 2$, using $2^{1-k} \leq 2$, we have:
    $$
        g(k) \geq k^2 - k + 2 - 2 = k(k-1) \geq \frac{k^2}{2}
    $$
    In every case, $g(k) \geq k^2/2$. Combining with the slope bound above, for any $x \geq k$:
    $$
        g(x) \geq \frac{k^2}{2} + (x - k) \cdot k = kx - \frac{k^2}{2} \geq \frac{kx}{2}
    $$
    where the last step uses $x \geq k$. Substituting $x = \log_2(l)$, this holds for every $l$ in the saturated regime (i.e. every $l \geq 2^k$), not only at $l = 2^k$:
    $$
        \left( 2k + 2^{-k} \right) \cdot \log_2(l) - \left( k^2 + k - 2 + (k + 2) \cdot 2^{-k} \right) \geq \frac{k \cdot \log_2(l)}{2}
    $$
    Therefore, we have:
    $$
        l \cdot \left[ \left( 2k + 2^{-k} \right) \cdot \log_2(l) - \left( k^2 + k - 2 + (k + 2) \cdot 2^{-k} \right) \right] = \Omega(d \cdot l \cdot \log(l))
    $$
    Once again, both correction terms being dominated by the first term, we obtain the desired lower bound:
    $$
        T_s^r(l, d) = \Omega \left( d \cdot l \cdot \log(l) \right)
    $$
    With both the upper and lower bounds matching, we have:
    $$
        T_s^r(l, d) = \Theta \left( d \cdot l \cdot \log(l) \right)
    $$

    \noindent\textbf{Unsaturated regime}\qquad In the unsaturated regime ($l < 2^{d-1}$), we have $k = \lfloor \log_2(l) \rfloor$, hence $\log_2(l) = k + \{\log_2(l)\}$.
    Focusing on the terms inside the brackets, we have:
    \begin{align*}
        &\left( 2k + 2^{-k} \right) \cdot \log_2(l) - \left( k^2 + k - 2 + (k + 2) \cdot 2^{-k} \right)\\
        &= \left( 2k + 2^{-k} \right) \cdot (k + \{\log_2(l)\}) - \left( k^2 + k - 2 + (k + 2) \cdot 2^{-k} \right)\\
        &= k^2 + \underbrace{(2 k \cdot \{\log_2(l)\} - k + 2)}_{=\mathcal{O}(k)} + \underbrace{2^{-k} (\{\log_2(l)\} - 2)}_{o(1)}\\
        &= k^2 + \mathcal{O}(k^2) + o(1)\\
        &= \Theta \left( k^2 \right)\\
        &= \Theta \left( (\log(l))^2 \right)
    \end{align*}
    Therefore, for the first term, we have:
    $$
        l \cdot \left[ \left( 2k + 2^{-k} \right) \cdot \log_2(l) - \left( k^2 + k - 2 + (k + 2) \cdot 2^{-k} \right) \right] = \Theta \left( l \cdot (\log(l))^2 \right)
    $$
    Moreover, both correction terms are dominated by the first term:
    \begin{align*}
        \mathcal{O}\left( 2^k \right) &= \mathcal{O}(l)\\
        \mathcal{O}\left( k \cdot \log(l) \right) &= \mathcal{O}\left( (\log(l))^2 \right)
    \end{align*}
    Therefore, we have:
    $$
        T_s^r(l, d) = \Theta \left( l \cdot \left( \log(l) \right)^2 \right)
    $$
\end{proof}

\begin{proof}[Proof of \autoref{theorem:9}]
    \label{proof:theorem:9}
    An interval of width $1$ contributes no extraction cost: its single value is copied directly, with no sorting, no quantile interpolation, and no mean computed. For every other interval, of width $m > 1$, the extraction cost consists of two terms:
    \begin{itemize}
        \item the cost of extracting the quantiles, which is proportional to the number of quantiles, and
        \item the cost of computing the mean of the subseries, which is proportional to the series length $m$.
    \end{itemize}
    Thus, the cost of the extraction step for a single sorted subseries of width $m > 1$ is:
    $$
        \text{extraction}(m) = c_2 \cdot n_q(m, \nu) + c_3 \cdot m
    $$
    with $c_2$ being the empirical per-quantile extraction cost (in seconds/quantile), and $c_3$ being the empirical per-element extraction cost (in seconds/element).
    Summing over all the intervals of width strictly greater than $1$ (\autoref{theorem:5}), we have:
    $$
        T_e^r(l, d, \nu) = c_2 \cdot N_q^{>1}(l, d, \nu) + c_3 \cdot W^{>1}(l, d)
    $$
\end{proof}

\begin{proof}[Proof of \autoref{theorem:10}]
    \label{proof:theorem:10}
    We first show that restricting the count and the total width to intervals of width strictly greater than $1$ does not change their asymptotic order. By \autoref{theorem:5}, $N_i^{(1)}(l, d) \leq 2 \cdot 2^{k+1} = \mathcal{O} \left( 2^k \right)$, and $N_i(l, d) = \Theta \left( 2^k \right)$ (\autoref{theorem:1}), so:
    $$
        N_i^{>1}(l, d) = N_i(l, d) - N_i^{(1)}(l, d) = \Theta \left( 2^k \right)
    $$
    Likewise, since $W(l, d) = \Theta(k \cdot l)$ (\autoref{theorem:2}) and $N_i^{(1)}(l, d) = \mathcal{O} \left( 2^k \right) = \mathcal{O}(l) \subseteq \mathcal{O}(k \cdot l)$, we have:
    $$
        W^{>1}(l, d) = W(l, d) - N_i^{(1)}(l, d) = \Theta(k \cdot l)
    $$

    For any length $m \geq 1$ and any $\nu \geq 1$, we have:
    $$
        1 \leq n_q(m, \nu) = 1 + \left\lfloor \frac{(m - 1)}{\nu} \right\rfloor \leq 1 + \frac{m}{\nu} \leq 1 + m
    $$
    Summing over all the $N_i^{>1}(l, d)$ intervals of width strictly greater than $1$, we have:
    $$
        N_i^{>1}(l, d) \leq N_q^{>1}(l, d, \nu) \leq N_i^{>1}(l, d) + W^{>1}(l, d)
    $$
    Using \autoref{theorem:9}, we have:
    $$
        c_2 \cdot N_i^{>1}(l, d) + c_3 \cdot W^{>1}(l, d) \leq T_e^r(l, d, \nu) \leq c_2 \cdot N_i^{>1}(l, d) + (c_2 + c_3) \cdot W^{>1}(l, d)
    $$
    Since $N_i^{>1}(l, d) = \Theta \left( 2^k \right)$ and $W^{>1}(l, d) \geq l - N_i^{(1)}(l,d) = \Omega(l)$ (as $W(l,d) \geq l$ by definition and $N_i^{(1)}(l,d) = \mathcal{O}(2^k) = \mathcal{O}(l)$), we have $N_i^{>1}(l, d) = \mathcal{O} \left( W^{>1}(l, d) \right)$, hence:
    $$
        T_e^r(l, d, \nu) = \Theta \left( W^{>1}(l, d) \right) = \Theta(k \cdot l) = \Theta(e \cdot l)
    $$
    Since $e = d$ in the saturated regime, and $e = \lfloor \log_2 l \rfloor + 1$ in the unsaturated regime, we conclude:
    $$
        T_e^r(l, d, \nu) = \begin{cases}
            \Theta(d \cdot l) & \textnormal{if } l \geq 2^{d-1} \text{ (saturated regime)}\\
            \Theta(l \cdot \log(l)) & \textnormal{if } l < 2^{d-1} \text{ (unsaturated regime)}
        \end{cases}
    $$
\end{proof}

\begin{proof}[Proof of \autoref{lemma:1}]
    \label{proof:lemma:1}
    Quant processes each non-trivial interval by first sorting its values and then extracting the (interpolated, possibly centered) quantiles from the sorted values.
    These are the only two operations performed on an interval, and they are applied sequentially, so the processing cost of a single interval is the sum of its sort cost and its extraction cost.
    Summing over all the intervals (\autoref{theorem:1}) preserves this additivity, giving $T^r(l, d, \nu) = T_s^r(l, d) + T_e^r(l, d, \nu)$, where $T_s^r(l, d)$ and $T_e^r(l, d, \nu)$ are respectively the total sort cost (\autoref{theorem:7}) and the total extraction cost (\autoref{theorem:9}).
\end{proof}

\begin{proof}[Proof of \autoref{theorem:11}]
    \label{proof:theorem:11}
    By \autoref{lemma:1}, we have:
    $$
        T^r(l, d, \nu) = T_s^r(l, d) + T_e^r(l, d, \nu)
    $$
    We recall the results of \autoref{theorem:8} and \autoref{theorem:10} below:
    \begin{align*}
        T_s^r(l, d) &= \begin{cases}
            \Theta \left( d \cdot l \cdot \log(l) \right) & \textnormal{if } l \geq 2^{d-1} \text{ (saturated regime)}\\
            \Theta \left( l \cdot \left( \log(l) \right)^2 \right) & \textnormal{if } l < 2^{d-1} \text{ (unsaturated regime)}
        \end{cases}\\
        T_e^r(l, d, \nu) &= \begin{cases}
            \Theta(d \cdot l) & \textnormal{if } l \geq 2^{d-1} \text{ (saturated regime)}\\
            \Theta(l \cdot \log(l)) & \textnormal{if } l < 2^{d-1} \text{ (unsaturated regime)}
        \end{cases}
    \end{align*}

    In the saturated regime, we have $d \leq \lfloor \log_2(l) \rfloor + 1$, so $d \cdot l = o \left( d \cdot l \cdot \log(l) \right)$, thus the extraction term is asymptotically negligible next to the sort term.

    In the non-saturated regime, we have $l \cdot \log(l) = o \left( l \cdot \left( \log(l) \right)^2 \right)$, thus the extraction term is also asymptotically negligible next to the sort term.

    Therefore, in both regimes, the extraction term is asymptotically negligible next to the sort term, leading to the desired result:
    $$
        T^r(l, d, \nu) = T_s^r(l, d) = \begin{cases}
            \Theta \left( d \cdot l \cdot \log(l) \right) & \textnormal{if } l \geq 2^{d-1} \text{ (saturated regime)}\\
            \Theta \left( l \cdot \left( \log(l) \right)^2 \right) & \textnormal{if } l < 2^{d-1} \text{ (unsaturated regime)}
        \end{cases}
    $$

\end{proof}

\begin{proof}[Proof of \autoref{theorem:12}]
    \label{proof:theorem:12}
    Using \autoref{theorem:6} and replacing $k$ with $5$, we have:
    $$
        T_s^r(l, 6) = c_1 \cdot l \cdot \left[ \frac{321}{32} \cdot \log_2(l) - \frac{903}{32} \right]
    $$
    Replacing $l$ with $2^b$, we have:
    $$
        T_s^r \left( 2^b, 6 \right) = c_1 \cdot 2^b \cdot \left[ \frac{321}{32} \cdot b - \frac{903}{32} \right] = c_1 \cdot 2^{b-5} \cdot (321 b - 903) = 3 \cdot c_1 \cdot 2^{b-5} \cdot \left( 107 \cdot b - 301 \right)
    $$
\end{proof}

\begin{proof}[Proof of \autoref{theorem:13}]
    \label{proof:theorem:13}
    The starting point of the proof is \autoref{theorem:9}, $T_e^r(l,d,\nu) = c_2 \cdot N_q^{>1}(l,d,\nu) + c_3 \cdot W^{>1}(l,d)$, together with \autoref{theorem:3}, $N_q(l,d,\nu) = N_i(l,d) + \left( W(l,d) - N_i(l,d) \right)/\nu - E(l,d,\nu)$, where $E(l,d,\nu) = \sum_j \{ (m_j - 1)/\nu \}$ is summed over all the base and shifted intervals, and \autoref{theorem:5}, which gives $N_i^{(1)}(2^b,6) = 0$ for every $b \geq 6$ considered below (so that $N_q^{>1} = N_q$ and $W^{>1} = W$ throughout the $b > 5$ case), and $N_i^{(1)}(32,6) = 32 \neq 0$ for $b = 5$, treated separately at the end.

    Since $l = 2^b$ is divisible by $2^r$ for every $r \in \{0, \ldots, b\}$, and in particular for every $r \in \{0, \ldots, k\}$ with $k = \min(5, b) = 5$ (as $b \geq 5$), every level-$r$ correction term in \autoref{theorem:2} vanishes exactly, so $N_i(2^b, 6)$ and $W(2^b, 6)$ reduce to the leading terms of \autoref{theorem:1} and \autoref{theorem:2}. For the same reason, every base and shifted interval at level $r$ has the exact same width $2^{b-r}$, so the corresponding term of $E(2^b, 6, 4)$ is identical across all of them, and the sum over intervals $j$ collapses to a sum over levels $r$: $E(2^b, 6, 4) = \sum_{r=0}^k \left( 2^{r+1} - 1 \right) \left\{ (2^{b-r} - 1)/4 \right\}$ if level $k$ has shifted intervals (i.e. $2^b > 1.5 \cdot 2^k$), or the same sum truncated to $r \in \{0, \ldots, k-1\}$ if it does not (i.e. $2^b = 2^k$), per the proof of \autoref{theorem:1}.

    There are two different cases: $b = 5$ and $b > 5$. We start with the latter.

    \noindent\ul{Case $b > 5$}. Here $l = 2^b > 1.5 \cdot 2^5 = 48$ (since $b > 5$ means $2^b \geq 64$), so \autoref{theorem:1} gives $N_i(2^b, 6) = 2^7 - 5 - 3 = 120$, \autoref{theorem:2} gives $W(2^b, 6) = (2 \cdot 5 + 2^{-5}) \cdot 2^b = \frac{321}{32} \cdot 2^b$, and level $k=5$ has shifted intervals, so $E(2^b, 6, 4) = \sum_{r=0}^5 \left( 2^{r+1} - 1 \right) \left\{ (2^{b-r}-1)/4 \right\}$. Substituting into \autoref{theorem:3} and \autoref{theorem:9}, we have:
    \begin{align*}
        T_e^r \left(2^b, 6, 4 \right) &= c_2 \cdot 120 \cdot \frac{3}{4} + 2^b \cdot \left( \frac{321}{32} \right) \left( \frac{c_2}{4} + c_3 \right) - c_2 \cdot \sum_{r=0}^5 \left( 2^{r+1} - 1 \right) \left\{ \frac{2^{b-r} - 1}{4} \right\}\\
        T_e^r \left(2^b, 6, 4 \right) &= 90 \cdot c_2 + 2^b \cdot \frac{321}{128} \cdot c_2 + 2^b \cdot \frac{321}{32} \cdot c_3 - c_2 \sum_{r=0}^5 \left( 2^{r+1} - 1 \right) \left\{ \frac{2^{b-r} - 1}{4} \right\}
    \end{align*}
    We have two cases to distinguish to simplify the last term. For any $r \in \{0, \ldots, 5\}$:
    \begin{itemize}
        \item if $b - r = 1$, then $\left\{ (2^{b-r} - 1) / 4 \right\} = 0.25$,
        \item if $b - r \geq 2$, then $\left\{ (2^{b-r} - 1) / 4 \right\} = 0.75$.
    \end{itemize}
    Thus, we have:
    $$
        \sum_{r=0}^5 \left( 2^{r+1} - 1 \right) \left\{ \frac{2^{b-r} - 1}{4} \right\} = \begin{cases}
            \displaystyle 0.75 \times (1 + 3 + 7 + 15 + 31) + 0.25 \times 63 = 58.5 & \text{if } b = 6\\
            \displaystyle 0.75 \times 120 = 90 & \text{if } b \geq 7
        \end{cases}
    $$
    Simplifying the formula with the value of the last term derived above, we have:
    $$
        T_e^r \left(2^b, 6, 4 \right) = \begin{cases}
            \displaystyle 192 \cdot c_2 + 642 \cdot c_3 & \text{if } b = 6\\
            \displaystyle 321 \cdot \left( 2^{b-7} \cdot c_2 + 2^{b-5} \cdot c_3 \right) & \text{if } b \geq 7
        \end{cases}
    $$

    \noindent\ul{Case $b = 5$}. Here $l = 32 = 2^5 < 1.5 \times 2^5 = 48$, so \autoref{theorem:1} gives $N_i(32, 6) = 3 \times 32 - 5 - 2 = 89$, \autoref{theorem:2} gives $W(32, 6) = (9 + 2^{-4}) \times 32 = 290$, and level $k=5$ has no shifted intervals here, so $E(32, 6, 4) = \sum_{r=0}^4 \left( 2^{r+1} - 1 \right) \left\{ (2^{5-r}-1)/4 \right\}$, which, using the same case-split as above with $b=5$, evaluates to $27.25$.
    Substituting into \autoref{theorem:3}, we have $N_q(32, 6, 4) = 89 \cdot 3/4 + (290 - 89 \times 3/4 \times 4/3)/4 - 27.25$, which simplifies to $N_q(32, 6, 4) = 112$.

    Unlike the $b > 5$ case, here $k = 5 = \lfloor \log_2(32) \rfloor$ is exactly the deepest level reached (\autoref{theorem:1}'s proof), so this case has trivial (length-$1$) intervals, and \autoref{theorem:9} actually applies. Since $32 < 1.5 \times 2^5 = 48$, \autoref{theorem:5} gives $N_i^{(1)}(32, 6) = 2^6 - 32 = 32$, so:
    \begin{align*}
        N_i^{>1}(32, 6) &= 89 - 32 = 57 \\
        W^{>1}(32, 6) &= 290 - 32 = 258 \\
        N_q^{>1}(32, 6, 4) &= 112 - 32 = 80
    \end{align*}
    Substituting into \autoref{theorem:9}, we have:
    $$
        T_e^r(32, 6, 4) = c_2 \cdot N_q^{>1}(32,6,4) + c_3 \cdot W^{>1}(32,6,4) = 80 \cdot c_2 + 258 \cdot c_3
    $$

    \noindent\ul{General case $b \geq 5$}. We conclude:
    $$
        T_e^r \left(2^b, 6, 4 \right) = \begin{cases}
            \displaystyle 80 \cdot c_2 + 258 \cdot c_3 & \text{if } b = 5\\
            \displaystyle 192 \cdot c_2 + 642 \cdot c_3 & \text{if } b = 6\\
            \displaystyle 321 \cdot \left( 2^{b-7} \cdot c_2 + 2^{b-5} \cdot c_3 \right) & \text{if } b \geq 7
        \end{cases}
    $$
\end{proof}

\begin{proof}[Proof of \autoref{theorem:14}]
    \label{proof:theorem:14}
    We use \autoref{theorem:12} to derive the total sort cost for $b = 5$ and $b = 6$, \autoref{theorem:13} to derive the total extraction cost, sum both terms and simplify the expressions to obtain the desired result:
    $$
        T^r\left( 2^b, 6, 4 \right) = \begin{cases}
            \displaystyle 702 \cdot c_1 + 80 \cdot c_2 + 258 \cdot c_3 & \text{if } b = 5\\
            \displaystyle 2046 \cdot c_1 + 192 \cdot c_2 + 642 \cdot c_3 & \text{if } b = 6\\
            \displaystyle 2^{(b-5)} \cdot \left[ 3 \cdot c_1 \cdot (107 \cdot b - 301) + \frac{321}{4} \cdot c_2 + 321 \cdot c_3 \right] & \text{if } b \geq 7
        \end{cases}
    $$
\end{proof}

\begin{proof}[Proof of \autoref{lemma:2}]
    \label{proof:lemma:2}
    The raw representation is the series itself, of length $l_1(l) = l$.
    The first-order difference of a length-$l$ series has length $l - 1$.
    Quant smooths it with a length-$5$ centered moving average after padding both ends by $2$ (using edge-value padding), which restores the original (post-differencing) length, so the smoothed first-difference representation has length $l_2(l) = l - 1$.
    The second-order difference (the first-order difference applied twice) removes one point at each step, giving length $l_3(l) = l - 2$.
    The one-sided magnitude spectrum of the real discrete Fourier transform of a length-$l$ real sequence has $\lfloor l/2 \rfloor + 1$ non-redundant frequency bins (indices $0$ through $\lfloor l/2 \rfloor$), so $l_4(l) = \lfloor l/2 \rfloor + 1$.
    We require $l \geq 3$ so that $l_3(l) \geq 1$, i.e., so that every representation is well-defined and non-empty.
\end{proof}

\begin{proof}[Proof of \autoref{theorem:15}]
    \label{proof:theorem:15}
    Quant processes the four representations of a series independently and sequentially, applying the same interval-building-and-sorting procedure to each.
    The total sort cost is therefore exactly the sum, over the four representations, of the per-representation sort cost.
    By \autoref{remark:representation-generalization-exact}, the per-representation sort cost of representation $p$ is given by \autoref{theorem:7} evaluated at length $l_p(l)$ (\autoref{lemma:2}), giving the desired result.
\end{proof}

\begin{proof}[Proof of \autoref{theorem:16}]
    \label{proof:theorem:16}
    Identical to the proof of \autoref{theorem:15}, using \autoref{theorem:9} in place of \autoref{theorem:7} for the per-representation extraction cost.
    Unlike \autoref{theorem:7}, \autoref{theorem:9} is exact (it carries no $\mathcal{O}(\cdot)$ correction term), so the total is exact as soon as $N_q^{>1}$ and $W^{>1}$ are evaluated exactly via \autoref{theorem:5}.
\end{proof}

\begin{proof}[Proof of \autoref{theorem:17}]
    \label{proof:theorem:17}
    The total processing cost is the sum of the total sort cost (\autoref{theorem:15}) and the total extraction cost (\autoref{theorem:16}), exactly as in \autoref{lemma:1}.
\end{proof}

\begin{proof}[Proof of \autoref{theorem:18}]
    \label{proof:theorem:18}
    We first show that each representation's length is within a bounded factor of $l$: by \autoref{lemma:2}, $l_1(l) = l$, $l_2(l) = l-1$, and $l_3(l) = l-2$ differ from $l$ by an additive constant, and $l_4(l) = \lfloor l/2 \rfloor + 1$ satisfies $l/2 \leq l_4(l) \leq l/2 + 1$.
    For every $p$ and every $l \geq 4$, we therefore have $l/2 \leq l_p(l) \leq l$, i.e., $l_p(l) = \Theta(l)$.

    Consequently, $\log_2(l_p(l)) = \log_2(l) + \mathcal{O}(1)$ for every $p$.
    Writing $k_p = \min(d-1, \lfloor \log_2(l_p(l)) \rfloor)$ (the exponent \autoref{theorem:7} associates with representation $p$) and recalling $k = \min(d-1, \lfloor \log_2(l) \rfloor)$ (the exponent for the raw representation), the bound above gives $k_p \in \{k - 1, k, k+1\} \cap \{0, \ldots, d-1\}$ for every $p$ and every $l \geq 4$: a bounded ($\mathcal{O}(1)$) deviation from $k$.
    Equivalently, $e_p := k_p + 1 = e + \mathcal{O}(1)$, with $e_p \leq d$ always.

    We now bound each term of \autoref{theorem:15} and \autoref{theorem:16} using \autoref{theorem:8} and \autoref{theorem:10} applied at $l_p(l)$ in place of $l$:
    $$
        T_s^{r_p}(l_p(l), d) = \Theta \left( l_p(l) \cdot e_p \cdot \log(l_p(l)) \right), \qquad T_e^{r_p}(l_p(l), d, \nu) = \Theta(e_p \cdot l_p(l))
    $$
    Since $l_p(l) = \Theta(l)$ and $e_p = e + \mathcal{O}(1)$ with both $e_p$ and $e$ bounded above by $d$ and below by $1$ (so that an $\mathcal{O}(1)$ additive shift is also a $\Theta(1)$ multiplicative factor whenever $e = \Theta(1)$, and is asymptotically negligible whenever $e = \Theta(\log l) \to \infty$), we have $e_p = \Theta(e)$ and $\log_2(l_p(l)) = \Theta(\log(l))$.
    Substituting:
    $$
        T_s^{r_p}(l_p(l), d) = \Theta(l \cdot e \cdot \log(l)), \qquad T_e^{r_p}(l_p(l), d, \nu) = \Theta(e \cdot l)
    $$
    for every $p \in \{1,2,3,4\}$.
    Summing four terms of matching $\Theta$-order does not change the order (the constant factor $4$ is absorbed into the $\Theta$), so:
    $$
        T_s(l,d) = \sum_{p=1}^4 T_s^{r_p}(l_p(l), d) = \Theta(l \cdot e \cdot \log(l)), \qquad T_e(l,d,\nu) = \sum_{p=1}^4 T_e^{r_p}(l_p(l), d, \nu) = \Theta(e \cdot l)
    $$
    Since $e \cdot l = \mathcal{O}(l \cdot e \cdot \log(l))$ (as $\log_2(l) \geq 1$ for $l \geq 2$), $T_e(l,d,\nu)$ is dominated by $T_s(l,d)$, exactly as in the proof of \autoref{theorem:11}, so, by \autoref{theorem:17}, $T(l,d,\nu) = T_s(l,d) + T_e(l,d,\nu) = \Theta(l \cdot e \cdot \log(l))$.
    Splitting by which term wins in the definition of $e$ recovers the saturated/unsaturated case split, identical in form to \autoref{theorem:11}.
\end{proof}

\begin{proof}[Proof of \autoref{theorem:19}]
    \label{proof:theorem:19}
    The interval-outer implementation performs exactly $N_i(l,d)$ iterations of its outer loop (\autoref{theorem:1}), one per interval, but only the $N_i^{>1}(l,d)$ iterations visiting an interval of width strictly greater than $1$ (\autoref{theorem:5}) dispatch the vectorized sort-and-extract call that $\kappa^{(I)}(l)$ accounts for: an iteration visiting a width-$1$ interval instead copies a single value directly, at a cost that does not scale with $n$ the way the sort-and-extract dispatch cost does, and which \autoref{theorem:9} already established is otherwise negligible for the purpose of this cost model. Summing over the $N_i^{>1}(l,d)$ sort-and-extract iterations gives the setup term $N_i^{>1}(l,d) \cdot \kappa^{(I)}(l)$.

    Within a single interval's vectorized call, NumPy sorts and extracts from each of the $n$ rows of the $(n \times m)$ sub-array independently, at the same per-row cost as sorting and extracting from a single length-$m$ array. Summing this cost across the $n$ rows, and then across the $N_i(l,d)$ intervals, therefore reproduces exactly $n$ times the per-series sort-work and extraction accounting of \autoref{theorem:7} and \autoref{theorem:9}, evaluated with this implementation's own constants $c_1^{(I)}, c_2^{(I)}, c_3^{(I)}$. This gives the marginal term $n \cdot T^{r,(I)}(l,d,\nu)$.
\end{proof}

\begin{proof}[Proof of \autoref{theorem:20}]
    \label{proof:theorem:20}
    The series-outer implementation performs exactly $n$ iterations of its outer loop, one per series.
    Each iteration dispatches a single compiled call that processes all $N_i(l,d)$ intervals of that series internally.
    Because the loop over intervals is compiled rather than interpreted, it contributes no dispatch overhead of its own, so the only cost per iteration beyond the compiled kernel's own work is the single call-dispatch overhead $\kappa^{(S)}(l)$ of invoking it, a quantity that may itself vary with $l$ (e.g., through the cost of marshalling an $l$-dependent amount of data into and out of the compiled call).
    The compiled kernel implements the same sort-then-extract algorithm accounted for in \autoref{theorem:7} and \autoref{theorem:9}, so its cost for one series is exactly $T^{r,(S)}(l,d,\nu)$, using this implementation's own constants $c_1^{(S)}, c_2^{(S)}, c_3^{(S)}$.
    Summing the per-iteration cost $\kappa^{(S)}(l) + T^{r,(S)}(l,d,\nu)$ over the $n$ independent series gives the result.
\end{proof}

\begin{proof}[Proof of \autoref{theorem:21}]
    \label{proof:theorem:21}
    Both $T^{r,(I)}$ and $T^{r,(S)}$ are affine in $n$ (\autoref{theorem:19}, \autoref{theorem:20}): $T^{r,(I)}(n) = A_I + B_I \cdot n$ and $T^{r,(S)}(n) = B_S \cdot n$, with $A_I \geq 0$.

    If $B_S \leq B_I$: since $A_I \geq 0$, we have $T^{r,(S)}(n) = B_S \cdot n \leq B_I \cdot n \leq A_I + B_I \cdot n = T^{r,(I)}(n)$ for every $n \geq 1$.

    If $B_S > B_I$: the difference $T^{r,(I)}(n) - T^{r,(S)}(n) = A_I - (B_S - B_I) \cdot n$ is a strictly decreasing affine function of $n$, non-negative at $n = 0$ (since $A_I \geq 0$), and equal to zero at $n = n^*(l,d,\nu) := A_I / (B_S - B_I)$. Therefore, $T^{r,(I)}(n) > T^{r,(S)}(n)$ for $n < n^*$, and $T^{r,(I)}(n) < T^{r,(S)}(n)$ for $n > n^*$.
\end{proof}

\begin{proof}[Proof of \autoref{theorem:22}]
    \label{proof:theorem:22}
    Computing the moments of a single interval of width $m$ requires exactly one Welford/Pebay pass over its $m$ elements (\autoref{sec6.1}): each element is visited exactly once, and each visit performs the same fixed number of arithmetic operations regardless of $m$ or of the element's position within the interval.
    Charging this fixed per-element work at rate $\tilde c_1$, the cost of computing the moments of a single interval of width $m$ is exactly $\tilde c_1 \cdot m$, with no further dependence on $m$ or on $d$.
    Summing over all $N_i(l,d)$ intervals, and using $\sum_j m_j = W(l,d)$ (\autoref{theorem:2}), the total moment computation cost is exactly:
    $$
        \widetilde T_m^r(l,d) = \tilde c_1 \cdot \sum_j m_j = \tilde c_1 \cdot W(l,d)
    $$
\end{proof}

\begin{proof}[Proof of \autoref{theorem:23}]
    \label{proof:theorem:23}
    By \autoref{theorem:22}, $\widetilde T_m^r(l,d) = \tilde c_1 \cdot W(l,d)$, and by \autoref{theorem:2}, $W(l,d) = \Theta(k \cdot l) = \Theta(e \cdot l)$ with $k = e - 1$.
    Therefore, $\widetilde T_m^r(l,d) = \Theta(e \cdot l)$.
    Splitting by regime as in the proof of \autoref{theorem:8} ($e = d$ in the saturated regime, $e = \lfloor \log_2(l) \rfloor + 1 = \Theta(\log(l))$ in the unsaturated regime) gives the stated two cases.
\end{proof}

\begin{proof}[Proof of \autoref{theorem:24}]
    \label{proof:theorem:24}
    For a single interval whose moments have already been computed (\autoref{theorem:22}), the extraction step consists of two sub-steps, which the two constants keep separate precisely because they behave differently on trivial (length-$1$, kind $=0$) intervals.
    Converting the raw moments $(m, M_2, M_3, M_4)$ into (variance, skewness, excess kurtosis) (\autoref{sec6.1}) is applied uniformly to every interval, including trivial ones, at the fixed per-interval rate $\tilde c_3$, since Quant computes it as a single element-wise pass over the moments of \emph{all} $N_i(l,d)$ intervals at once, with no branch skipping trivial ones.
    Evaluating the Cornish-Fisher expansion, by contrast, is applied only to a trivial interval's non-existent request for an approximated quantile.
    A trivial interval's single requested position is instead filled directly from its (degenerate) mean, exactly as its exact-mode counterpart is filled directly from the interval's raw value (\autoref{theorem:9}).
    Thus, only the interior, genuinely-evaluated positions of the $N_q^{>1}(l,d,\nu)$ non-trivial-interval quantiles (\autoref{theorem:5}) are charged at rate $\tilde c_2$.
    Therefore, the extraction cost for a single non-trivial interval with $n_q(m,\nu)$ requested quantiles is $\tilde c_2 \cdot n_q(m,\nu) + \tilde c_3$, while a trivial interval costs only $\tilde c_3$ (the conversion, still performed, but no Cornish-Fisher evaluation).
    Summing over all $N_i(l,d)$ intervals for the $\tilde c_3$ term, and over the non-trivial ones only for the $\tilde c_2$ term, and using $\sum_{j:\, m_j > 1} n_q(m_j,\nu) = N_q^{>1}(l,d,\nu)$ (\autoref{theorem:5}), we obtain:
    $$
        \widetilde T_e^r(l,d,\nu) = \tilde c_2 \cdot N_q^{>1}(l,d,\nu) + \tilde c_3 \cdot N_i(l,d)
    $$
\end{proof}

\begin{proof}[Proof of \autoref{theorem:25}]
    \label{proof:theorem:25}
    By the same per-interval argument as in the proof of \autoref{theorem:10} ($1 \leq n_q(m,\nu) \leq m$ for every interval), restricted to the $N_i^{>1}(l,d)$ non-trivial intervals: $N_i^{>1}(l,d) \leq N_q^{>1}(l,d,\nu) \leq N_i^{>1}(l,d) + W(l,d)$ (the upper bound using $W(l,d)$, rather than the tighter $W^{>1}(l,d)$, since $W^{>1}(l,d) \leq W(l,d)$).
    Combined with \autoref{theorem:24}, this gives:
    $$
        \tilde c_2 \cdot N_i^{>1}(l,d) + \tilde c_3 \cdot N_i(l,d) \leq \widetilde T_e^r(l,d,\nu) \leq \tilde c_2 \cdot \left( N_i^{>1}(l,d) + W(l,d) \right) + \tilde c_3 \cdot N_i(l,d)
    $$
    Since $N_i(l,d) = \mathcal{O}(l) = \mathcal{O}(W(l,d))$ (\autoref{theorem:1}, and $W(l,d) \geq l$ by definition) and likewise $N_i^{>1}(l,d) \leq N_i(l,d) = \mathcal{O}(W(l,d))$, both the lower and upper bounds above are $\Theta(W(l,d))$, exactly as in the proof of \autoref{theorem:10}, so:
    $$
        \widetilde T_e^r(l,d,\nu) = \Theta(W(l,d)) = \Theta(e \cdot l)
    $$
    using \autoref{theorem:2}.
    Splitting by regime as in \autoref{theorem:10} gives the stated two cases.
\end{proof}

\begin{proof}[Proof of \autoref{lemma:3}]
    \label{proof:lemma:3}
    The approximate algorithm processes each interval by first computing its moments and then extracting the (Cornish-Fisher-approximated, possibly centered) quantiles from those moments.
    These are the only two operations performed on an interval, applied sequentially, exactly as in the proof of \autoref{lemma:1} for the exact algorithm's sort-then-extract processing. Summing over all the intervals (\autoref{theorem:1}) preserves this additivity.
\end{proof}

\begin{proof}[Proof of \autoref{theorem:26}]
    \label{proof:theorem:26}
    By \autoref{lemma:3}, $\widetilde T^r(l,d,\nu) = \widetilde T_m^r(l,d) + \widetilde T_e^r(l,d,\nu)$.
    By \autoref{theorem:23} and \autoref{theorem:25}, both terms are $\Theta(e \cdot l)$, unlike the exact algorithm (\autoref{theorem:11}), where the sort term dominates the extraction term.
    Here the two terms are already of the same order, so their sum is $\Theta(e \cdot l)$ as well.
    Splitting by regime as in \autoref{theorem:8} gives the stated two cases.
\end{proof}

\begin{proof}[Proof of \autoref{theorem:27}]
    \label{proof:theorem:27}
    Quant computes the moments of the four representations of a series independently and sequentially, applying the same interval-building-and-moment-computation procedure to each. The total moment computation cost is therefore exactly the sum, over the four representations, of the per-representation moment cost. By \autoref{remark:representation-generalization-approx}, the per-representation cost of representation $p$ is given by \autoref{theorem:22} evaluated at length $l_p(l)$ (\autoref{lemma:2}), giving the desired result.
\end{proof}

\begin{proof}[Proof of \autoref{theorem:28}]
    \label{proof:theorem:28}
    Identical to the proof of \autoref{theorem:27}, using \autoref{theorem:24} in place of \autoref{theorem:22} for the per-representation extraction cost. As in \autoref{theorem:16}, \autoref{theorem:24} is exact (it carries no $\mathcal{O}(\cdot)$ correction term), so the total is exact as soon as $N_q^{>1}$ and $N_i$ are evaluated exactly via \autoref{theorem:5} and \autoref{theorem:1}.
\end{proof}

\begin{proof}[Proof of \autoref{theorem:29}]
    \label{proof:theorem:29}
    The total processing cost is the sum of the total moment computation cost (\autoref{theorem:27}) and the total extraction cost (\autoref{theorem:28}), exactly as in \autoref{lemma:3}.
\end{proof}

\begin{proof}[Proof of \autoref{theorem:30}]
    \label{proof:theorem:30}
    By the proof of \autoref{theorem:18}, $l_p(l) = \Theta(l)$ and, writing $e_p := \min(d, \lfloor \log_2(l_p(l)) \rfloor + 1)$, we have $e_p = \Theta(e)$ for every $p \in \{1,2,3,4\}$. Applying \autoref{theorem:23} and \autoref{theorem:25} at $l_p(l)$ in place of $l$:
    $$
        \widetilde T_m^{r_p}(l_p(l), d) = \Theta(e_p \cdot l_p(l)) = \Theta(e \cdot l), \qquad \widetilde T_e^{r_p}(l_p(l), d, \nu) = \Theta(e_p \cdot l_p(l)) = \Theta(e \cdot l)
    $$
    for every $p$. Summing four terms of matching $\Theta$-order does not change the order (the constant factor $4$ is absorbed into the $\Theta$), so, using \autoref{theorem:27} and \autoref{theorem:28}:
    $$
        \widetilde T_m(l,d) = \sum_{p=1}^4 \widetilde T_m^{r_p}(l_p(l), d) = \Theta(e \cdot l), \qquad \widetilde T_e(l,d,\nu) = \sum_{p=1}^4 \widetilde T_e^{r_p}(l_p(l), d, \nu) = \Theta(e \cdot l)
    $$
    By \autoref{theorem:29}, $\widetilde T(l,d,\nu) = \widetilde T_m(l,d) + \widetilde T_e(l,d,\nu)$ is the sum of two terms of the same order, hence itself $\Theta(e \cdot l)$. Splitting by which term wins in the definition of $e$ gives the stated two cases, identical in form to \autoref{theorem:23} and \autoref{theorem:25}.
\end{proof}

\begin{proof}[Proof of \autoref{theorem:31}]
    \label{proof:theorem:31}
    For a single representation $p$, the moment computation of all $N_i(l_p(l),d)$ intervals (\autoref{theorem:22}) is dispatched as a single call, independent of $n_i$.
    It is the subsequent per-interval extraction step, converting each interval's moments into its requested quantiles, that the interval-outer implementation runs as $N_i(l_p(l),d)$ separate outer-loop iterations, one per interval (\autoref{theorem:1}, evaluated at $l_p(l)$), mirroring the proof of \autoref{theorem:19}.
    Of these, only the $N_i^{>1}(l_p(l),d)$ iterations visiting a non-trivial interval (\autoref{theorem:5}) dispatch a genuine Cornish-Fisher evaluation call.
    An iteration visiting a trivial interval instead copies its already-known mean directly, at a cost this proof, like that of \autoref{theorem:19} and \autoref{theorem:24}, treats as negligible.
    This gives the setup term $N_i^{>1}(l_p(l),d) \cdot \tilde\kappa^{(I)}(l_p(l))$ for representation $p$.
    The moment computation and (for non-trivial intervals) Cornish-Fisher extraction are each applied to all $n$ rows at once within their respective calls, reproducing $n$ times \autoref{theorem:22} and \autoref{theorem:24}'s per-series accounting at length $l_p(l)$, evaluated with this implementation's own constants, giving the marginal term $n \cdot \widetilde T^{r,(I)}(l_p(l),d,\nu)$.
    Since the four representations are processed independently and sequentially, summing this per-representation cost over $p \in \{1,2,3,4\}$ gives the total, using $\sum_p n \cdot \widetilde T^{r,(I)}(l_p(l),d,\nu) = n \cdot \widetilde T^{(I)}_\Sigma(l,d,\nu)$.
\end{proof}

\begin{proof}[Proof of \autoref{theorem:32}]
    \label{proof:theorem:32}
    For a single representation $p$, the argument is identical in structure to the proof of \autoref{theorem:20}: the series-outer implementation performs $n$ outer-loop iterations, each incurring a per-series overhead $\tilde\kappa^{(S)}(l_p(l))$, possibly varying with the representation's own length $l_p(l)$.
    The compiled kernel's cost for one series at length $l_p(l)$ is exactly $\widetilde T^{r,(S)}(l_p(l),d,\nu)$, using this implementation's own constants.
    Summing over the $n$ independent series gives $n \cdot \left[ \tilde\kappa^{(S)}(l_p(l)) + \widetilde T^{r,(S)}(l_p(l),d,\nu) \right]$ for representation $p$ alone.
    Since the four representations are processed independently and sequentially, each incurring its own per-series call-dispatch overhead $\tilde\kappa^{(S)}(l_p(l))$, summing over $p \in \{1,2,3,4\}$ gives:
    \begin{align*}
        \widetilde T^{(S)}(n,l,d,\nu) &= \sum_{p=1}^4 n \cdot \left[ \tilde\kappa^{(S)}(l_p(l)) + \widetilde T^{r,(S)}(l_p(l),d,\nu) \right]\\
        &= n \cdot \left[ \sum_{p=1}^4 \tilde\kappa^{(S)}(l_p(l)) + \sum_{p=1}^4 \widetilde T^{r,(S)}(l_p(l),d,\nu) \right]\\
        T^{(S)}(n,l,d,\nu) &= n \cdot \left[ \tilde\kappa^{(S)}_\Sigma(l) + \widetilde T^{(S)}_\Sigma(l,d,\nu) \right]
    \end{align*}
\end{proof}

\begin{proof}[Proof of \autoref{theorem:33}]
    \label{proof:theorem:33}
    Identical to the proof of \autoref{theorem:21}, substituting $\widetilde T^{(I)}, \widetilde T^{(S)}, \widetilde A_I, \widetilde B_I, \widetilde B_S$ for their exact-algorithm counterparts: the argument only uses that both $\widetilde T^{(I)}(n)$ and $\widetilde T^{(S)}(n)$ are affine in $n$ (\autoref{theorem:31}, \autoref{theorem:32}), with a non-negative intercept for $\widetilde T^{(I)}$, which holds here by the same reasoning as in the exact algorithm.
\end{proof}

\section{Additional results}\label{secB}

\autoref{sec8.5} shows the exact-vs-approx quantile correlation histograms for StarLightCurves in the main text, the largest and most fully depth-populated of the six data sets used there.
\autoref{fig:ecg200-correlation} through \autoref{fig:yoga-correlation} show the same figure for the remaining five data sets (see \autoref{table:quantile_correlation_stats} for their summary statistics).
ItalyPowerDemand only populates depths $0$ to $4$, since its series length ($l=24$) is short enough that the depth cap $\min(d, \lfloor \log_2(l) \rfloor + 1)$ (\autoref{sec3}) binds before the configured $d=6$ is reached, leaving its depth-$5$ panel empty.

\begin{figure}[tbp]
    \centering
    \includegraphics[width=\textwidth]{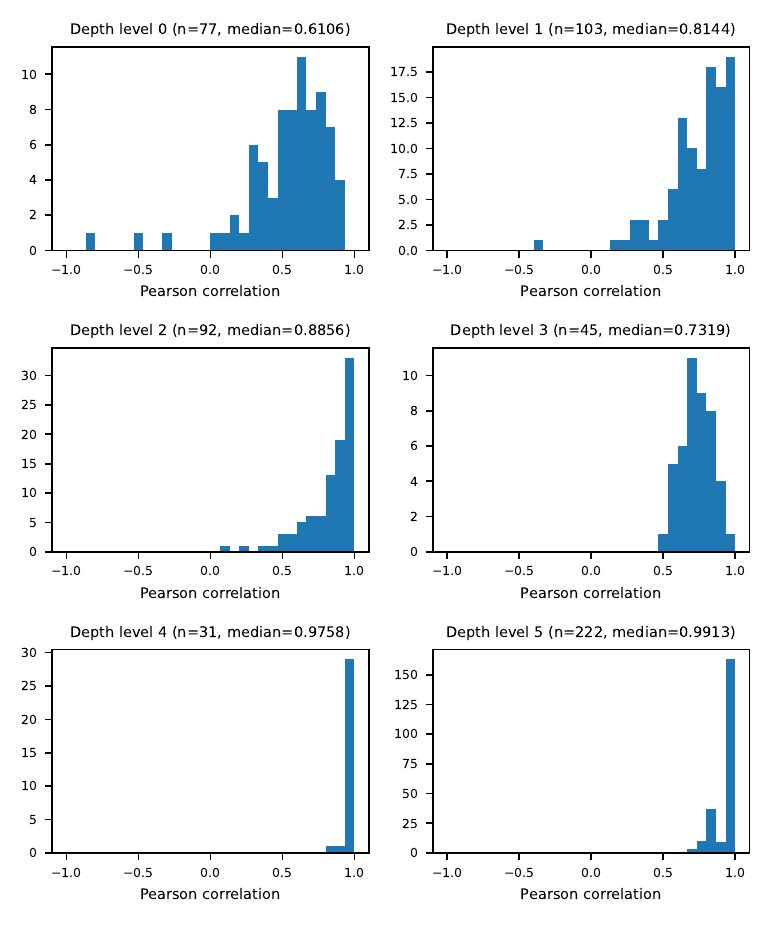}
    \caption{
        Distribution of the Pearson correlation coefficients between the exact and approximated moment-based quantile values on the ECG200 data set ($l = 96$), with one histogram per depth level $0$ to $5$.
        Each panel's title reports the number of genuinely-approximated columns at that depth and their median correlation.
    }
    \label{fig:ecg200-correlation}
\end{figure}

\begin{figure}[tbp]
    \centering
    \includegraphics[width=\textwidth]{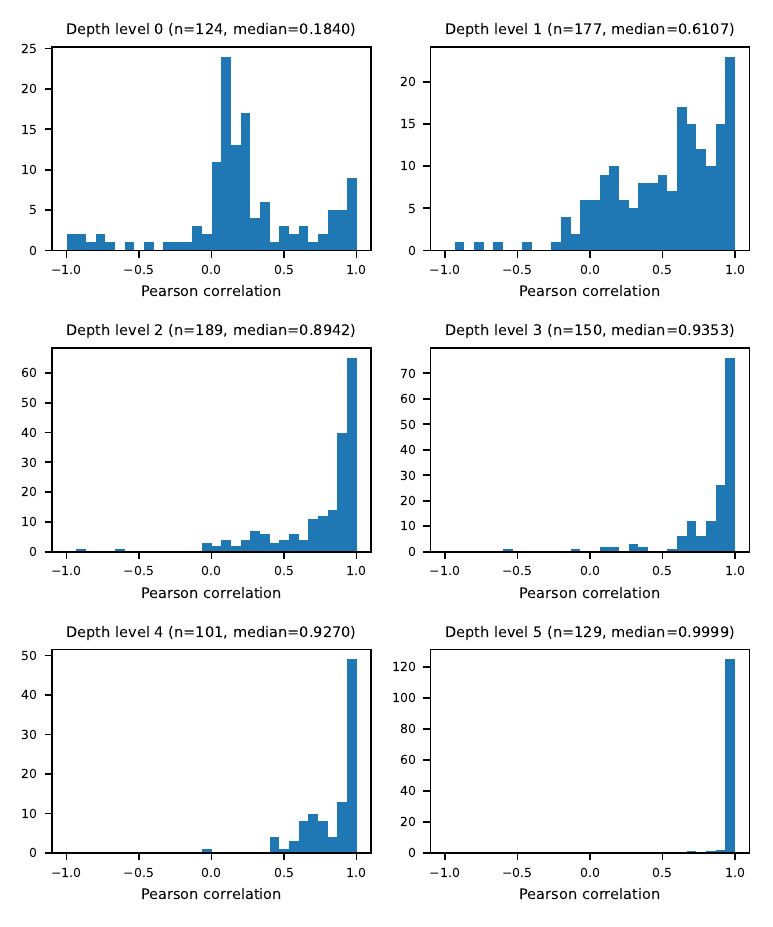}
    \caption{
        Distribution of the Pearson correlation coefficients between the exact and approximated moment-based quantile values on the GunPoint data set ($l = 150$), with one histogram per depth level $0$ to $5$.
        Each panel's title reports the number of genuinely-approximated columns at that depth and their median correlation.
    }
    \label{fig:gunpoint-correlation}
\end{figure}

\begin{figure}[tbp]
    \centering
    \includegraphics[width=\textwidth]{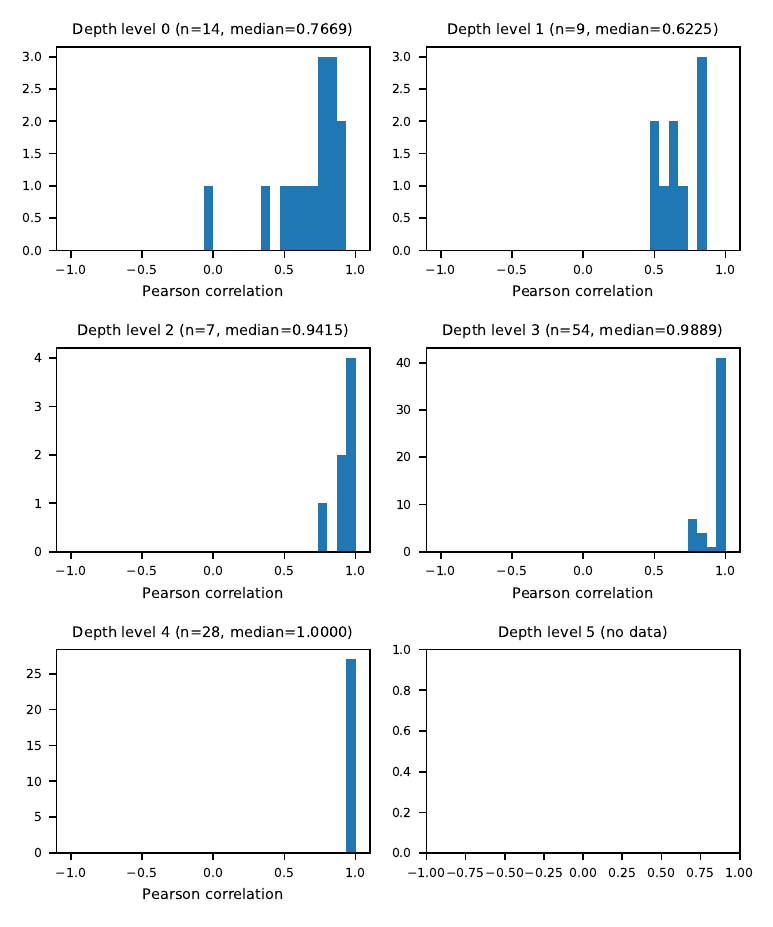}
    \caption{
        Distribution of the Pearson correlation coefficients between the exact and approximated moment-based quantile values on the ItalyPowerDemand data set ($ l =24$), with one histogram per depth level $0$ to $5$.
        The depth cap binds for this data set, so the panel for depth level $5$ has no data.
        Each panel's title reports the number of genuinely-approximated columns at that depth and their median correlation.
    }
    \label{fig:italypowerdemand-correlation}
\end{figure}

\begin{figure}[tbp]
    \centering
    \includegraphics[width=\textwidth]{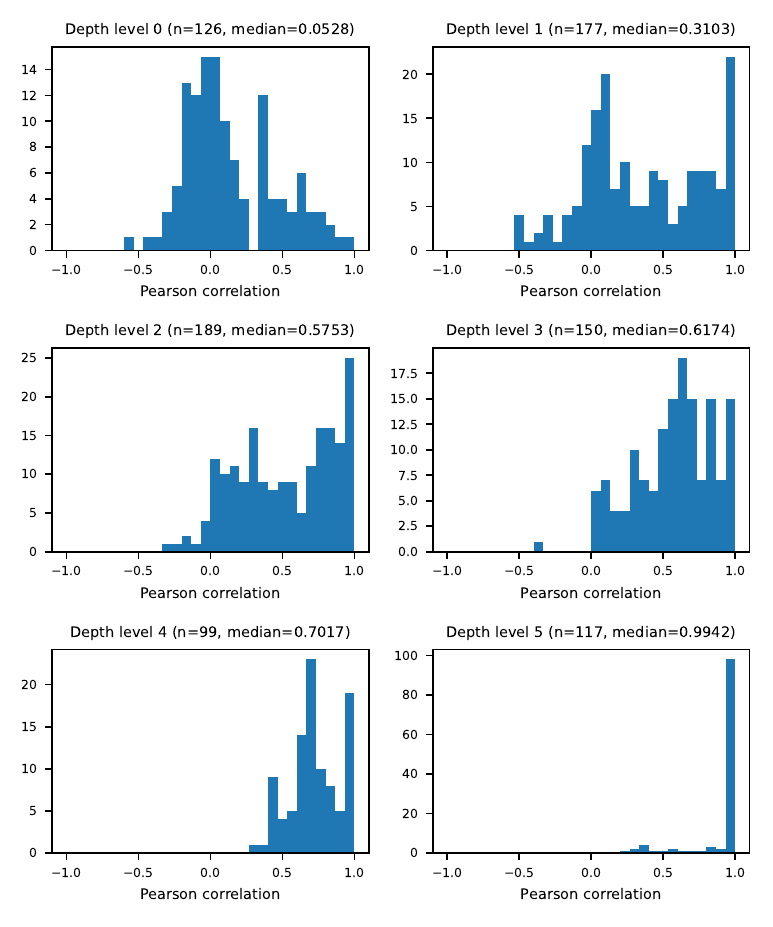}
    \caption{
        Distribution of the Pearson correlation coefficients between the exact and approximated moment-based quantile values on the Wafer data set ($l = 152$), with one histogram per depth level $0$ to $5$.
        Each panel's title reports the number of genuinely-approximated columns at that depth and their median correlation.
    }
    \label{fig:wafer-correlation}
\end{figure}

\begin{figure}[tbp]
    \centering
    \includegraphics[width=\textwidth]{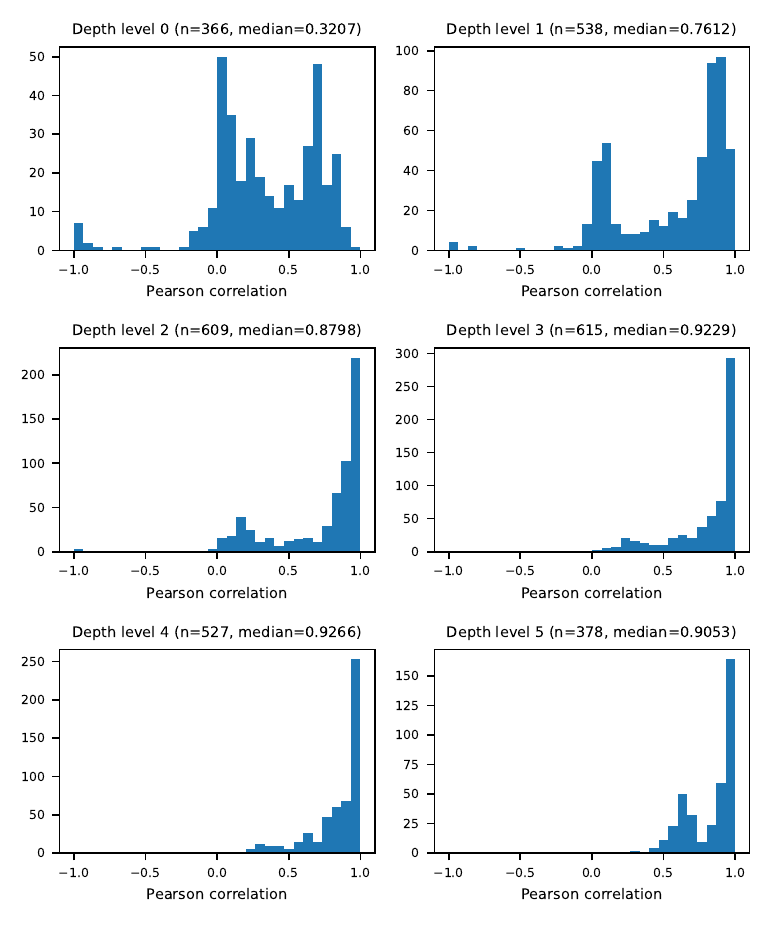}
    \caption{
        Distribution of the Pearson correlation coefficients between the exact and approximated moment-based quantile values on the Yoga data set ($l = 426$), with one histogram per depth level $0$ to $5$.
        Each panel's title reports the number of genuinely-approximated columns at that depth and their median correlation.
    }
    \label{fig:yoga-correlation}
\end{figure}

\end{appendices}

\clearpage

\bibliography{sn-bibliography}

\end{document}